\documentclass[final,12pt]{clear2023} 

\usepackage[utf8]{inputenc} % allow utf-8 input
\usepackage[T1]{fontenc}    % use 8-bit T1 fonts
\usepackage{hyperref}       % hyperlinks
\usepackage{url}            % simple URL typesetting
\usepackage{booktabs}       % professional-quality tables
\usepackage{amsfonts}       % blackboard math symbols
\usepackage{nicefrac}       % compact symbols for 1/2, etc.
\usepackage{microtype}      % microtypography
\usepackage{xcolor}         % colors

\usepackage{makecell}
\usepackage{enumitem}
\usepackage{wrapfig}

\usepackage{bbm}

\usepackage{caption}
\usepackage{float}

\usepackage[ruled,lined,linesnumbered]{algorithm2e}
\usepackage{etoolbox}

\makeatletter
\patchcmd{\@algocf@start}% <cmd>
  {-1.5em}% <search>
  {0pt}% <replace>
  {}{}% <success><failure>
\makeatother

\usepackage{tikz}
\usetikzlibrary{automata,decorations.markings,positioning,arrows}
\usepackage{wrapfig, mathtools}

\usepackage{cleveref}

\usetikzlibrary{shapes,decorations,arrows,calc,arrows.meta,fit,positioning}
\tikzset{
    -Latex,auto,node distance =1 cm and 1 cm,semithick,
    state/.style ={ellipse, draw, minimum width = 0.7 cm},
    point/.style = {circle, draw, inner sep=0.04cm,fill,node contents={}},
    bidirected/.style={Latex-Latex,dashed},
    el/.style = {inner sep=2pt, align=left, sloped}
}

\def\E{{\bf E}}

\def\L{{\bf L}}

\def\I{{\bf I}}

\def\Q{{\bf Q}}

\def\S{{\bf S}}

\def\Z{{\bf Z}}
\def\M{{\bf M}}

\def\U{{\bf U}}

\def\V{{\bf V}}

\def\0{{\bf 0}}
\def\1{{\bf 1}}

\def\UM{{\mathcal U}}

\def\OM{{\mathcal O}}

\def\GM{{\mathcal G}}

\def\argmin{\mathop{\rm argmin}}
\def\vec{\text{vec}}

\newcommand{\DNe}{\text{Ne}}
\newcommand{\DNep}{\text{Ne}^{+}}
\newcommand{\MB}{\text{MB}}
\newcommand{\MBp}{\text{MB}^{+}}
\newcommand{\nss}{\text{mns}}
\newcommand{\mns}{\nss}
\newcommand{\Pa}{\text{Pa}}
\newcommand{\Uo}{\text{Uo}}
\newcommand{\Ch}{\text{Ch}}

\newcommand{\Desc}{\text{Desc}}
\newcommand{\PossDesc}{\text{PossDesc}}

\newcommand{\CUC}{\text{CUC}}

\newcommand{\indep}{\perp \!\!\! \perp}
\newcommand{\notindep}{\not\!\perp\!\!\!\perp}

\newtheorem*{remark*}{Remark}

\newtheorem*{proposition*}{Proposition}
\newtheorem{cor}{Corollary}
\newtheorem{assumption}{Assumption}

\newenvironment{proofsketch}{%
  \renewcommand{\proofname}{Proof Sketch}\proof}{\endproof}

\SetKwInput{KwInput}{Input}
\SetKwInput{KwOutput}{Output}

\title[Local Causal Discovery for Estimating Causal Effects]{Local Causal Discovery for Estimating Causal Effects}
\usepackage{times}

\clearauthor{\Name{Shantanu Gupta} \Email{shantang@cmu.edu}\\
 \Name{David Childers} \Email{dchilders@cmu.edu}\\
 \Name{Zachary C. Lipton} \Email{zlipton@cmu.edu}\\
 \addr Carnegie Mellon University}

\begin{document}

\maketitle

\begin{abstract}
Even when the causal graph underlying our data is unknown,
we can use observational data to
narrow down the possible values 
that an average treatment effect (ATE) can take
by (1) identifying the graph up to a Markov equivalence class;
and (2) estimating that ATE for each graph in the class.
While the PC algorithm can identify this class 
under strong faithfulness assumptions,
it can be computationally prohibitive. 
Fortunately, only the local graph structure 
around the treatment is required to identify 
the set of possible ATE values,
a fact exploited by local discovery algorithms 
to improve computational efficiency.
In this paper, we introduce
Local Discovery using Eager Collider Checks (LDECC),
a new local causal discovery algorithm 
that leverages unshielded colliders to orient
the treatment's parents differently
from existing methods.
We show that there exist graphs where LDECC
exponentially outperforms existing local discovery algorithms and vice versa.
Moreover, we show that
LDECC and existing algorithms
rely on different faithfulness assumptions,
leveraging this insight to weaken the assumptions 
for identifying the set of possible ATE values.
\end{abstract}


\section{Introduction}
\label{sec:introduction}
Estimating an average treatment effect (ATE) from observational data 
typically requires structural knowledge,
which can be represented in the form of a causal graph.
While a rich literature offers methods
for identifying and estimating causal effects
given a \emph{known} causal graph 
\citep{tian2002general, pearl2009causality, jaber2019causal},  %},
many applications require that we investigate
the values that an ATE could possibly take 
when the causal graph is unknown.
In such cases, we can (i) perform causal discovery,
using observational data
to identify the graph up to
a Markov equivalence class (MEC); and 
(ii) estimate the desired ATE 
for every graph in the MEC,
thus identifying the set of possible ATE values.
We denote this (unknown) set of identified ATE values by $\Theta^{*}$.

Causal discovery has been investigated
under a variety of assumptions \citep{glymour2019review, squires2022causal}.
Under causal sufficiency (i.e., no unobserved variables) and faithfulness, 
the PC algorithm \citep[Sec.~5.4.2]{spirtes2000causation} 
can identify the MEC of the true graph
from observational data (and thus $\Theta^{*}$).
However, fully characterizing the MEC 
can be computationally expensive.
Addressing this concern,
\citet{maathuis2009estimating} proved that
the local structure around the treatment node
is sufficient for identifying $\Theta^{*}$.
Leveraging this insight,
existing local causal discovery algorithms (e.g.,
PCD-by-PCD \citep{yin2008partial}, MB-by-MB \citep{wang2014discovering})
discover just enough of the graph
to identify any parents and children of the treatment
that PC would have discovered.
These methods sequentially discover the local structure
around the treatment, its neighbors, and so on, 
terminating whenever all neighbors 
of the treatment are oriented 
(or no remaining neighbors can be oriented).

In this work, we introduce 
Local Discovery with Eager Collider Checks (LDECC),
a new local causal discovery algorithm 
that provides an  alternative method to orient
the parents of a treatment $X$ (Sec.~\ref{sec:ldecc}).
Initially, LDECC performs local discovery around
$X$ to discover its neighbors.
Subsequently, LDECC chooses the same Conditional Independence (CI) tests 
as PC would choose given the state of the graph,
with one crucial exception:
whenever we find two nodes $A$ and $B$ 
such that $A \indep B | \S$ for some set $\S$ with $X \notin \S$,
LDECC immediately checks whether they become dependent
when $X$ is added to the conditioning set.
If the test reveals dependence $A \notindep B | \S \cup \{X\}$,
then $X$ must either be a collider or a descendant of a collider
that lies at the intersection 
of some path from $A$ to $X$ or from $B$ to $X$.
On this basis, LDECC can orient the smallest subset 
of $X$'s neighbors that d-separate it from $\{A,B\}$ as parents.
We prove that, under faithfulness,
the identified ATE set is equal to $\Theta^*$.

We represent the ideas underlying existing local causal discovery algorithms
using a simple algorithm that we call Sequential Discovery (SD) that 
sequentially runs the PC algorithm locally
for local structure learning.
While existing algorithms differ from SD in subtle ways,
they share key steps with SD allowing us to
compare LDECC to this existing class of algorithms.
We highlight LDECC's complementary strengths to SD
in terms of computational requirements.
We present classes of causal graphs where LDECC 
performs exponentially fewer CI tests than SD, and vice versa (Sec.~\ref{sec:comparison-of-tests}).
Thus, the methods can be combined profitably
(by running LDECC and SD in parallel and terminating
when either algorithm terminates),
avoiding exponential runtimes 
if either algorithm's runtime is subexponential,
thereby expanding the class of graphs where efficient
local discovery is possible.

We also compare SD and LDECC based on their
faithfulness requirements (Sec.~\ref{sec:faithfulness}).
We show that both SD and PC require weaker
assumptions to identify $\Theta^*$ than to identify the entire MEC.
We also find that LDECC and SD rely on different
sets of faithfulness assumptions.
There are classes of faithfulness violations 
where one algorithm will correctly identify $\Theta^*$
while the other will not.
Under the assumption that 
one of the algorithms' faithfulness assumptions is correct, 
we propose a procedure that recovers a 
conservative bound on the ATE set.
Aiming to make this bound sharp, we prove that
LDECC and SD can be combined to construct a
procedure that can identify $\Theta^*$
under strictly weaker assumptions,
again highlighting LDECC's complementary
nature relative to existing methods.
Finally, we empirically test LDECC on synthetic as well as 
semi-synthetic graphs (Sec.~\ref{sec:experiments}) and
show that it performs comparably to SD (and PC) 
and typically runs fewer CI tests than SD
\footnote{The code and data are available at \href{https://github.com/acmi-lab/local-causal-discovery}{https://github.com/acmi-lab/local-causal-discovery}.}.

\section{Related Work}
\label{sec:related}
Several works have developed procedures
for identifying the set of possible
ATE values $\Theta^*$ using local information.
\citet{maathuis2009estimating} propose IDA, 
an algorithm that outputs
a set of ATE values by only using the local structure
around the treatment $X$:
For a given MEC,
$\Theta^*$ can be identified by adjusting 
for all possible parent sets of $X$
such that
no new unshielded colliders get created at $X$.
IDA has been extended to
account for multiple interventions \citep{nandy2017estimating}, background information \citep{perkovic2017interpreting, fang2020ida}, and hidden variables \citep{malinsky2016estimating}.
\citet{hyttinen2015calculus} combine causal discovery and inference
by presenting a SAT solver-based approach to do-calculus with an unknown graph.
\citet{geffner2022deep} propose a deep learning based
end-to-end method that learns a posterior over graphs
and estimates the ATE by marginalizing over the graphs.
\citet{toth2022active} propose a Bayesian active learning 
approach to jointly learn a posterior over causal models and
some target query of interest.
However, these works do not propose local discovery methods.

Other works attempt to perform causal inference
under weaker assumptions than requiring the entire 
graph or MEC.
Some works propose methods to find valid
adjustment sets under the assumption that
the observed covariates are pre-treatment
\citep{de2011covariate, vanderweele2011new, entner2012statistical, entner2013data, witte2019covariate, gultchin2020differentiable}.
\citet{cheng2022toward} use an anchor variable 
(which they call COSO) 
and \citet{shah2022finding} use a known causal parent of the treatment
to identify a valid adjustment set.
Other works present data-driven methods to find 
valid instrumental variables \citep{silva2017learning, cheng2022discovering}.
\citet{watson2022causal} propose a method to learn the causal order
amongst variables that are descendants of some covariate set.
By contrast, we do not make any partial ordering assumptions.

A complementary line of work focuses on
only discovering the local structure around a given node.
Many works attempt to find the Markov blanket (MB) or the
parent-child set of a target node
\citep{aliferis2003hiton, tsamardinos2006max, yu2020causality, ling2021causal}.
Unlike our work, they do not distinguish between the parent
and child identities of nodes inside the MB 
(see \citet{aliferis2010local} for a review of these methods).
For determining these causal identities, 
existing works use sequential approaches 
by repeatedly finding local structures like a MB 
or a parent-child set
starting from the target node and 
then propagating edge orientations
within the discovered subgraph
\citep{yin2008partial, zhou2010discover, wang2014discovering, gao2015local, ling2020using}.
These methods typically differ in how they discover 
the local structure around each node.
\citet{cooper1997simple} present an algorithm
to learn pairwise causal relationships 
by leveraging a node that is not
caused by any other nodes.
Some works use independence patterns in Y-structures around a given node
to find local causal relationships \citep{mani2012theoretical, mooij2015empirical, versteeg2022local}.

\section{Preliminaries}
\label{sec:prelim}
We assume that 
the causal structure of the
observational data can be encoded 
using a DAG $\GM^{*}(\V, \E)$, 
where $\V$ and $\E$ are the set of nodes and edges. 
Each edge $A \rightarrow B \in \E$ indicates that $A$ is a direct cause of $B$.
We denote the treatment node by $X$ 
and the outcome node by $Y$.
For a node $V$, we denote its Markov blanket, neighbors, parents, children, and descendants by $\MB(V)$, $\DNe(V), \Pa(V), \Ch(V)$, and $\Desc(V)$
(with $\MB(V) = \DNe(V) \cup_{N \in \Ch(V)} \Pa(N)$).
Let $\DNe^{+}(V) = \DNe(V) \cup \{ V \}$ and 
$\MB^{+}(V) = \MB(V) \cup \{ V \}$.
An unshielded collider (UC) is a triple $P \rightarrow R \leftarrow Q$
s.t. $P \text{---} Q \notin \E$.
For a UC $\alpha = (P \rightarrow R \leftarrow Q)$,
let $\text{sep}(\alpha) = \min \{ |\S| : 
\S \subseteq \V \setminus \{ P, Q \} \, \text{and} \,
P \indep Q | \S \}$.

A DAG entails a set of CIs via \emph{d-separation} \citep[Sec.~1.2.3]{pearl2009causality}.
DAGs that entail the same set of CIs form an MEC,
which can be characterized by a 
\emph{completed  partially  directed  acyclic graph} (CPDAG).
For the true DAG $\GM^{*}$, we denote by
$\Theta^{*}$ the set of ATE values (of $X$ on $Y$)
in each DAG in the MEC corresponding to $\GM^{*}$.
The causal faithfulness assumption (CFA)
holds iff all CIs satisfied by
the observational joint distribution $\mathbb{P}(\V)$ 
are entailed by $\GM^{*}$.
Throughout this work, we focus on causally sufficient graphs.
Under the CFA, the PC algorithm 
recovers this MEC roughly as follows
(demonstrated for the DAG in Fig.~\ref{fig:pc-true-graph}):
(i) estimate the skeleton by running CI tests;
(ii) find UCs in this skeleton (Fig.~\ref{fig:pc-cpdag} red edges); and 
(iii) orient additional edges using Meek's rules \citep{meek2013causal} (Fig.~\ref{fig:pc-cpdag} blue edges) to get a CPDAG (full details on the PC algorithm are in Appendix~\ref{apdx:prelim}).

We instantiate the key ideas of 
existing local discovery algorithms
using SD (Fig.~\ref{fig:algo-sd}).
SD sequentially finds the local structure
around nodes (\emph{LocalPC} in Fig.~\ref{fig:algo-common-additional-functions}) starting from
$X$, then its neighbors, and so on (Lines~\ref{alg-sd:start-seq}--\ref{alg-sd:end-seq}).
After each such local discovery step, SD orients nodes in the 
subgraph discovered until that point
using UCs and Meek's rules like PC (Line~\ref{alg-sd:orient-in-subgraph}; 
see Fig.~\ref{fig:apdx-pc-algorithm} for \emph{GetCPDAG} and Fig.~\ref{fig:apdx-nbrs-subroutine} for \emph{Nbrs}).
If at any point all neighbors of $X$ get oriented,
SD terminates.
The ATE set is estimated by applying the backdoor adjustment \citep[Thm.~3.3.2]{pearl2009causality} with every \emph{locally valid} parent set of $X$, 
i.e., one that does not create a new UC at $X$
(see Lines~\ref{alg-sd:start-ate-est}--\ref{alg-sd:end-ate-est} and \emph{isLocallyValid} in Fig.~\ref{fig:algo-common-additional-functions}).
The ATE estimated using each such parent set $\S$
is denoted as $\theta_{X \rightarrow Y|\S }$ (Line~\ref{alg-sd:ate-estimate}).
Consider the DAG in Fig.~\ref{fig:pc-true-graph}.
Here, SD will use \emph{LocalPC} to sequentially discover the 
neighbors of $X, W, M, A$, and $B$.
Then the UC $A \rightarrow W \leftarrow B$ will be
detected and, after propagating orientations via Meek's rules, both
$W$ and $M$ will be oriented and SD will terminate (Fig.~\ref{fig:SD-4})
and output $\Theta_{\text{SD}} = \{ \theta_{X \rightarrow Y|\{W\}} \} = \Theta^*$.

\begin{figure}
\centering
\subfigure[True DAG $\GM^{*}$]{\includegraphics[scale=0.42]{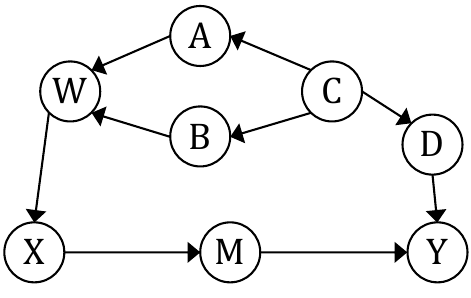}\label{fig:pc-true-graph}}
\hfill
\subfigure[CPDAG from PC]{\includegraphics[scale=0.42]{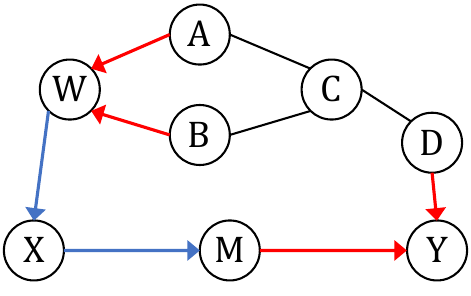}\label{fig:pc-cpdag}}
\hfill
\subfigure[SD's output]{
\hspace{0.1em}
\includegraphics[scale=0.42]{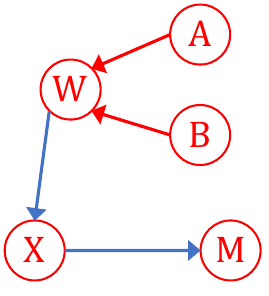}
\hspace{0.1em}
\label{fig:SD-4}}
\hfill
\subfigure[LDECC]{\includegraphics[scale=0.42]{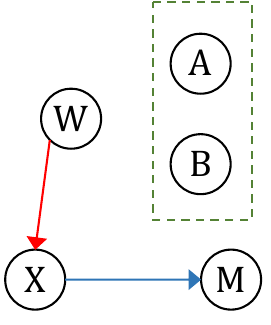}\label{fig:ldecc-graph}}
\hfill
\caption{Demonstration of the PC, SD, and LDECC algorithms for the graph in \hyperref[fig:pc-true-graph]{(a)}.}
\end{figure}

\begin{figure}[t]
\centering
\begin{minipage}[b]{0.48\textwidth}
    \setlength{\interspacetitleruled}{0pt}%
    \setlength{\algotitleheightrule}{0pt}%
    \LinesNumberedHidden
    \begin{algorithm}[H]
    \SetAlgoLined
    \SetKwFunction{FLocalPC}{LocalPC}
    \SetKwProg{FLPC}{def}{:}{}
    \FLPC{\FLocalPC{Skeleton $\UM$, Target $V$}}{
        $s \gets 0$\;
        \While{$|\DNe_{\UM}(V)| > s$}{
            \For{$B \in \DNe_{\UM}(V)$, $\S \subseteq \DNe_{\UM}(V) \setminus \{ B \}$ s.t. $|\S| = s$}{
                \uIf{$V \indep B | \S$}{
                    $\UM$.removeEdge($V \text{---} B$)\;
                    $\text{DSep}(V, B) \gets \S$\;
                    \textbf{break}\;
                }
            }
            $s \leftarrow s + 1$\;
        }
        \KwRet $\UM, \DNe_{\UM}(V), \text{DSep}$\;
    }
    \SetKwFunction{FLocallyValid}{isLocallyValid}
    \SetKwProg{FLV}{def}{:}{}
    \FLV{\FLocallyValid{$X, \S$}}{
        \For{every $A, B \in \S$}{
            \lIf{isNonCollider($A \text{---} X \text{---} B$)}{\textbf{return} False}
        }
        \KwRet True\;
    }
    \end{algorithm}
\captionof{figure}{Additional subroutines.}
\label{fig:algo-common-additional-functions}
\end{minipage}
\hfill
\begin{minipage}[b]{0.49\textwidth}
    \setlength{\interspacetitleruled}{0pt}%
    \setlength{\algotitleheightrule}{0pt}%
    \begin{algorithm}[H]
    \SetAlgoLined
    \KwInput{Treatment $X$.}
    Fully connected undirected graph $\UM$\;
    queue $\gets [X]$, done $\gets \emptyset$\;
    \While{queue is not empty}{ \label{alg-sd:start-seq}
        $V \gets \text{queue.removeFirstItem}()$\; \label{alg:sd-queue-pop}
        $\UM, \DNe_{\UM}(V), \text{DSep} \gets \text{LocalPC}(\UM, V)$\;
        $\text{done} \leftarrow \text{done} \cup \{V\}$\;
        queue.append($\DNe_{\UM}(V) \setminus (\text{done} \cup \text{queue})$)\; \label{alg-sd:end-seq}
        $\GM \gets \text{GetCPDAG}(\UM[\text{done}], \text{DSep})$\; \label{alg-sd:orient-in-subgraph}
        parents, children, unoriented $\gets$ Nbrs($\GM, X$)\;
        \lIf{unoriented $= \emptyset$}{
            break
        }
    }
    $\Theta_{\text{SD}} \gets \emptyset$\; \label{alg-sd:start-ate-est}
    \For{$\S \subseteq (\text{unoriented} \cup \text{parents})$}{
        \uIf{isLocallyValid($X, \S$)}{
            $\Theta_{\text{SD}}$.add$(\theta_{X \rightarrow Y|\S})$\; \label{alg-sd:ate-estimate}
        }
    } \label{alg-sd:end-ate-est}
    \KwRet $\Theta_{\text{SD}}$\;
    \end{algorithm}
\captionof{figure}{The SD algorithm.}
\label{fig:algo-sd}
\end{minipage}
\end{figure}

\section{Local Discovery using Eager Collider Checks (LDECC)}
\label{sec:ldecc}
\begin{figure}[t]
\centering
\begin{minipage}[b]{0.46\textwidth}
    \vspace{0pt}
    \setlength{\interspacetitleruled}{0pt}%
    \setlength{\algotitleheightrule}{0pt}%
    \begin{algorithm}[H]
    \SetAlgoLined
    \SetKwFunction{FOrientChildren}{UCChildren}
    \SetKwProg{FOC}{def}{:}{}
    \FOC{\FOrientChildren{$\MB(X), \DNe(X)$, DSep}}{
        spousesX $\gets \MB(X) \setminus \DNe(X)$\;
        children $\gets \emptyset$\;
        \For{$D \in$ spousesX}{
            \For{$C \in \DNe(X) \setminus \text{DSep}(D, X)$}{
                \lIf{$C \notindep D | \text{DSep}(D, X)$}{
                    children.add($C$)
                }
            }
        }
        \KwRet children\;
    }
    \SetKwFunction{FGetMNS}{GetMNS}
    \SetKwProg{FGM}{def}{:}{}
    \FGM{\FGetMNS{$V \notin \DNep(X)$}}{
        $s \gets 0$\;
        \While{|\DNe(X)| > s}{
            \For{$\S \subseteq \DNe(X)$ s.t. $|\S| = s$}{ 
                \lIf{$V \indep X | \S$}{ \label{alg-get-mns:check-subsets}
                    \KwRet $\S$
                }
            }
            $s \gets s + 1$\;
        }
        \KwRet MNS not found\;
    }
    \SetKwFunction{FECCParents}{ECCParents}
    \SetKwProg{FEP}{def}{:}{}
    \FEP{\FECCParents{$A, B$, check=\emph{False}}}{
        \uIf{check \emph{and} $\{A, B\} \cap \DNe(X) = \emptyset$ }{ \label{alg-ecc-par:check-start}
            $m_A \gets \text{GetMNS}(A)$ \;
            $m_B \gets \text{GetMNS}(B)$\;
            \lIf{$m_A = m_B$}{\KwRet $m_A$}
            \lElse{\KwRet $\emptyset$} \label{alg-ecc-par:check-end}
        }
        parents $\gets \emptyset$\;
        \For{$V \in \{ A, B \}$}{
            \lIf{$V \in \DNe(X)$}{parents.add($V$)}
            \lElse{parents.add(GetMNS$(V)$)}
        }
        \KwRet parents\;
    }
    \end{algorithm}
    \captionof{figure}{Subroutines used by LDECC.}
    \label{fig:algo-ldecc-functions}
\end{minipage}
\hfill
\begin{minipage}[b]{0.50\textwidth}
    \vspace{0pt}
    \setlength{\interspacetitleruled}{0pt}%
    \setlength{\algotitleheightrule}{0pt}%
    \begin{algorithm}[H]
    \SetAlgoLined
    \KwInput{Treatment $X$.}
    $\MB(X) \gets \text{FindMarkovBlanket}(X)$\; \label{alg-ldecc:find-MB-x}
    $\_, \DNe(X), \text{DSep} \gets \text{LocalPC}(\MBp(X), X)$\; \label{alg-ldecc:find-Ne-x}
    children $\gets$ UCChildren($\MB(X), \DNe(X)$, DSep)\; \label{alg-ldecc:find-children-paofch}
    parents $\gets \emptyset$, unoriented $\gets \emptyset$\;
    Completely connected undirected graph $\UM$\;
    \For{$(A \indep B | \S) \in \text{PCTest}(\UM)$  s.t. $A, B \neq X$}{ \label{alg-ldecc:for-loop-pc-tests}
        \uIf{$A, B \in \DNe(X)$ and $X \notin \S$}{ \label{alg-ldecc:if-cond-nbr-uc}
            parents.add($\{A, B\}$)\; \label{alg-ldecc:if-cond-nbr-uc-mark}
            parents.add($\S \cap \DNe(X)$)\; \label{alg-ldecc:meek-rule-3}
        }
        \uElseIf{$A, B \in \DNe(X)$ and $X \in \S$}{ \label{alg-ldecc:if-cond-non-coll}
            MarkNonCollider($A \text{---} X \text{---} B$)\; \label{alg-ldecc:if-cond-non-coll-mark}
            \For{$V \in \DNe(X) \setminus (\S \cup \{A, B\})$}{ \label{alg-ldecc:meek-rule-4-begin}
                \lIf{$A \notindep B | \{ \S, V \}$}{
                    children.add($V$) \label{alg-ldecc:meek-rule-3-and-4-child}
                }
            } \label{alg-ldecc:meek-rule-4-end}
        }
        \ElseIf{$A \notindep B | \{\S, X \}$ and $X \notin \S$}{ \label{alg-ldecc:if-cond-ecc}
            parents.add(ECCParents($A, B, \S$))\; \label{alg-ldecc:if-cond-ecc-mark}
        }
        \For{$P \in $ parents, $C \in$ unoriented}{ \label{alg-ldecc:mark-child-using-non-coll-for-loop}
            \lIf{isNonCollider($P \text{---} X \text{---} C$)}{
                    children.add($C$) 
                } \label{alg-ldecc:mark-child-using-non-coll}
        }
        unoriented $\leftarrow \DNe(X) \setminus $ (parents $\cup$ children)\;
        \lIf{unoriented $= \emptyset$}{
            break
        }
    }
    $\Theta_{\text{LDECC}} \gets \emptyset$\;
    \For{$\S \subseteq (\text{unoriented} \cup \text{parents})$}{
        \uIf{isLocallyValid($\S$)}{
            $\Theta_{\text{LDECC}}$.add$(\theta_{X \rightarrow Y|\S})$\;
        }
    }
    \KwOutput{$\Theta_{\text{LDECC}}$}
    \end{algorithm}
    \captionof{figure}{The LDECC algorithm.}
    \label{fig:algo-cdud}
\end{minipage}
\end{figure}

In this section,
we propose LDECC, a local causal discovery algorithm
that orients the parents of $X$
by leveraging UCs differently from
SD (Prop.~\ref{prop:ecc}).
We prove its correctness under the CFA (Thm.~\ref{thm:ldecc-correctness})
and show its complementary nature relative to SD
in terms of computational (Sec.~\ref{sec:comparison-of-tests}) and 
faithfulness (Sec.~\ref{sec:faithfulness}) requirements.
We defer the proofs to Appendix~\ref{sec:apdx-ldecc-ommitted-proofs}.
We first define a \emph{Minimal Neighbor Separator} (MNS)
which plays a key role in LDECC.

\begin{definition}[Minimal Neighbor Separator]
For a DAG $\GM(\V, \E)$ and nodes $X$ and $A \notin \DNep(X)$,
$\nss_X(A) \subseteq \DNe(X)$ is 
the unique set (see Prop.~\ref{prop:mns-unique}) of nodes such that
(i) (d-separation) $A \indep X | \nss_X(A)$, and 
(ii) (minimality) for any $\S \subset \nss_X(A)$, $A \notindep X | \S$.
\end{definition}

For the graph in Fig.~\ref{fig:pc-true-graph}, we have $\nss_X(A) = \nss_X(B) = \nss_X(C) = \{ W \}$ and $\nss_X(Y) = \{W, M\}$.
While an MNS need not exist for every node 
(see Example~\ref{example:apdx-mns-does-not-exist} in Appendix~\ref{apdx:ldecc}),
$\nss_X(V)$ always exists $\forall V \notin \Desc(X)$ 
(which is sufficient for correctness of LDECC):
\begin{proposition}\label{prop:valid-nss-non-desc}
For any node $V \notin (\Desc(X) \cup \DNep(X))$, $\nss_X(V)$ exists and $\nss_X(V) \subseteq \Pa(X)$.
\end{proposition}

\begin{proposition}[Uniqueness of MNS]\label{prop:mns-unique}
For every node $V$ such that $\nss_X(V)$ exists, it is unique.
\end{proposition}

\begin{proposition}[Eager Collider Check]\label{prop:ecc}
For nodes $A, B \in \V \setminus \DNe^{+}(X)$
and $\S \subseteq \V \setminus \{ A, B, X \}$, if
(i) $A \indep B | \S$;
and (ii) $A \notindep B | \S \cup \{X\}$; 
then $A, B \notin \Desc(X)$
and $\nss_X(A), \nss_X(B) \subseteq \Pa(X)$.
\end{proposition}

Prop.~\ref{prop:ecc} suggests a different strategy
for orienting parents of $X$:
if two nodes that are d-separated by $\S$ become d-connected
by $\S \cup \{X \}$, the MNS of
such nodes contains the parents of $X$
(see proof in Appendix~\ref{sec:apdx-ldecc-ommitted-proofs}).
Thus, unlike PC and SD, the skeleton is not needed to
propagate orientations via Meek's rules.
Similar ideas have been used in prior work
for causal discovery.
\citet{claassen2012logical} use minimal (in)dependencies
to develop a logic-based causal discovery algorithm 
that also avoids graphical orientation rules.
\citet{magliacane2016ancestral} extend this and
propose an algorithm that only reasons over ancestral
relations rather than the space of DAGs. 
While these works do not study local causal discovery, their
results can be used to derive
an alternative proof of Prop.~\ref{prop:ecc}.

The LDECC algorithm (Fig.~\ref{fig:algo-cdud})
first finds the Markov blanket $\MB(X)$ (Line~\ref{alg-ldecc:find-MB-x}):
we use the IAMB algorithm \citep[Fig.~2]{tsamardinos2003algorithms} in our experiments (see Fig.~\ref{fig:apdx-algo-iamb} in Appendix~\ref{sec:apdx-ldecc-ommitted-proofs}).
Next, we run \emph{LocalPC} within the subgraph containing $\MB(X)$
to prune this set to obtain $\DNe(X)$ (Line~\ref{alg-ldecc:find-Ne-x}).
We then use the nodes in $\MB(X) \setminus \DNe(X)$ to
detect children $C$ of $X$
that are part of
UCs of the form $X \rightarrow C \leftarrow D$ 
(Line~\ref{alg-ldecc:find-children-paofch} and \emph{GetUCChildren} in Fig.~\ref{fig:algo-ldecc-functions}).
LDECC then starts running CI tests in the same way
as PC would but 
excluding tests for 
$X$ since we already know $\DNe(X)$
(the for-loop in Line~\ref{alg-ldecc:for-loop-pc-tests};
see Fig.~\ref{fig:apdx-pc-algorithm} for \emph{PCTest}).
Every time we detect a CI 
$A \indep B | \S$,
we check the following cases:
(i) if $A, B \in \DNe(X)$ and $X \notin \S$,
then there must be a UC $A \rightarrow X \leftarrow B$
and so we mark $A$ and $B$ as parents 
(Lines~\ref{alg-ldecc:if-cond-nbr-uc}, \ref{alg-ldecc:if-cond-nbr-uc-mark});
(ii) if $A, B \in \DNe(X)$ and $X \in \S$,
then we mark $A \text{---} X \text{---} B$ as a non-collider 
(Lines~\ref{alg-ldecc:if-cond-non-coll}, \ref{alg-ldecc:if-cond-non-coll-mark});  and
(iii) \textbf{Eager Collider Check (ECC):} if $X \notin \S$ and $A \notindep B | \S \cup \{ X \}$,
then, by leveraging Prop.~\ref{prop:ecc},
we apply \emph{GetMNS} (Fig.~\ref{fig:algo-ldecc-functions}) to
mark $\nss_X(A), \nss_X(B)$ as parents 
(Lines~\ref{alg-ldecc:if-cond-ecc}, \ref{alg-ldecc:if-cond-ecc-mark}; \emph{ECCParents} in Fig.~\ref{fig:algo-ldecc-functions}).
Next, for each oriented parent,
we use the non-colliders detected in Case~(ii)
to mark children 
(Lines~\ref{alg-ldecc:mark-child-using-non-coll-for-loop}, \ref{alg-ldecc:mark-child-using-non-coll}).
LDECC terminates if there are no unoriented neighbors.
The ATE set is computed similarly to SD by applying the
backdoor adjustment with every \emph{locally valid} parent set.
Lines~\ref{alg-ldecc:meek-rule-3},\ref{alg-ldecc:meek-rule-4-begin}--\ref{alg-ldecc:meek-rule-4-end}
are inserted to account 
for Meek's rules $3, 4$.
The \emph{check=True} argument in \emph{ECCParents} aims to make
it more robust to faithfulness violations 
(see Remark~\ref{remark:weaker-ff-ldecc-check-true} in Sec.~\ref{sec:faithfulness}).

Consider the running example of the DAG in Fig.~\ref{fig:pc-true-graph}.
We first find the local structure 
around $X$. 
Then we start running CI tests in the same way as PC. 
After we run $A \indep B | C$,
we do an ECC and find $A \notindep B | \{C, X\}$.
Since $\nss_X(A) = \nss_X(B) = \{W\}$,
we mark $W$ as a parent of $X$.
Next, we find that $W \indep M | X$ and 
mark $W \text{---} X \text{---} M$ as a non-collider
which lets us mark $M$ as a child 
(because $W$ is a parent). 
LDECC terminates as all neighbors of $X$ are oriented (Fig.~\ref{fig:ldecc-graph})
and the ATE set is
computed by using $\{W\}$ as the backdoor adjustment set: $\Theta_{\text{LDECC}} = \{ \theta_{X \rightarrow Y|\{W\}} \} = \Theta^{*}$.

\begin{theorem}[Correctness]\label{thm:ldecc-correctness}
Under the CFA and with access to a CI oracle, we have 
$\Theta_{\text{LDECC}} \overset{\text{set}}{=} \Theta^{*}$.
\end{theorem}

\begin{remark*}
A DAG can have multiple valid adjustment sets.
We show that local information suffices for
identifying the optimal adjustment set
according to the asymptotic variance criterion
in \citet{henckel2019graphical} and
we present a procedure to discover it (see Fig.~\ref{fig:apdx-algo-optimal-adjustment-set}, 
Prop.~\ref{prop:apdx-finding-the-OAS} in Appendix~\ref{sec:apdx-adjustment-oas}).
\end{remark*}

In summary, LDECC differs from SD
in two key ways:
(i) LDECC runs tests in a different order---SD runs \emph{LocalPC} 
by recursively considering nodes starting from $X$ whereas LDECC
runs CI tests in the same order as PC; and
(ii) SD uses the skeleton to orient edges from a UC
via Meek's rules 
(like the PC algorithm) 
whereas LDECC uses an ECC
to orient the parents of $X$.

\subsection{Comparison of computational requirements}\label{sec:comparison-of-tests}
\begin{figure}
\centering
\subfigure[LDECC outperforms SD]{\includegraphics[scale=0.38]{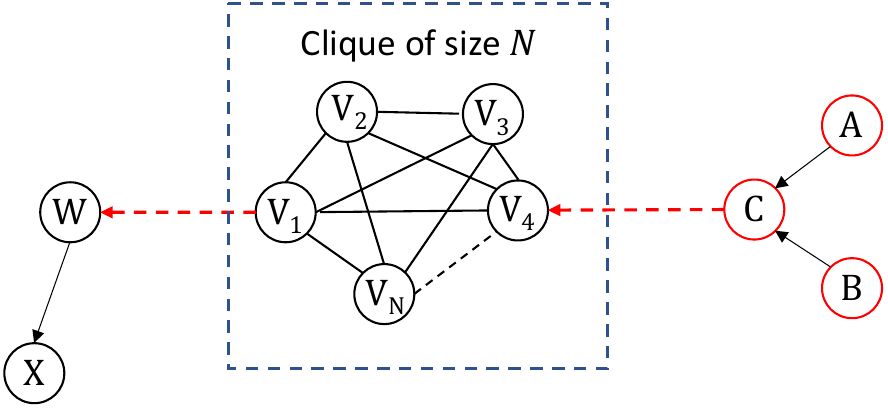}\label{fig:ldecc-better-graph}}
\hfill
\subfigure[SD outperforms LDECC]{\includegraphics[scale=0.38]{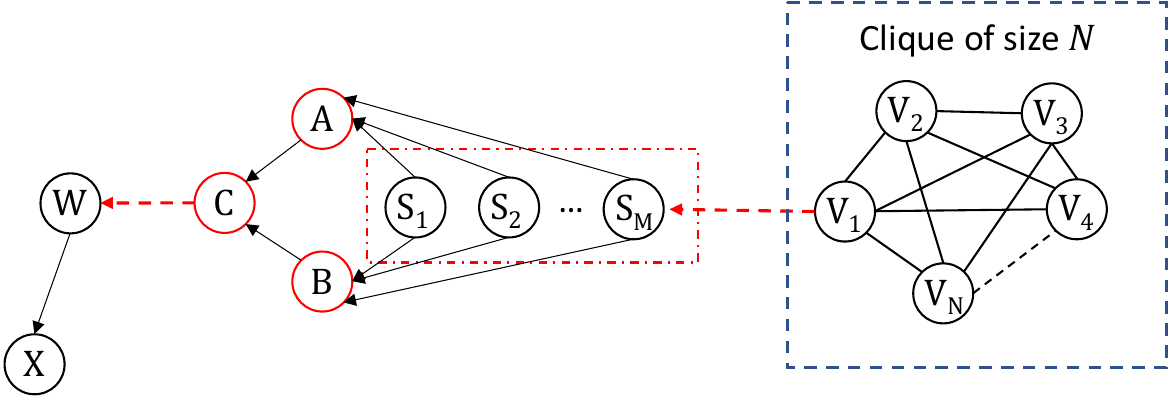}\label{fig:sd-better-graph}}
\caption{Classes of graphs where LDECC and SD have different runtimes.}
\label{fig:comparison-of-tests}
\end{figure}

In this section, we compare PC, SD, and LDECC based 
on the number of CI tests they perform.
There are classes of DAGs where LDECC is
polynomial time while SD is not, and vice versa (Props.~\ref{prop:ldecc-exp-better},\ref{prop:sd-exp-better}).
Using this insight, we construct an algorithm
that runs in polynomial time on a larger class of DAGs (Remark~\ref{remark:hybrid-algorithm}).
We defer proofs to Appendix~\ref{sec:apdx-computational}.
We make the following assumption in this section:
\begin{assumption}
The CFA holds and we have access to a CI oracle.
\end{assumption}
Let the number of CI tests performed by PC, SD, and LDECC be $T_{\text{PC}}$, $T_{\text{SD}}$, 
and  $T_{\text{LDECC}}$. 
We first show that, in the worst case, LDECC only runs a polynomial (in $|\V|$) number of 
extra CI tests compared to PC when $|\DNe(X)|$ is bounded.
In practice, we observe that the upper bound is quite loose:
Even when $|\DNe(X)|$ is large, PC also has to run a large
number of tests for $X$.
\begin{proposition}[PC vs LDECC]\label{prop:num-tests-pc-vs-ldecc}
We have $T_{\text{LDECC}} \leq T_{\text{PC}} + \OM(|\V|^2) + \OM\left(|\V| \cdot 2^{|\DNe(X)|}\right).$
\end{proposition}

\begin{proposition}[LDECC exponentially better]\label{prop:ldecc-exp-better}
There exist classes of graphs such that $T_{\text{LDECC}} + \Theta(2^{|\V|}) \leq T_{\text{PC}}; \,\, \text{and}  \,\, T_{\text{LDECC}} + \Theta(2^{|\V|}) \leq T_{\text{SD}}$.
\end{proposition}
\begin{proof}
Consider the class of graphs shown in Fig.~\ref{fig:ldecc-better-graph}
where there is a clique of size $N$ 
on the path from the UC $A \rightarrow C \leftarrow B$ to $W$.
Both PC and SD will perform $\Theta(2^{N})$ CI tests
due to this clique. 
Thus, when $N \asymp |\V|$, SD is exponential time.
By contrast, since LDECC runs tests in the same order as PC,
it unshields the UC, 
orients $W$ as a parent via an ECC, and terminates 
in $\OM(|\V|^2)$ tests.
\end{proof}
Qualitatively, the result shows that if there is a
dense region between a UC and the parent it orients,
SD might perform poorly because it has to wade through
this dense region to get to the UC.
By contrast, LDECC can avoid the CI tests in these dense regions 
as it uses an ECC to orient parents of $X$
instead of the skeleton.
However, LDECC does \emph{not} uniformly dominate SD.
There are classes of
DAGs where LDECC performs 
exponentially more CI tests than SD:
\begin{proposition}[SD exponentially better]\label{prop:sd-exp-better}
Let $\U$ be the set of UCs in $\GM^*$
and $M = \max_{U \in \U} \text{sep}(U)$.
There exist classes of graphs such that $T_{\text{SD}} + \Theta(2^{M-1}) \leq T_{\text{PC}};$ and $T_{\text{SD}} + \Theta(2^{M-1}) \leq T_{\text{LDECC}}$.
\end{proposition}
\begin{proof}
Consider the class of graphs shown in Fig.~\ref{fig:sd-better-graph}
with a UC with a separating set of size $M$ and a clique of size $N$ upstream of that UC.
Since LDECC runs CI tests in the same order as PC, 
if $N \geq M$, LDECC performs $\Theta(2^{M-1})$
CI tests for nodes inside the clique (all tests of the form 
$V_i \indep V_j | \S$ s.t. $\S \subseteq \{ V_1, \hdots, V_N \} \setminus \{V_i, V_j\}$ and $|\S| < M$ will be run)
before the UC is unshielded via the test $A \indep B | \{ S_1, \hdots, S_M \}$.
By contrast, SD terminates before even getting to the clique,
thus avoiding these tests.
So if $M \asymp |\V|$, LDECC will perform exponentially more tests than SD.
\end{proof}
Qualitatively, this shows that if UCs have large separating sets
and there are dense regions upstream of such UCs, LDECC might
perform poorly whereas SD can avoid tests in these regions since it 
would not reach these dense regions.

Next, we present upper bounds for the computational
requirements of
LDECC and SD for a restricted class of graphs that
we call \emph{Locally Orientable Graphs}.

\begin{definition}[Locally Orientable Graph.]\label{defn:locally-orientable}
A DAG $\GM$ is said to be \emph{locally orientable} 
with respect to a node $X$ if, $\forall V \in \DNe(X)$,
the edge $X \text{---} V$ is oriented 
in the CPDAG for the MEC of $\GM$.
Since all nodes in $\DNe(X)$ can be oriented, 
the ATE is point identified, i.e., $|\Theta^*| = 1$.
\end{definition} 

\begin{definition}[Parent Orienting Collider (POC)]\label{defn:parent-orienting-collider}
Consider a DAG $\GM^*$.
For each $P \in \Pa(X)$,
let $\text{POC}(P)$ be the set of
UCs that can orient $P$, i.e., either
(i) $P$ is a part of that UC; or
(ii) applying Meek's rules from that UC in the undirected skeleton
of $\GM^*$ will orient $P \rightarrow X$.
\end{definition}

In Props.~\ref{prop:LDECC-upper-bound},\ref{prop:sd-upper-bound}, 
we provide sufficient graphical conditions for
when both the algorithms run in polynomial time.
For LDECC, we show that
the time complexity of orienting the parents depends only on
the size of the separating sets for the UCs, i.e., \emph{sep} (Prop.~\ref{prop:LDECC-upper-bound}).
Importantly, the bound
is agnostic to how far away the UCs
are from the parents.

\begin{proposition}[LDECC upper bound]\label{prop:LDECC-upper-bound}
For a locally orientable DAG $\GM^{*}$, 
let $D = \max_{V \in \MBp(X)} \\ |\DNe(V)|$ and $S = \max_{ P \in \Pa(X)} \min_{\alpha \in \text{POC}(P)} \text{sep}(\alpha)$.
Then $T_{\text{LDECC}} \leq \OM(|\V|^{\max\{ S, D \}})$.
\end{proposition}

For SD, we show that
the time complexity of orienting a parent is upper bounded
by (i) the \emph{sep} of the closest UC 
that orients that parent and 
(ii) the complexity of discovering the neighbors 
of nodes on the path to that collider (Prop.~\ref{prop:sd-upper-bound}).
Importantly, this bound is agnostic to the structure of the nodes
that are upstream of the closest colliders.

\begin{proposition}[SD upper bound]\label{prop:sd-upper-bound}
For a locally orientable DAG $\GM^{*}$,
let $\pi : \V \rightarrow \mathbb{N}$ be
the order in which nodes are processed by SD
(Line~\ref{alg:sd-queue-pop} of SD).
For $P \in \Pa(X)$, let
$\CUC(P) = \argmin_{\alpha \in \text{POC}(P)}\pi(\alpha)$
denote the closest UC
to $P$.
Let $C = \max_{ P \in \Pa(X)} \text{sep}(\CUC(P))$,
$D = \max_{V \in \MBp(X)} |\DNe(V)|$, and
$E = \max_{ \{ V : \pi(V) < \pi(\CUC(P)) \} } |\DNe(V)|$.
Then $T_{\text{SD}} \leq \OM(|\V|^{\max\{C, D, E \}})$.
\end{proposition}

\begin{remark}\label{remark:hybrid-algorithm}
The following procedure
demonstrates one way that
LDECC and SD can be combined to broaden the class of graphs 
where local causal discovery can be performed efficiently:
Run both SD and LDECC in parallel and 
terminate when either algorithm terminates.
The number of CI tests performed will be 
$T_{\text{combined}} \leq 2 \min \left\{ T_{\text{SD}}, T_{\text{LDECC}}  \right\}$.
This procedure 
will be subexponential 
if at least one of the algorithms is subexponential.
\end{remark}

\subsection{Comparison of faithfulness requirements}\label{sec:faithfulness}
In this section, we show that SD and LDECC rely on
different faithfulness assumptions (Props.~\ref{prop:sd-ff},\ref{prop:ldecc-ff}).
The CFA is controversial \citep{uhler2013geometry} and
some works attempt to test or weaken it \citep{zhang2008detection, ramsey2012adjacency, zhang2016three}.
Sharing a similar motivation, we show that 
LDECC and SD can be fruitfully combined to
strictly weaken the faithfulness assumptions needed to identify $\Theta^*$ (Prop.~\ref{prop:hybrid-faithfulness}).
We defer the proofs to Appendix~\ref{sec:apdx-faithfulness}.

\begin{definition}[Adjacency Faithfulness (AF)]
Given a DAG $\GM(\V, \E)$, AF holds for node $A \in \V$ iff $\,\forall A \text{---} B \in \E$, we have $\forall \S \subseteq \V \setminus \{A, B\}, \,\, A \notindep B | \S$.
\end{definition}

\begin{definition}[Orientation Faithfulness (OF)]
Given a DAG $\GM(\V, \E)$, OF holds for an unshielded triple $A \text{---} C \text{---} B$ (where $A \text{---} B \notin E$) iff
(i) when it is a UC,
$A \notindep B | \S$ for any $\S \subseteq \V \setminus \{A, B\}$ s.t. $C \in \S$; and
(ii) when it
is a non-collider, 
$A \notindep B | \S$ for any $\S \subseteq V \setminus \{A, B\}$ s.t. $C \notin \S$.
\end{definition}

\begin{definition}[Local Faithfulness (LF)]
Given a DAG $\GM(\V, \E)$,
LF holds for node $X$ iff 
(i) AF holds for $X$; and
(ii) OF holds for every unshielded triple $A \text{---} X \text{---} B \in \GM$.
\end{definition}

\begin{proposition}\label{prop:pc-mec}
PC will identify the MEC of $\GM^*$ if LF holds for all nodes.
\end{proposition}

If the goal is to identify $\Theta^{*}$
(and not the entire MEC), weaker faithfulness conditions than those of Prop.~\ref{prop:pc-mec} 
are sufficient for SD and PC (Prop.~\ref{prop:sd-ff}).
Let $\UM$ be the undirected graph on the nodes $\V$
such that an edge $A \text{---} B$ exists iff 
$\forall \S \subseteq \V \setminus \{A, B\}, \, A \notindep B | \S$.
Informally, $\UM$ is the skeleton of the observational distribution.
Let $J^* = \{ (A \rightarrow C \leftarrow B) \in \GM^* \}$ be the set of UCs
in the true DAG $\GM^*$.

\begin{proposition}[Faithfulness for PC and SD]\label{prop:sd-ff}
PC and SD will identify $\Theta^{*}$ if
(i) LF holds $\forall V \in \MBp(X)$;
(ii) $\forall (A \rightarrow C \leftarrow B) \in J^*$,
(a) LF holds for $A$, $B$, and $C$, and
(b) LF holds for each node on all paths $C \rightarrow \hdots \rightarrow V \in \GM^*$ 
s.t. $V \in \DNe(X)$;
(iii) For every edge $A\text{---}B \notin \GM^*$, $\exists \S \subseteq (\DNe_{\UM}(A) \cup \DNe_{\UM}(B))$ s.t. $A \indep B | \S$; and
(iv) OF holds for all unshielded triples in $\GM^*$.
\end{proposition}
\begin{proofsketch}
By Conditions~(i,ii), UCs in $\GM^*$ are 
detected and orientations are
propagated correctly to $X$. 
By Condition~(iii), the skeleton discovered by
PC and SD is a subgraph of the skeleton of $\GM^*$
and thus there are no spurious paths.
By Condition~(iv), no spurious UCs are detected.
Prop.~\ref{prop:sd-ff} is weaker because
AF can be violated
for nodes not on paths from UCs to $X$.
\end{proofsketch}

We now present sufficient faithfulness conditions for LDECC.
Let $\L = \{ (A, B, \S) : |\{A, B\} \cap \DNe(X)| \leq 1, \, \S \subseteq \V \setminus \{A, B, X\} \}$ and $H^{*} = \{ (A, B, \S) \in \L : A \indep_{\GM^*} B | \S, \, A \notindep_{\GM^*} B | \S \cup \{X\} \}$,
where $\indep_{\GM^*}$ ($\notindep_{\GM^*}$) denotes d-separation (d-connection).
Informally, $H^*$ contains the set of ECCs entailed by the true DAG $\GM^*$.
Analogously, let
$H = \{ (A, B, \S) \in \L : A \indep B | \S, \, A \notindep B | \S \cup \{X\} \}$.
Informally, $H$ contains the set of ECCs entailed by 
the observational joint distribution $\mathbb{P}(\V)$.

\begin{definition}[MNS Faithfulness (MFF)]\label{defn:mff}
Given a DAG $\GM(\V, \E)$,
MFF holds for node $V \notin \DNep(X)$ iff
(a) When $\nss_X(V)$ does \emph{not} exist, 
$\forall \S \subseteq \DNe(X)$, we have $V \notindep X | \S$; and
(b) When $\nss_X(V)$ exists,
$\forall \S \subseteq \DNe(X)$ s.t. $\S \neq \nss_X(V)$
and $V \indep X | \S$, we have $|\S| > |\nss_X(V)|$.
\end{definition}

\begin{proposition}[Faithfulness for LDECC]\label{prop:ldecc-ff}
LDECC will identify $\Theta^{*}$  
if (i) LF holds $\forall V \in \MBp(X)$;
(ii) $H \subseteq H^*$;
(iii) $\forall (A, B, \S) \in H$, MFF holds for $\{A, B\} \setminus \DNe(X)$; and
(iv) $\forall (A, B, \S) \in H^*$ s.t. there is a UC $(A \rightarrow C \leftarrow B) \in \GM^*$,
we have (a) AF holds for $A$ and $B$; and (b) $(A, B, \S) \in H$.
\end{proposition}
\begin{proofsketch}
By Condition~(i), the structure within $\MBp(X)$ is correctly discovered.
By Condition~(ii), every ECC run by LDECC will be valid (since it is also part of $H^*$),
and together with Condition~(iii), every ECC will mark the correct parents.
By Condition~(iv), an ECC will be run for every UC that can
orient parents of $X$, ensuring that all parents get oriented:
(iv)(a) ensures that the UC gets unshielded, and 
(iv)(b) ensures that an ECC is run in Line~\ref{alg-ldecc:if-cond-ecc-mark} after the UC is unshielded.
\end{proofsketch}

\begin{remark}\label{remark:weaker-ff-ldecc-check-true}
In LDECC, if we run \emph{ECCParents} with 
\emph{check=True}
where we mark parents with an ECC only if $\text{GetMNS}(A) = \text{GetMNS}(B)$ (Lines~\ref{alg-ecc-par:check-start}--\ref{alg-ecc-par:check-end} in Fig.~\ref{fig:algo-ldecc-functions}),
then LDECC requires weaker assumptions than Prop.~\ref{prop:ldecc-ff} (see Prop.~\ref{prop:apdx-ldecc-weaker-ff} in Appendix~\ref{sec:apdx-faithfulness})
and does better empirically (Sec.~\ref{sec:experiments}).
\end{remark}

We demonstrate that SD and LDECC rely on
different faithfulness assumptions: 
There are faithfulness violations that affect one algorithm
but not the other.
This motivates a procedure that outputs a (conservative)
ATE set containing the values in $\Theta^*$ under weaker assumptions (Prop.~\ref{prop:conservative-ate}).

\begin{figure}
\centering
\subfigure[DAG with unfaithfulness]{\hspace{1em}\includegraphics[scale=0.42]{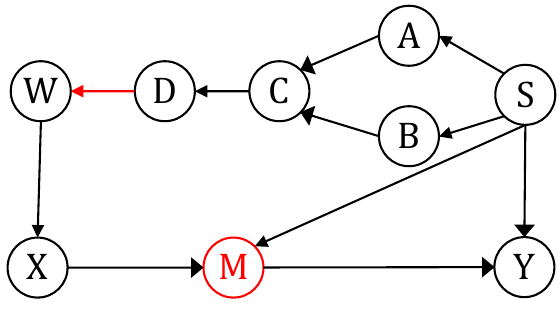}\label{fig:ff-violation}\hspace{1em}}
\hfill
\subfigure[Testable violation for SD]{\hspace{2.2em}\includegraphics[scale=0.42]{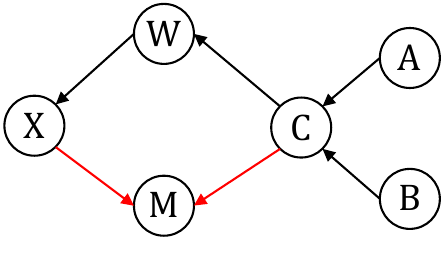}\label{fig:ff-violation-testable-sd}\hspace{2.2em}}
\hfill
\subfigure[MFF fails for nodes $A, B$]{\hspace{1.5em}\includegraphics[scale=0.42]{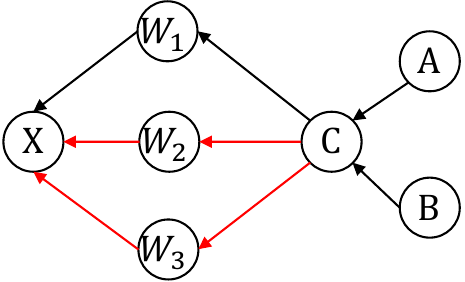}\label{fig:mff-fail-node-A}\hspace{1.5em}}
\caption{Different faithfulness violations.}
\end{figure}

\begin{example}[SD incorrect]
Consider the faithfulness violation $W \indep D | M$ for the DAG in Fig.~\ref{fig:ff-violation}
(conditioning on the collider $M$ cancels out the $W \leftarrow D$ edge).
Thus, the edge $W \leftarrow D$ might be incorrectly removed
and both SD and PC will not mark $W$ as a parent
since the orientations will not be propagated to $W$.
However, LDECC can still correctly mark $W$ as a
parent using an ECC.
\end{example}

\begin{example}[LDECC incorrect]\label{example:ldecc-violation}
Consider the faithfulness violations $A \indep X | M$ and $B \indep X | M$
for the DAG in Fig.~\ref{fig:ff-violation}.
We can incorrectly find $\nss_X(A) = \nss_X(B) = \{ M \}$
(if Line~\ref{alg-get-mns:check-subsets} in \emph{GetMNS}
tests $A \indep X | M$ or $B \indep X |M$)
causing LDECC to orient $M$ as a parent via an ECC.
But these violations do not impact SD and PC
since the skeleton and UCs are still discovered correctly.
\end{example}

\begin{proposition}[Conservative ATE set]\label{prop:conservative-ate}
Consider the following procedure:
(i) run both SD and LDECC to get $\Theta_{\text{SD}}$
and $\Theta_{\text{LDECC}}$;
(ii) output the set $\Theta_{\text{union}} = \Theta_{\text{LDECC}} \cup \Theta_{\text{SD}}$.
If the faithfulness assumptions for either SD or LDECC hold,
then $\Theta^* \subseteq \Theta_{\text{union}}$.
\end{proposition}

Next, we show that some violations that affect
SD and LDECC are detectable.
We use this insight to construct a procedure
(Prop.~\ref{prop:hybrid-faithfulness})
that identifies $\Theta^*$ under 
strictly weaker faithfulness assumptions
by switching to LDECC if a violation of the assumptions 
of SD is detected (or vice versa).

\begin{figure}[t]
\centering
\begin{minipage}[b]{0.53\textwidth}
    \setlength{\interspacetitleruled}{0pt}%
    \setlength{\algotitleheightrule}{0pt}%
    \begin{algorithm}[H]
    \SetAlgoLined
    \KwInput{Node $A \notin \DNep(X)$, $\DNe(X)$}
    $\Q \gets \{ \S \subseteq \DNe(X) : A \indep X | \S \}$\; \label{alg:test-mff-additional-tests}
    \lIf{$|\Q| = 0$}{
        \textbf{return} noValidMNS
    }
    
    $\Q_{\min} \gets \{ \S \in \Q : \forall \S' \subset \S,  \S' \notin \Q \}$\;
    \lIf{$|\Q_{\min}| > 1$}{
        \textbf{return} Fail \label{alg:test-mff-failure-1}
    }
    $\S \gets $ GetElement($\Q_{\min}$) \tcp*{$\text{Here} \, |\Q_{\min}| = 1$.} \label{alg:test-mff-q-min-is-size-one}
    \For{$\S' \in \Q$ such that $\S \subset \S'$}{
        \lIf{$\exists \S'' \notin \Q$ s.t. $\S \subset \S'' \subset \S'$}{
            \textbf{return} Fail
        }\label{alg:test-mff-failure-2}
    }
    \textbf{return} Unknown\; \label{alg:test-mff-end}
    \end{algorithm}
    \captionof{figure}{Testing MFF.}
    \label{fig:algo-test-mff}
\end{minipage}
\hfill
\begin{minipage}[b]{0.42\textwidth}
    \vspace{0pt}
    \setlength{\interspacetitleruled}{0pt}%
    \setlength{\algotitleheightrule}{0pt}%
    \begin{algorithm}[H]
    \SetAlgoLined
    \KwInput{UC $\alpha=(A \rightarrow C \leftarrow B)$}
    $\Q(A) \gets \{ \S \subseteq \DNe(X) : A \indep X | \S \}$\;
    $\Q(B) \gets \{ \S \subseteq \DNe(X) : B \indep X | \S \}$\;
    \tcp{See Defn.~\ref{defn:parent-orienting-collider} for \emph{POC}.}
    $\M \gets \{ P \in \Pa(X) : \alpha \in \text{POC}(P) \}$\; \label{alg:test-sd-ff-M-poc}
    $\widetilde{M}(A) \gets \{\S \in \Q(A) : \M \subseteq \S\}$\;
    $\widetilde{M}(B) \gets \{\S \in \Q(B) : \M \subseteq \S\}$\;
    \lIf{$|\widetilde{M}(A)| = 0$ or $|\widetilde{M}(B)| = 0$}{
        \textbf{return} Fail \label{alg:test-sd-ff-fail}
    }
    \textbf{return} Unknown\;
    \end{algorithm}
    \captionof{figure}{Testing faithfulness for SD.}
    \label{fig:test-sd-ff}
\end{minipage}
\end{figure}

\begin{proposition}[Testing faithfulness for LDECC]\label{prop:test-mff}
Consider running the algorithm in Fig.~\ref{fig:algo-test-mff}
before invoking \emph{GetMNS}($A$) for some node $A$ in LDECC.
If the algorithm returns \emph{Fail}, MFF is violated for node $A$.
If the algorithm returns \emph{Unknown}, 
we could not ascertain if MFF holds for node $A$.
\end{proposition}
\begin{proofsketch}
If Line~\ref{alg:test-mff-failure-1} returns \emph{Fail},
there are multiple minimal sets thereby 
violating the uniqueness of MNS
(see Example~\ref{example:testing-mff-1}).
If Line~\ref{alg:test-mff-failure-2} returns \emph{Fail},
for the detected MNS $\S$, there is a superset $\S'$ s.t.
$A \indep X | \S'$ but also an intermediate set $\S''$ s.t.
$A \notindep X | \S''$ which cannot happen if MFF holds
because removing nodes from $\S'$ should not violate independence
(see Example~\ref{example:testing-mff-2}).
\end{proofsketch}

\begin{example}[Testing MFF]\label{example:testing-mff-1}
Consider the faithfulness violations $A \indep X | M$ and $B \indep X | M$
for the DAG in Fig.~\ref{fig:ff-violation}.
We might incorrectly detect $\nss_X(A) = \nss_X(B) = \{ M \}$
(similar to Example~\ref{example:ldecc-violation}).
Line~\ref{alg:test-mff-failure-1} in Fig.~\ref{fig:algo-test-mff}
will return \emph{Fail} because for $A, B$, we have $\Q_{\min} = \{ \{W\}, \{M\} \}$.
\end{example}

\begin{example}[Testing MFF]\label{example:testing-mff-2}
Consider the DAG in Fig.~\ref{fig:mff-fail-node-A}.
If a faithfulness violation causes
the paths $X \leftarrow W_2 \leftarrow C$ and $X \leftarrow W_3 \leftarrow C$ to cancel each other out,
we will observe $A \indep X | W_1$ and thus wrongly
detect $\mns_X(A) = \{ W_1 \}$.
But Line~\ref{alg:test-mff-failure-2} in Fig.~\ref{fig:algo-test-mff}
will detect this failure since $A \indep X | \{ W_1, W_2, W_3 \}$ but $A \notindep X | \{ W_1, W_2 \}$ and $A \notindep X | \{ W_1, W_3 \}$.
\end{example}

\begin{proposition}[Testing faithfulness for SD]\label{prop:test-sd-ff}
Consider running the algorithm in Fig.~\ref{fig:test-sd-ff} with
each UC detected by SD.
If the algorithm returns \emph{Fail}, then faithfulness is violated for SD.
If the algorithm returns \emph{Unknown}, we could not
ascertain if the faithfulness assumptions for SD hold.
\end{proposition}
\begin{proofsketch}
In Line~\ref{alg:test-sd-ff-M-poc}, $\M$ contains nodes
marked as parents of $X$ due to the UC
$A \rightarrow C \leftarrow B$. 
In Line~\ref{alg:test-sd-ff-fail}, we check that there is some
superset of $\M$ that makes $A$ and $B$ independent
of $X$:
this must be true if the nodes in $\M$ were actually
parents of $X$ (see Example~\ref{example:testing-sd-ff}).
\end{proofsketch}

\begin{example}[Testing faithfulness for SD]\label{example:testing-sd-ff}
Consider an OF violation
$C \indep X | M$
for the DAG in Fig.~\ref{fig:ff-violation-testable-sd}
that causes $X \rightarrow M \leftarrow C$ to be detected
as a non-collider.
SD will mark $M$ as a
parent via the UC $A \rightarrow C \leftarrow B$.
But Line~\ref{alg:test-sd-ff-fail} will return \emph{Fail}
because $\M = \{ M \}$ and $\Q(A) = \Q(B) = \{ \{W\} \}$.
\end{example}

\begin{proposition}\label{prop:hybrid-faithfulness}
Consider the following procedure: 
(i) run SD to get $\Theta_{\text{SD}}$;
(ii) test for faithfulness violations using Prop.~\ref{prop:test-sd-ff};
(iii) if a violation is found, run LDECC and output $\Theta_{\text{LDECC}}$,
else output $\Theta_{\text{SD}}$.
This procedure identifies $\Theta^*$ under strictly weaker assumptions
than Prop.~\ref{prop:sd-ff}.
\end{proposition}
\begin{proof}
If the conditions of Prop.~\ref{prop:sd-ff} hold,
this procedure will output $\Theta_{\text{SD}}$ 
because we will not detect a faithfulness violation.
However, if a violation is detected, this procedure will still
identify $\Theta^*$ if the faithfulness conditions for LDECC hold.
This is strictly weaker than the assumptions of Prop.~\ref{prop:sd-ff}.
\end{proof}

We can construct a similar procedure for LDECC.
This again shows the complementary nature of the two algorithms:
if there are detectable faithfulness violations for one algorithm, we can output
an alternative ATE set thereby
allowing us to identify $\Theta^*$ under weaker faithfulness assumptions
than would be possible with only one of the algorithms.
Currently, our testing procedures can only
detect some (but not all) violations.
Moreover, faithfulness violations in DAGs do not occur
independently of each other: 
an unfaithful mechanism 
can lead to other (entangled)
faithfulness violations.
As a result, the 
violations affecting SD and LDECC will be entangled
and the given faithfulness assumptions 
may not be the minimal ones
required for local causal discovery.

\section{Experiments}
\label{sec:experiments}
\begin{figure}
\centering
\subfigure[Tests with a CI oracle]{\includegraphics[scale=0.35]{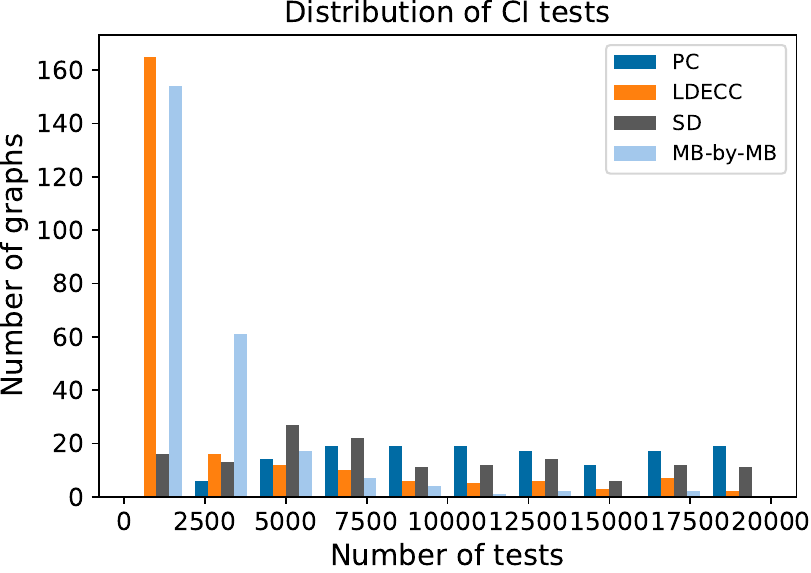}\label{fig:synthetic-with-ci-oracle}}
\hfill
\subfigure[Accuracy]{\includegraphics[scale=0.35]{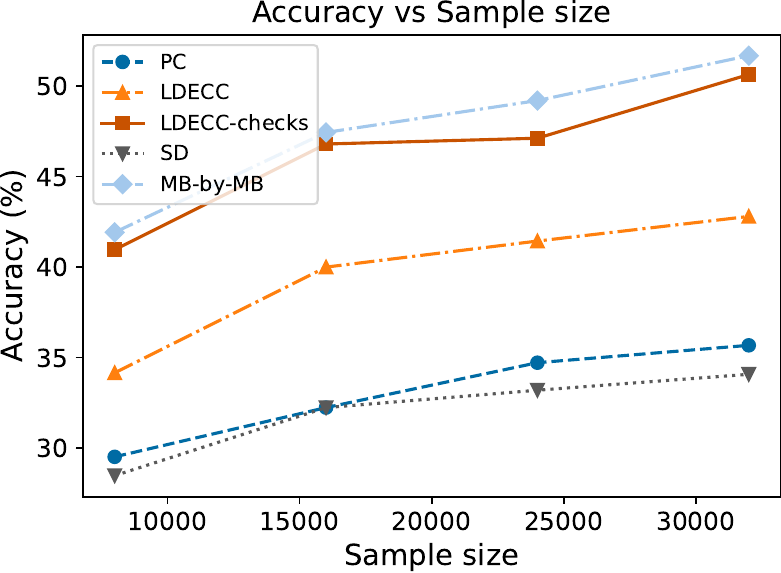}\label{fig:synthetic-finite-sample-accuracy}}
\hfill
\subfigure[Recall]{\includegraphics[scale=0.35]{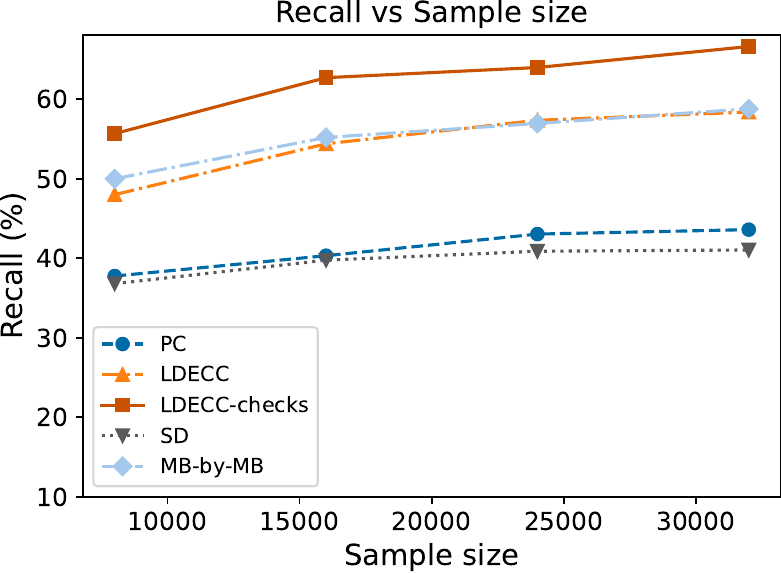}\label{fig:synthetic-finite-sample-recall}}
\\
\subfigure[MSE vs sample size]{\includegraphics[scale=0.33]{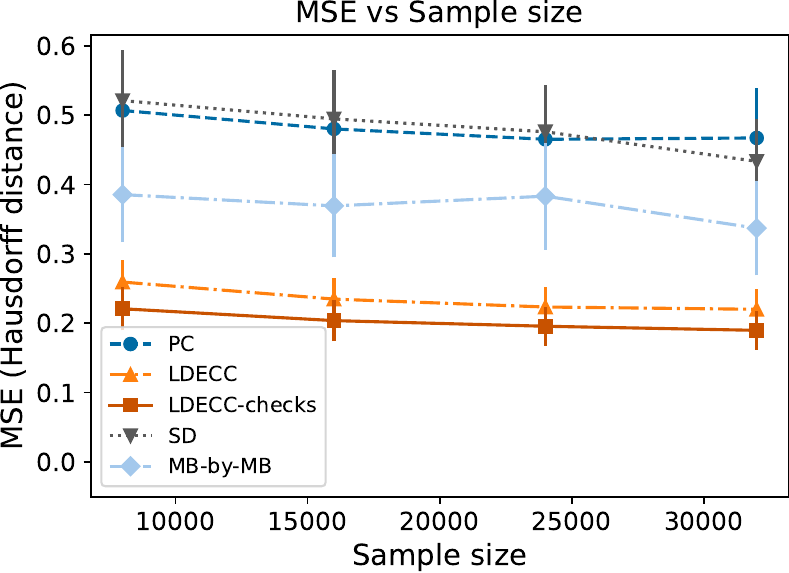}\label{fig:synthetic-finite-sample-mse}}
\hspace{1em}
\subfigure[Number of CI tests]{\includegraphics[scale=0.35]{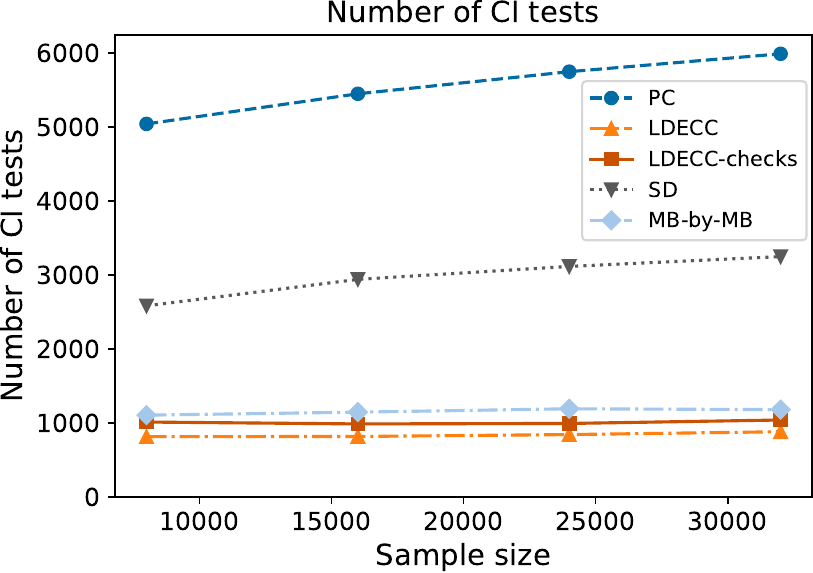}\label{fig:synthetic-finite-sample-num-tests}}
\caption{Results on synthetic linear Gaussian graphs.}
\end{figure}

We present results on both synthetic linear graphs
and the semi-synthetic \emph{MAGIC-NIAB} graph 
from \emph{bnlearn} \citep{scutari2009learning}.
We also ran the MB-by-MB
approach \citep{wang2014discovering},
an instantiation of SD
that uses IAMB (Fig.~\ref{fig:apdx-algo-iamb}) 
and \emph{LocalPC} for local structure learning.
LDECC usually performs comparably to MB-by-MB and 
typically outperforms SD while 
running a similar number of CI tests.
Results on synthetic binomial graphs and 
additional graphs from \emph{bnlearn} are in Appendix~\ref{sec:apdx-experiments}.

\paragraph{Results on synthetic data.}
We generated synthetic linear graphs
with Gaussian errors,
$20$ covariates---non-descendants of $X$ and $Y$---and 
$3$ mediators---nodes
on some causal path from $X$ to $Y$.
We sampled edges between the
nodes with varying probabilities 
and sampled the edge weights uniformly from 
$[-1, -0.25] \cup [0.25, 1]$
(see Appendix~\ref{sec:apdx-experiments-synthetic-linear-dgp} for the precise procedure).
We compared the algorithms based on the
number of CI tests performed with a CI oracle
on $250$ synthetic graphs 
(of which $\approx 77\%$ were \emph{locally orientable}; see Defn.~\ref{defn:locally-orientable}).
The distribution of CI tests (Fig.~\ref{fig:synthetic-with-ci-oracle}) shows
that LDECC performs comparably to MB-by-MB
and substantially outperforms SD and PC on 
$\approx 150$ graphs.
Next, we evaluated our methods on $250$ synthetic graphs
at four sample sizes.
For each graph, we generated data $5$ times.
Thus, we have $N = 1250$ runs.
We used the Fisher-z CI test and ordinary least squares 
for ATE estimation.
We also ran LDECC with \emph{check=True} as
described in Remark~\ref{remark:weaker-ff-ldecc-check-true} (denoted by \emph{LDECC-checks}).
We define \emph{accuracy} as the fraction of times
the estimated ATE set is the same as 
the ATE set we would get if a CI oracle were used
(instead of the Fisher-z test);
and \emph{recall} as the fraction of times
the estimated ATE set contains
the ATE obtained by adjusting for $\Pa(X)$ in the true graph, i.e.,
the fraction of $i \in [N]$ s.t.
$\widehat{\theta}(X \rightarrow Y | \Pa(X; \GM_i^*)) \in \widehat{\Theta}_i$.
We used Hausdorff distance to compute mean squared error (MSE) between the 
estimated $\widehat{\Theta}$ and the ground-truth set $\Theta^{*}$:
\begin{align*}
    \text{MSE}_{\text{Hausdorff}}(\{\widehat{\Theta}\}_{i=1}^{N}, \{\Theta^{*}\}_{i=1}^{N})
        = \frac{1}{N} \sum_{i=1}^{N} \max \left\{ \sup_{u \in \widehat{\Theta}_i} \inf_{v \in \Theta^{*}_i} \left( u - v \right)^2,
            \sup_{v \in \Theta^{*}_i} \inf_{u \in \widehat{\Theta}_i} \left( u - v \right)^2
        \right\}.
\end{align*}

LDECC-checks and MB-by-MB have comparable accuracy
and are better than both SD and PC (Fig.~\ref{fig:synthetic-finite-sample-accuracy}).
In terms of recall, 
LDECC-checks is better than
MB-by-MB, and both 
outperform PC and SD (Fig.~\ref{fig:synthetic-finite-sample-recall}).
LDECC-checks has better
accuracy and recall than LDECC
suggesting that the check helps in practice.
Both variants of LDECC have substantially lower MSE than
MB-by-MB, PC, and SD (Fig.~\ref{fig:synthetic-finite-sample-mse})
while still doing a similar number of CI tests as
MB-by-MB (Fig.~\ref{fig:synthetic-finite-sample-num-tests}).

\begin{figure}
\centering
\subfigure[Tests with CI oracle]{\includegraphics[scale=0.35]{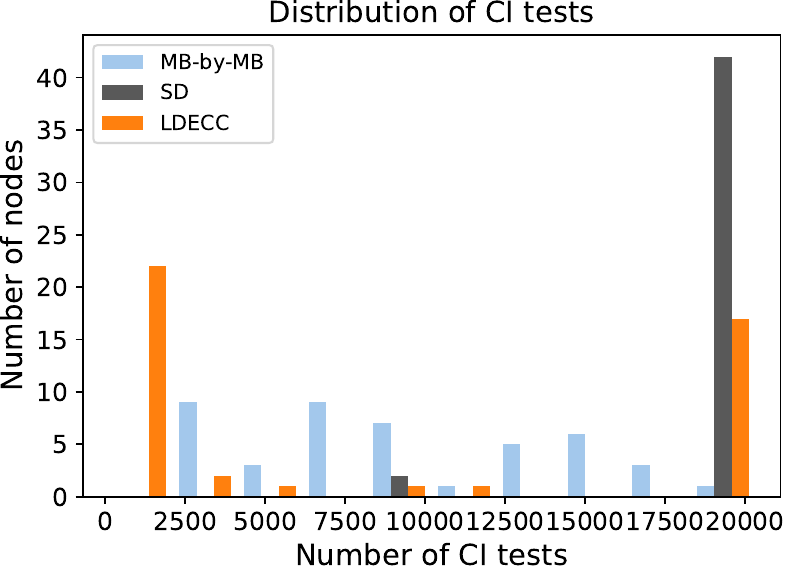}\label{fig:semi-synthetic-with-ci-oracle}}
\hfill
\subfigure[Accuracy]{\includegraphics[scale=0.35]{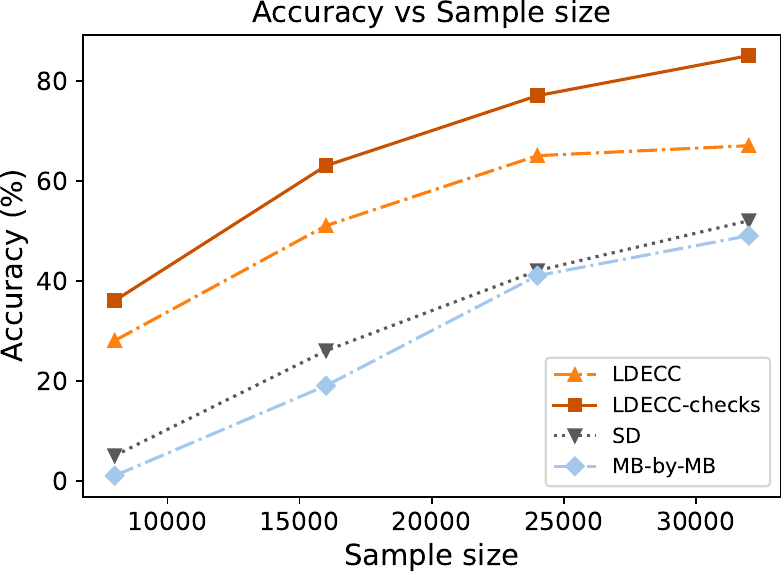}\label{fig:semi-synthetic-finite-sample-accuracy}}
\hfill
\subfigure[Recall]{\includegraphics[scale=0.35]{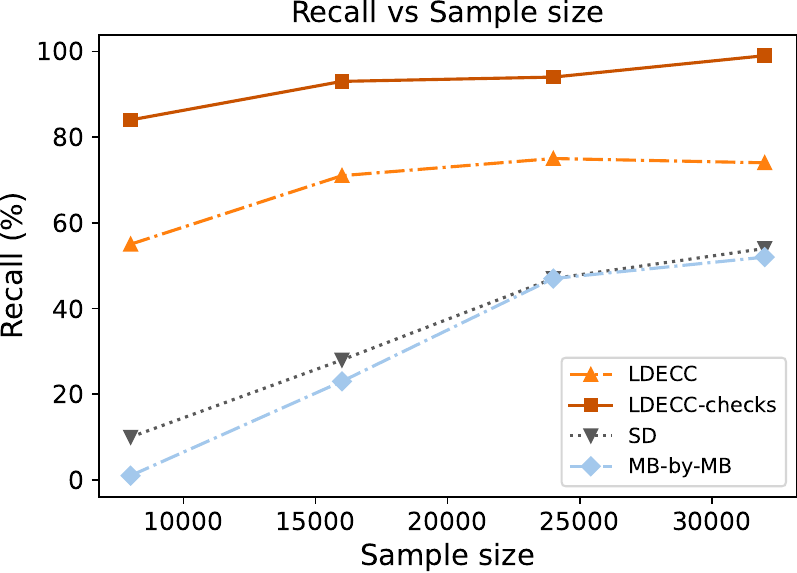}\label{fig:semi-synthetic-finite-sample-recall}}
\\
\subfigure[Median SE vs sample size]{\includegraphics[scale=0.33]{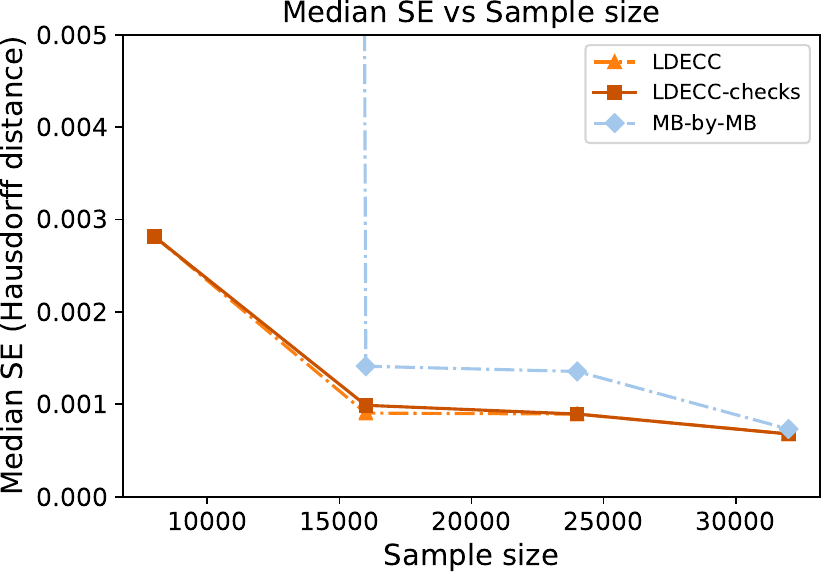}\label{fig:semi-synthetic-finite-sample-mse}}
\hspace{1em}
\subfigure[Average number of CI tests]{\includegraphics[scale=0.35]{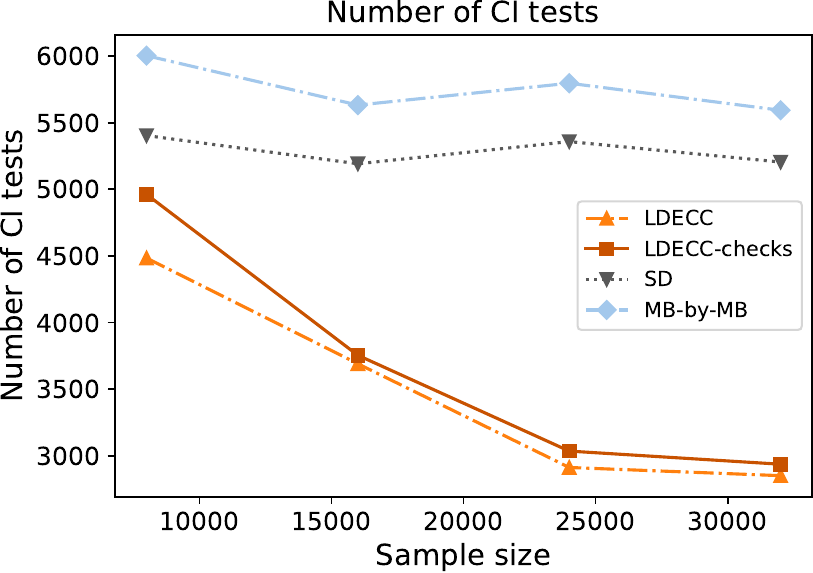}\label{fig:semi-synthetic-finite-sample-num-tests}}
\caption{Results on the semi-synthetic \emph{MAGIC-NIAB} linear Gaussian graph.}
\end{figure}

\paragraph{Results on \emph{MAGIC-NIAB} from bnlearn.}
\emph{MAGIC-NIAB} is a linear Gaussian DAG with $44$ nodes.
With a CI oracle, PC performed $\approx 1.472 \times 10^6$ tests. 
We plot the distribution of CI tests
by repeatedly setting each node as the treatment 
(capping the maximum number of tests per node to $20000$).
We see that LDECC and MB-by-MB perform better than SD 
with MB-by-MB performing well across all nodes
(Fig.~\ref{fig:semi-synthetic-with-ci-oracle}).
Next, we designated the nodes \emph{G266}
and \emph{HT}
as the treatment and outcome, respectively.
We compared the methods
by sampling the data $100$ times at four sample sizes
(capping the maximum number of tests run by each algorithm to $7000$).
LDECC has higher accuracy (Fig.~\ref{fig:semi-synthetic-finite-sample-accuracy}) and recall (Fig.~\ref{fig:semi-synthetic-finite-sample-recall}) than MB-by-MB
as well as SD.
We plot \emph{median} squared error (Fig.~\ref{fig:semi-synthetic-finite-sample-mse}) for MB-by-MB and LDECC 
(MB-by-MB does poorly at small sample sizes due to high bias in the estimated ATE).
Both are similar at larger sample sizes.
Additionally, LDECC performs fewer CI tests than MB-by-MB and SD (Fig.~\ref{fig:semi-synthetic-finite-sample-num-tests}).

\section{Conclusion}
\label{sec:conclusion}
Broadening the landscape of local causal discovery,
we propose a new algorithm
that uses ECCs to orient parents.
We show that LDECC has different computational and faithfulness
requirements compared to the existing class of sequential algorithms.
This allows us to profitably combine the two methods to
get polynomial runtimes on a larger class of graphs
as well as identify the set of possible ATE values
under weaker assumptions.
In future work, we hope to further weaken 
the faithfulness requirements
and extend our methods to handle
causal graphs with latent variables.

\acks{
We gratefully acknowledge the NSF (FAI 2040929 and IIS2211955), UPMC, Highmark Health, Abridge, Ford Research, Mozilla, the PwC Center, Amazon AI, JP Morgan Chase, the Block Center, the Center for Machine Learning and Health, and the CMU Software Engineering Institute (SEI) via Department of Defense contract FA8702-15-D-0002, for their generous support of ACMI Lab’s research. In particular, we are grateful to the PwC center for supporting Shantanu Gupta as a presidential scholar during the completion of this research, and to Amazon AI who has named Shantanu Gupta a recipient of the Amazon Ph.D. fellowship for the coming year.
}

\bibliography{refs}

\begin{thebibliography}{51}
\providecommand{\natexlab}[1]{#1}
\providecommand{\url}[1]{\texttt{#1}}
\expandafter\ifx\csname urlstyle\endcsname\relax
  \providecommand{\doi}[1]{doi: #1}\else
  \providecommand{\doi}{doi: \begingroup \urlstyle{rm}\Url}\fi

\bibitem[Aliferis et~al.(2003)Aliferis, Tsamardinos, and
  Statnikov]{aliferis2003hiton}
Constantin~F Aliferis, Ioannis Tsamardinos, and Alexander Statnikov.
\newblock Hiton: a novel markov blanket algorithm for optimal variable
  selection.
\newblock In \emph{AMIA annual symposium proceedings}. American Medical
  Informatics Association, 2003.

\bibitem[Aliferis et~al.(2010)Aliferis, Statnikov, Tsamardinos, Mani, and
  Koutsoukos]{aliferis2010local}
Constantin~F Aliferis, Alexander Statnikov, Ioannis Tsamardinos, Subramani
  Mani, and Xenofon~D Koutsoukos.
\newblock Local causal and markov blanket induction for causal discovery and
  feature selection for classification part i: algorithms and empirical
  evaluation.
\newblock \emph{Journal of Machine Learning Research}, 11\penalty0 (1), 2010.

\bibitem[Cheng et~al.(2022{\natexlab{a}})Cheng, Li, Liu, Yu, Le, and
  Liu]{cheng2022toward}
Debo Cheng, Jiuyong Li, Lin Liu, Kui Yu, Thuc~Duy Le, and Jixue Liu.
\newblock Toward unique and unbiased causal effect estimation from data with
  hidden variables.
\newblock \emph{IEEE Transactions on Neural Networks and Learning Systems},
  2022{\natexlab{a}}.

\bibitem[Cheng et~al.(2022{\natexlab{b}})Cheng, Li, Liu, Yu, Lee, and
  Liu]{cheng2022discovering}
Debo Cheng, Jiuyong Li, Lin Liu, Kui Yu, Thuc~Duy Lee, and Jixue Liu.
\newblock Discovering ancestral instrumental variables for causal inference
  from observational data.
\newblock \emph{arXiv preprint arXiv:2206.01931}, 2022{\natexlab{b}}.

\bibitem[Claassen and Heskes(2012)]{claassen2012logical}
Tom Claassen and Tom Heskes.
\newblock A logical characterization of constraint-based causal discovery.
\newblock \emph{arXiv preprint arXiv:1202.3711}, 2012.

\bibitem[Cooper(1997)]{cooper1997simple}
Gregory~F Cooper.
\newblock A simple constraint-based algorithm for efficiently mining
  observational databases for causal relationships.
\newblock \emph{Data Mining and Knowledge Discovery}, 1997.

\bibitem[De~Luna et~al.(2011)De~Luna, Waernbaum, and
  Richardson]{de2011covariate}
Xavier De~Luna, Ingeborg Waernbaum, and Thomas~S Richardson.
\newblock Covariate selection for the nonparametric estimation of an average
  treatment effect.
\newblock \emph{Biometrika}, 2011.

\bibitem[Entner et~al.(2012)Entner, Hoyer, and Spirtes]{entner2012statistical}
Doris Entner, Patrik Hoyer, and Peter Spirtes.
\newblock Statistical test for consistent estimation of causal effects in
  linear non-gaussian models.
\newblock In \emph{Artificial Intelligence and Statistics}, 2012.

\bibitem[Entner et~al.(2013)Entner, Hoyer, and Spirtes]{entner2013data}
Doris Entner, Patrik Hoyer, and Peter Spirtes.
\newblock Data-driven covariate selection for nonparametric estimation of
  causal effects.
\newblock In \emph{Artificial Intelligence and Statistics}. PMLR, 2013.

\bibitem[Fang and He(2020)]{fang2020ida}
Zhuangyan Fang and Yangbo He.
\newblock Ida with background knowledge.
\newblock In \emph{Conference on Uncertainty in Artificial Intelligence}. PMLR,
  2020.

\bibitem[Gao and Ji(2015)]{gao2015local}
Tian Gao and Qiang Ji.
\newblock Local causal discovery of direct causes and effects.
\newblock \emph{Advances in Neural Information Processing Systems}, 2015.

\bibitem[Geffner et~al.(2022)Geffner, Antoran, Foster, Gong, Ma, Kiciman,
  Sharma, Lamb, Kukla, Pawlowski, et~al.]{geffner2022deep}
Tomas Geffner, Javier Antoran, Adam Foster, Wenbo Gong, Chao Ma, Emre Kiciman,
  Amit Sharma, Angus Lamb, Martin Kukla, Nick Pawlowski, et~al.
\newblock Deep end-to-end causal inference.
\newblock \emph{arXiv preprint arXiv:2202.02195}, 2022.

\bibitem[Glymour et~al.(2019)Glymour, Zhang, and Spirtes]{glymour2019review}
Clark Glymour, Kun Zhang, and Peter Spirtes.
\newblock Review of causal discovery methods based on graphical models.
\newblock \emph{Frontiers in genetics}, 2019.

\bibitem[Gultchin et~al.(2020)Gultchin, Kusner, Kanade, and
  Silva]{gultchin2020differentiable}
Limor Gultchin, Matt Kusner, Varun Kanade, and Ricardo Silva.
\newblock Differentiable causal backdoor discovery.
\newblock In \emph{International Conference on Artificial Intelligence and
  Statistics}. PMLR, 2020.

\bibitem[Henckel et~al.(2019)Henckel, Perkovi{\'c}, and
  Maathuis]{henckel2019graphical}
Leonard Henckel, Emilija Perkovi{\'c}, and Marloes~H Maathuis.
\newblock Graphical criteria for efficient total effect estimation via
  adjustment in causal linear models.
\newblock \emph{arXiv preprint arXiv:1907.02435}, 2019.

\bibitem[Hyttinen et~al.(2015)Hyttinen, Eberhardt, and
  J{\"a}rvisalo]{hyttinen2015calculus}
Antti Hyttinen, Frederick Eberhardt, and Matti J{\"a}rvisalo.
\newblock Do-calculus when the true graph is unknown.
\newblock In \emph{UAI}. Citeseer, 2015.

\bibitem[Jaber et~al.(2019)Jaber, Zhang, and Bareinboim]{jaber2019causal}
Amin Jaber, Jiji Zhang, and Elias Bareinboim.
\newblock Causal identification under markov equivalence: Completeness results.
\newblock In \emph{International Conference on Machine Learning}. PMLR, 2019.

\bibitem[Ling et~al.(2020)Ling, Yu, Wang, Li, and Wu]{ling2020using}
Zhaolong Ling, Kui Yu, Hao Wang, Lei Li, and Xindong Wu.
\newblock Using feature selection for local causal structure learning.
\newblock \emph{IEEE Transactions on Emerging Topics in Computational
  Intelligence}, 2020.

\bibitem[Ling et~al.(2021)Ling, Yu, Zhang, Liu, and Li]{ling2021causal}
Zhaolong Ling, Kui Yu, Yiwen Zhang, Lin Liu, and Jiuyong Li.
\newblock Causal learner: A toolbox for causal structure and markov blanket
  learning.
\newblock \emph{arXiv preprint arXiv:2103.06544}, 2021.

\bibitem[Maathuis and Colombo(2015)]{maathuis2015generalized}
Marloes~H Maathuis and Diego Colombo.
\newblock A generalized back-door criterion.
\newblock \emph{The Annals of Statistics}, 2015.

\bibitem[Maathuis et~al.(2009)Maathuis, Kalisch, and
  B{\"u}hlmann]{maathuis2009estimating}
Marloes~H Maathuis, Markus Kalisch, and Peter B{\"u}hlmann.
\newblock Estimating high-dimensional intervention effects from observational
  data.
\newblock \emph{The Annals of Statistics}, 2009.

\bibitem[Magliacane et~al.(2016)Magliacane, Claassen, and
  Mooij]{magliacane2016ancestral}
Sara Magliacane, Tom Claassen, and Joris~M Mooij.
\newblock Ancestral causal inference.
\newblock \emph{Advances in Neural Information Processing Systems}, 2016.

\bibitem[Malinsky and Spirtes(2016)]{malinsky2016estimating}
Daniel Malinsky and Peter Spirtes.
\newblock Estimating causal effects with ancestral graph markov models.
\newblock In \emph{Conference on Probabilistic Graphical Models}. PMLR, 2016.

\bibitem[Mani et~al.(2012)Mani, Spirtes, and Cooper]{mani2012theoretical}
Subramani Mani, Peter~L Spirtes, and Gregory~F Cooper.
\newblock A theoretical study of y structures for causal discovery.
\newblock \emph{arXiv preprint arXiv:1206.6853}, 2012.

\bibitem[Meek(2013)]{meek2013causal}
Christopher Meek.
\newblock Causal inference and causal explanation with background knowledge.
\newblock \emph{arXiv preprint arXiv:1302.4972}, 2013.

\bibitem[Mooij et~al.(2015)Mooij, Cremers, et~al.]{mooij2015empirical}
Joris~M Mooij, Jerome Cremers, et~al.
\newblock An empirical study of one of the simplest causal prediction
  algorithms.
\newblock In \emph{UAI 2015 Workshop on Advances in Causal Inference}, 2015.

\bibitem[Nandy et~al.(2017)Nandy, Maathuis, and
  Richardson]{nandy2017estimating}
Preetam Nandy, Marloes~H Maathuis, and Thomas~S Richardson.
\newblock Estimating the effect of joint interventions from observational data
  in sparse high-dimensional settings.
\newblock \emph{The Annals of Statistics}, 2017.

\bibitem[Pearl(2009)]{pearl2009causality}
Judea Pearl.
\newblock \emph{Causality}.
\newblock Cambridge university press, 2009.

\bibitem[Perkovi{\'c} et~al.(2017)Perkovi{\'c}, Kalisch, and
  Maathuis]{perkovic2017interpreting}
Emilija Perkovi{\'c}, Markus Kalisch, and Maloes~H Maathuis.
\newblock Interpreting and using cpdags with background knowledge.
\newblock \emph{arXiv preprint arXiv:1707.02171}, 2017.

\bibitem[Ramsey et~al.(2012)Ramsey, Zhang, and Spirtes]{ramsey2012adjacency}
Joseph Ramsey, Jiji Zhang, and Peter~L Spirtes.
\newblock Adjacency-faithfulness and conservative causal inference.
\newblock \emph{arXiv preprint arXiv:1206.6843}, 2012.

\bibitem[Rotnitzky and Smucler(2019)]{rotnitzky2019efficient}
Andrea Rotnitzky and Ezequiel Smucler.
\newblock Efficient adjustment sets for population average treatment effect
  estimation in non-parametric causal graphical models.
\newblock \emph{arXiv preprint arXiv:1912.00306}, 2019.

\bibitem[Scutari(2009)]{scutari2009learning}
Marco Scutari.
\newblock Learning bayesian networks with the bnlearn r package.
\newblock \emph{arXiv preprint arXiv:0908.3817}, 2009.

\bibitem[Shah et~al.(2022)Shah, Shanmugam, and Ahuja]{shah2022finding}
Abhin Shah, Karthikeyan Shanmugam, and Kartik Ahuja.
\newblock Finding valid adjustments under non-ignorability with minimal dag
  knowledge.
\newblock In \emph{International Conference on Artificial Intelligence and
  Statistics}. PMLR, 2022.

\bibitem[Silva and Shimizu(2017)]{silva2017learning}
Ricardo Silva and Shohei Shimizu.
\newblock Learning instrumental variables with structural and non-gaussianity
  assumptions.
\newblock \emph{Journal of Machine Learning Research}, 2017.

\bibitem[Spirtes et~al.(2000)Spirtes, Glymour, Scheines, and
  Heckerman]{spirtes2000causation}
Peter Spirtes, Clark~N Glymour, Richard Scheines, and David Heckerman.
\newblock \emph{Causation, prediction, and search}.
\newblock MIT press, 2000.

\bibitem[Squires and Uhler(2022)]{squires2022causal}
Chandler Squires and Caroline Uhler.
\newblock Causal structure learning: a combinatorial perspective.
\newblock \emph{arXiv preprint arXiv:2206.01152}, 2022.

\bibitem[Tian and Pearl(2002)]{tian2002general}
Jin Tian and Judea Pearl.
\newblock \emph{A general identification condition for causal effects}.
\newblock eScholarship, University of California, 2002.

\bibitem[Toth et~al.(2022)Toth, Lorch, Knoll, Krause, Pernkopf, Peharz, and
  Von~K{\"u}gelgen]{toth2022active}
Christian Toth, Lars Lorch, Christian Knoll, Andreas Krause, Franz Pernkopf,
  Robert Peharz, and Julius Von~K{\"u}gelgen.
\newblock Active bayesian causal inference.
\newblock \emph{arXiv preprint arXiv:2206.02063}, 2022.

\bibitem[Tsamardinos et~al.(2003)Tsamardinos, Aliferis, Statnikov, and
  Statnikov]{tsamardinos2003algorithms}
Ioannis Tsamardinos, Constantin~F Aliferis, Alexander~R Statnikov, and
  Er~Statnikov.
\newblock Algorithms for large scale markov blanket discovery.
\newblock In \emph{FLAIRS conference}. St. Augustine, FL, 2003.

\bibitem[Tsamardinos et~al.(2006)Tsamardinos, Brown, and
  Aliferis]{tsamardinos2006max}
Ioannis Tsamardinos, Laura~E Brown, and Constantin~F Aliferis.
\newblock The max-min hill-climbing bayesian network structure learning
  algorithm.
\newblock \emph{Machine learning}, 2006.

\bibitem[Uhler et~al.(2013)Uhler, Raskutti, B{\"u}hlmann, and
  Yu]{uhler2013geometry}
Caroline Uhler, Garvesh Raskutti, Peter B{\"u}hlmann, and Bin Yu.
\newblock Geometry of the faithfulness assumption in causal inference.
\newblock \emph{The Annals of Statistics}, 2013.

\bibitem[VanderWeele and Shpitser(2011)]{vanderweele2011new}
Tyler~J VanderWeele and Ilya Shpitser.
\newblock A new criterion for confounder selection.
\newblock \emph{Biometrics}, 2011.

\bibitem[Versteeg et~al.(2022)Versteeg, Mooij, and Zhang]{versteeg2022local}
Philip Versteeg, Joris Mooij, and Cheng Zhang.
\newblock Local constraint-based causal discovery under selection bias.
\newblock In \emph{Conference on Causal Learning and Reasoning}. PMLR, 2022.

\bibitem[Wang et~al.(2014)Wang, Zhou, Zhao, and Geng]{wang2014discovering}
Changzhang Wang, You Zhou, Qiang Zhao, and Zhi Geng.
\newblock Discovering and orienting the edges connected to a target variable in
  a dag via a sequential local learning approach.
\newblock \emph{Computational Statistics \& Data Analysis}, 2014.

\bibitem[Watson and Silva(2022)]{watson2022causal}
David~S Watson and Ricardo Silva.
\newblock Causal discovery under a confounder blanket.
\newblock In \emph{Uncertainty in Artificial Intelligence}. PMLR, 2022.

\bibitem[Witte and Didelez(2019)]{witte2019covariate}
Janine Witte and Vanessa Didelez.
\newblock Covariate selection strategies for causal inference: Classification
  and comparison.
\newblock \emph{Biometrical Journal}, 2019.

\bibitem[Yin et~al.(2008)Yin, Zhou, Wang, He, Zheng, and Geng]{yin2008partial}
Jianxin Yin, You Zhou, Changzhang Wang, Ping He, Cheng Zheng, and Zhi Geng.
\newblock Partial orientation and local structural learning of causal networks
  for prediction.
\newblock In \emph{Causation and Prediction Challenge}. PMLR, 2008.

\bibitem[Yu et~al.(2020)Yu, Guo, Liu, Li, Wang, Ling, and Wu]{yu2020causality}
Kui Yu, Xianjie Guo, Lin Liu, Jiuyong Li, Hao Wang, Zhaolong Ling, and Xindong
  Wu.
\newblock Causality-based feature selection: Methods and evaluations.
\newblock \emph{ACM Computing Surveys (CSUR)}, 2020.

\bibitem[Zhang and Spirtes(2008)]{zhang2008detection}
Jiji Zhang and Peter Spirtes.
\newblock Detection of unfaithfulness and robust causal inference.
\newblock \emph{Minds and Machines}, 2008.

\bibitem[Zhang and Spirtes(2016)]{zhang2016three}
Jiji Zhang and Peter Spirtes.
\newblock The three faces of faithfulness.
\newblock \emph{Synthese}, 2016.

\bibitem[Zhou et~al.(2010)Zhou, Wang, Yin, and Geng]{zhou2010discover}
You Zhou, Changzhang Wang, Jianxin Yin, and Zhi Geng.
\newblock Discover local causal network around a target to a given depth.
\newblock In \emph{Causality: Objectives and Assessment}. PMLR, 2010.

\end{thebibliography}

\clearpage

\appendix

\section{Additional details on the PC algorithm} \label{apdx:prelim}

\begin{figure}[t]
\centering
\begin{minipage}[b]{0.52\textwidth}
    \vspace{0pt}
    \setlength{\interspacetitleruled}{0pt}%
    \setlength{\algotitleheightrule}{0pt}%
    \begin{algorithm}[H]
        \SetAlgoLined
        \SetKwFunction{FOrientInGraph}{GetCPDAG}
        \SetKwProg{FOIG}{def}{:}{}
        \FOIG{\FOrientInGraph{Undirected graph $\UM$, DSep}}{
            \For{every unshielded $A \text{---} C \text{---B} \in \UM$}{
                \uIf{$C \notin \text{DSep}(A, B)$}{
                    Orient $A \rightarrow C \leftarrow B$\;
                }
            }
            CPDAG $\mathcal{G} \gets$ ApplyMeekRules($\UM$)\;
            \KwRet $\mathcal{G}$\;
        }
    \end{algorithm}
    \begin{algorithm}[H]
        \SetAlgoLined
        \SetKwFunction{FPCTest}{PCTest}
        \SetKwProg{FPCT}{def}{:}{}
        \FPCT{\FPCTest{Undirected graph $\UM(\V, \E)$}}{
            $s \gets 0$\;
            \While{$\exists (A \text{---} B) \in \E$ s.t. $|\DNe(A) \setminus \{ B \}| \geq s$}{
                \For{$\S \subseteq \DNe(A) \setminus \{ B \}$ s.t. $|\S| = s$}{
                    \uIf{$A \indep B | \S$}{
                        $\UM$.removeEdge($A \text{---} B$)\;
                        \textbf{yield} $(A \indep B | \S)$\;
                        \textbf{break}\;
                    }
                }
                $s \gets s + 1$\;
            }
        }
    \end{algorithm}
    \begin{algorithm}[H]
    \SetAlgoLined
    Completely connected undirected graph $\UM(\V, \E)$\;
    $\forall A, B \in \V, \,\, \text{DSep}(A, B) \gets \text{null}$\;
    \For{$(A \indep B | \S) \in \text{PCTest}(\UM)$}{
        $\text{DSep}(A, B) \gets \S$\;
    }
    CPDAG $\mathcal{G} \gets$ GetCPDAG($\UM$, DSep)\;
    \KwOutput{$\mathcal{G}, \text{DSep}$}
    \end{algorithm}
    \captionof{figure}{The PC algorithm \citep[Sec.~5.4.2]{spirtes2000causation}.}
    \label{fig:apdx-pc-algorithm}
\end{minipage}
\hfill
\begin{minipage}[b]{0.46\textwidth}
    \vspace{0pt}
    \centering
    \includegraphics[scale=0.55]{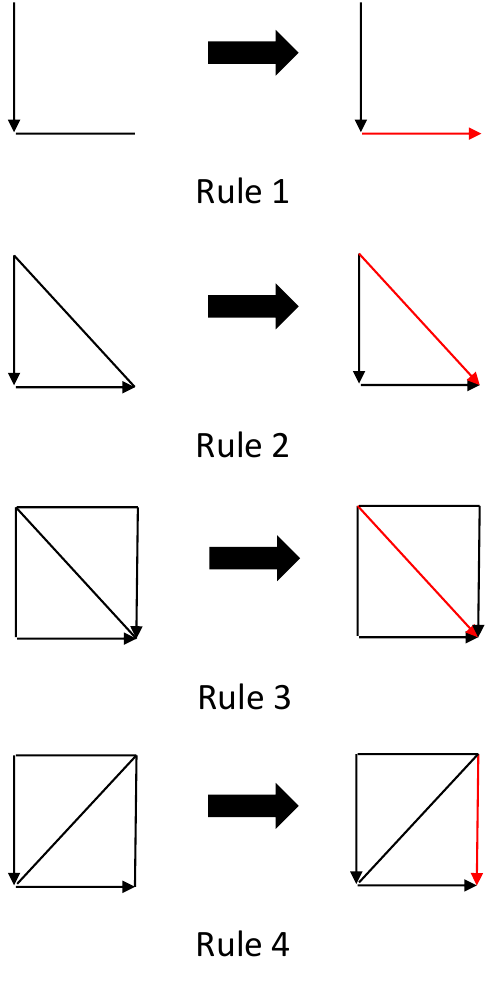}
    \captionof{figure}{Meek's orientation rules.}
    \label{fig:meek-rules}
\end{minipage}
\end{figure}

\begin{definition}[CPDAG {\citep[Pg.~5]{maathuis2015generalized}}]\label{defn:apdx-cpdag}
A set of DAGs that entail the same set of CIs form an MEC.
This MEC can be uniquely represented using a CPDAG.
A CPDAG is a graph with the same skeleton as each DAG in the MEC
and contains both directed ($\rightarrow$) and undirected ($\text{---}$) edges.
A directed edge $A \rightarrow B$ means that the $A \rightarrow B$ is present
in every DAG in the MEC.
An undirected edge $A \text{---} B$ means that there is at least one DAG in the
MEC with an $A \rightarrow B$ edge and 
at least one DAG with the $B \rightarrow A$ edge.
\end{definition}

The PC algorithm (Fig.~\ref{fig:apdx-pc-algorithm}) starts with a 
fully connected skeleton and runs CI tests to remove
edges. For each pair of nodes $(A, B)$ that are adjacent in the skeleton,
we run CIs of size $s$---starting with $s=0$ and then increasing it by one
in each subsequent iteration---until the edge is removed or the number
of nodes adjacent to both $A$ and $B$ is less than $s$.
Once the skeleton is found, UCs are detected and then additional edges are
oriented by repeatedly applying Meek's rules (Fig.~\ref{fig:meek-rules})
until no additional edges can be oriented. Under faithfulness,
with access to a CI oracle, the output of PC is a CPDAG 
encoding the MEC of the true DAG $\GM^{*}$.

\section{Additional details for Section~\ref{sec:ldecc}} \label{apdx:ldecc}

\begin{figure}
\centering
\begin{minipage}[b]{0.50\textwidth}
    \vspace{0pt}
    \setlength{\interspacetitleruled}{0pt}%
    \setlength{\algotitleheightrule}{0pt}%
    \begin{algorithm}[H]
    \SetAlgoLined
    \KwInput{Node $X$.}
    $\MB(X) \gets \emptyset$\;
    \tcp{Forward Pass.}
    \While{$\MB(X)$ has changed}{
        \For{$V \in \V \setminus (\MB(X) \cup \{X\})$}{
            \lIf{$V \notindep X | \MB(X)$}{
                $\MB(X)$.add($V$)
            }
        }
    }
    \tcp{Backward Pass.}
    \For{$V \in \MB(X)$}{
        \lIf{$V \indep X | (\MB(X) \setminus \{X\})$}{
            $\MB(X)$.remove($V$)
        }
    }
    \KwOutput{$\MB(X)$}
    \end{algorithm}
    \captionof{figure}{The IAMB algorithm.}
    \label{fig:apdx-algo-iamb}
\end{minipage}
\hfill
\begin{minipage}[b]{0.45\textwidth}
    \vspace{0pt}
    \setlength{\interspacetitleruled}{0pt}%
    \setlength{\algotitleheightrule}{0pt}%
    \begin{algorithm}[H]
        \SetAlgoLined
        \SetKwFunction{FNBrs}{Nbrs}
        \SetKwProg{FNB}{def}{:}{}
        \FNB{\FNBrs{CPDAG $\GM$, Target $X$}}{
            parents $\gets \emptyset$\;
            children $\gets \emptyset$\;
            unoriented $\gets \emptyset$\;
            \For{$V \in \DNe_{\GM}(X)$}{
                \uIf{$X \leftarrow V \in \GM$}{
                    parents.add($V$)\;
                }
                \uElseIf{$X \rightarrow V \in \GM$}{
                    children.add($V$)\;
                }
                \uElse{
                    unoriented.add($V$)\;
                }
            }
            \KwRet parents, children, unoriented\;
        }
    \end{algorithm}
    \captionof{figure}{The \emph{Nbrs} subroutine used by SD.}
    \label{fig:apdx-nbrs-subroutine}
\end{minipage}
\end{figure}

In the example below, we demonstrate that
there exist causal graphs with nodes for
which $\nss$ does not exist. 

\begin{example}[MNS does not exist]\label{example:apdx-mns-does-not-exist}
In the following graph, for node $Y \in \Desc(X)$, 
$\nss_X$ does not exist because
there is no subset of $\DNe(X)$ that d-separates $Y$ from $X$.

\includegraphics[scale=0.42]{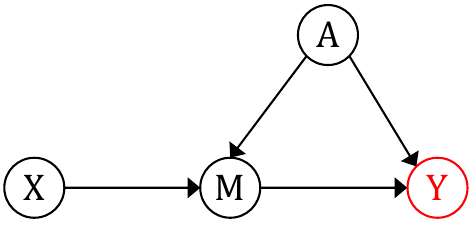}
\end{example}

\subsection{Omitted Proofs for Section~\ref{sec:ldecc}} \label{sec:apdx-ldecc-ommitted-proofs}

\newtheorem*{prop:valid-nss-non-desc}{Proposition~\ref{prop:valid-nss-non-desc}}
\begin{prop:valid-nss-non-desc}
For any node $V \notin (\Desc(X) \cup \DNep(X))$, $\nss_X(V)$ exists and $\nss_X(V) \subseteq \Pa(X)$.
\end{prop:valid-nss-non-desc}
\begin{proof}
Let $\Q = \Desc(X) \cup \DNep(X)$.
Since $\Pa(X)$ blocks all backdoor paths from $X$, 
for every $V \notin \Q$, 
we have $V \indep X | \Pa(X)$.
Therefore, for every $V \notin \Q$, there exists some subset
$\S \subseteq \Pa(X)$ such that $V \indep X | \S$.
\end{proof}

\newtheorem*{prop:mns-unique}{Proposition~\ref{prop:mns-unique}}
\begin{prop:mns-unique}[Uniqueness of MNS]
For nodes $V$ s.t. $\mns_X(V)$ exists, it is unique.
\end{prop:mns-unique}
\begin{proof}
We will prove this by contradiction.
Consider a node $V$ with two MNSs: $\S_1 \subseteq \DNe(X)$ and $\S_2 \subseteq \DNe(X)$ 
with $\S_1 \neq \S_2$.
If $\S_1 \subset \S_2$ or $\S_2 \subset \S_1$, then minimality is violated.
Hence, going forward we will only consider the case where $\S_1 \setminus \S_2 \neq \emptyset$
and $\S_2 \setminus \S_1 \neq \emptyset$.
Consider any node $A \in \S_1 \setminus \S_2$. 
For $\S_2$ to be a valid MNS, some nodes in $\S_2 \setminus \S_1$ 
must block all paths from $V$ to $X$ that contain $A$
(this is because if this path were to be 
only be blocked by some nodes in $\S_1$,
then minimality of $\S_1$ will be violated as $\S_1 \setminus \{ A \}$
would also have been a valid MNS).
This means that there is a path from $V$ to $X$
through some nodes in $\S_2 \setminus \S_1$ that
cannot be blocked by $\S_1$ 
(else these nodes in $\S_1$ would have blocked the paths 
from $V$ to $A$ violating minimality of $\S_1$).
This contradicts the fact that $\S_1$ is a valid MNS.
Therefore, we must have $\S_1 = \S_2$.
\end{proof}

\newtheorem*{prop:ecc}{Proposition~\ref{prop:ecc}}
\begin{prop:ecc}[Eager Collider Check]
For nodes $A, B \in \V \setminus \DNe^{+}(X)$,
any $\S \subset \V \setminus \{ A, B, X \}$, if
(i) $A \indep B | \S$;
and (ii) $A \notindep B | \S \cup \{X\}$; 
then $A, B \notin \Desc(X)$
and $\nss_X(A), \nss_X(B) \subseteq \Pa(X)$.
\end{prop:ecc}
\begin{proof}
We prove this by contradiction. 
Let's say there is a child $M$ of $X$ 
such that $M \in \mns_X(B)$ or $M \in \mns_X(A)$.
First, note that if Conditions~(i, ii) hold, then 
there is a path of the form 
$A \bullet \rightarrow C$ and $B \bullet \rightarrow C$ and 
$C \rightarrow \hdots X$, where $\bullet$ means that there can be either an arrowhead or tail
(i.e., there can be either a directed path $A \rightarrow \hdots \rightarrow C$ 
or a backdoor path $A \leftarrow \hdots \rightarrow C$ and likewise for $B$)
with $C \notin \S$.
W.l.o.g., let's say that $M \in \mns_X(B)$
(the argument for node $A$ follows similarly).
Then there is a directed path from $X$ to $B$ through $M$
(i.e., $X \rightarrow M \rightarrow \hdots \rightarrow B$). 
There cannot be a path $B \rightarrow \hdots \rightarrow M$
because then $M$ will be a collider and therefore we will have
$M \notin \mns_X(B)$.
These components are illustrated in the figure below:

\includegraphics[scale=0.45]{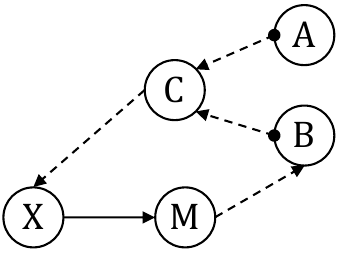}

We now show that $B \notin \Desc(X)$
(the argument for node $A$ is the same).
There cannot be a directed path $B \rightarrow \hdots \rightarrow C$ because otherwise a cycle 
$B \rightarrow \hdots \rightarrow C \rightarrow \hdots \rightarrow X \rightarrow M \rightarrow \hdots \rightarrow B$ 
gets created.
Thus the path from $B$ to $C$ must be of the form
$B \leftarrow \hdots \rightarrow C$.
Note that there is an active path between $A$ and $B$ 
through $X$ 
($A \bullet \rightarrow \hdots \rightarrow C \rightarrow \hdots \rightarrow X \rightarrow M
\rightarrow \hdots \rightarrow B$).
Since $A \indep B | \S$, there are two possibilities:
(i) $\S$ contains $X$ to block this path which contradicts the definition of $\S$ 
(where $X \notin \S$); or
(ii) $\S$ blocks all paths between $A$ and $X$ or between $B$ and $X$ in
which case $A$ and $B$ cannot become dependent when additionally
conditioned on $X$ thereby violating Condition~(ii).
Therefore, we have that $A, B \notin \Desc(X)$ and
by Prop.~\ref{prop:valid-nss-non-desc}, $\mns_X(A)$ and $\mns_X(B)$
will be valid and only contain parents of $X$.
\end{proof}

\citet{claassen2012logical} use minimal (in)dependencies to construct
three logical rules which are sound and complete for performing
causal discovery. While their algorithm cannot directly be
used for local causal discovery, we show below that Lemma~3 in their
paper can be used to simplify the proof of Eager Collider Check:
\begin{proof}[Alternative proof of ECC]
Since we have $A \indep B | \S$ and $A \notindep B | \S \cup \{X\}$,
by \citet[Lemma~3]{claassen2012logical}, 
we have $A, B \notin \Desc(X)$.
Therefore, by Prop.~\ref{prop:valid-nss-non-desc}, $\mns_X(A)$ and $\mns_X(B)$
will be valid and only contain parents of $X$.
\end{proof}

\subsubsection{Proof of correctness of LDECC under the CFA}

\begin{lemma}\label{lemma:apdx-child-cannot-sep-two-non-desc}
Consider a DAG $\GM(\V, \E)$ and a node $X \in \V$.
Let $A, B \in \Pa(X)$ be two parents of $X$ such that $A \text{---} B \notin \E$.
Then, for any $\S \subseteq \V \setminus \{A, B, X\}$ such that $A \indep B | \S$,
we have $\Ch(X) \cap \S = \emptyset$.
\end{lemma}
\begin{proof}
We prove this by contradiction. 
Let's say that there is child $M \in \S$ 
(the relevant component of the graph is shown in the figure below).

\includegraphics[scale=0.45]{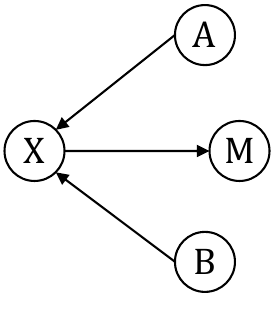}

Since $M$ is a child of $X$, conditioning on $M$
opens up the path $A \rightarrow X \leftarrow B$
rendering $A$ and $B$ dependent conditioned on $\S$.
Therefore, we must have $M \notin \S$.
\end{proof}

\newtheorem*{thm:ldecc-correctness}{Theorem~\ref{thm:ldecc-correctness}}
\begin{thm:ldecc-correctness}[Correctness]
Under the CFA and with access to a CI oracle, we have 
$\Theta_{\text{LDECC}} \overset{\text{set}}{=} \Theta^{*}$.
\end{thm:ldecc-correctness}
\begin{proof}
We will prove the correctness of LDECC by showing that 
(i) every orientable neighbor of the treatment $X$ will get oriented correctly by LDECC; and 
(ii) every unorientable neighbor of the treatment will remain unoriented.

We assume that the function \emph{FindMarkovBlanket} finds the
Markov blanket correctly under the CFA. The IAMB algorithm,
which we use in our experiments, has this property.
Additionally, the function \emph{RunLocalPC} will also return
the correct $\DNe(X)$ under the CFA and with a CI oracle.

\paragraph{Parents are oriented correctly.}
In PC, edges get oriented using UCs and then additional orientations are
propagated via the application of Meek's rules (Figure~\ref{fig:meek-rules}). 

The simplest case is where two parents form a UC at $X$. 
Consider parents $W_1$ and $W_2$ that get oriented because 
they form a UC $W_1 \rightarrow X \leftarrow W_2$. 
Lines~\ref{alg-ldecc:if-cond-nbr-uc},\ref{alg-ldecc:if-cond-nbr-uc-mark} 
will mark $W_1$ and $W_2$ as parents.

We will now consider parents that get oriented due
to each of the four Meek rules and show that LDECC
orients parents for each of the four cases.

\emph{Meek Rule $1$:}

Consider a parent $W$ that gets oriented due to the application of Meek's rule $1$. 
This can only happen due to some UC $A \rightarrow C \leftarrow B$ from which these orientations have been propagated (relevant components of the graph are illustrated in the figure below).

\includegraphics[scale=0.45]{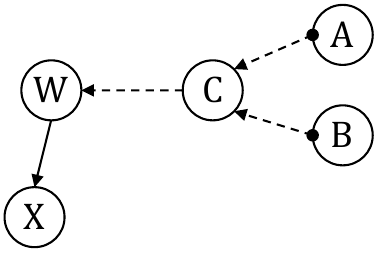}

Thus there is a directed path $C \rightarrow \hdots \rightarrow W \rightarrow X$.
This would mean that $W \in \mns_X(A)$ and $W \in \mns_X(B)$. 
Thus Line~\ref{alg-ldecc:if-cond-ecc-mark} will mark $W$ as a parent.

\emph{Meek Rule $2$:}

Consider a parent $W_2$ that gets oriented due to the application of Meek's rule $2$.
In this case, we have an oriented path $W_2 \rightarrow W_1 \rightarrow X$
but the edge $W_2 \text{---} X$ is unoriented (and Meek Rule 2 must be applied to orient it).

The first possibility is that 
the $X \leftarrow W_1$ was oriented due to 
some UC $A \rightarrow C \leftarrow B$ with a path $C \rightarrow \hdots \rightarrow W_1$
(relevant components of the graph are illustrated in the figure below):

\includegraphics[scale=0.45]{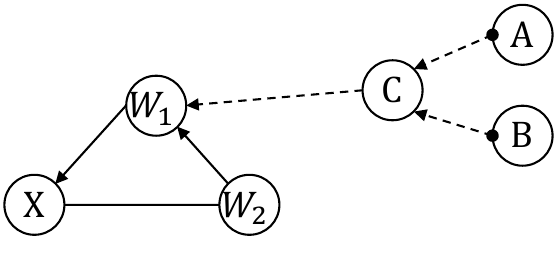}

In this case, there would be a collider at $W_1$: 
$C \rightarrow \hdots \rightarrow W_1 \leftarrow W_2$. 
Thus if $W_1 \in \mns_X(A)$, then $W_2 \in \mns_X(A)$ 
and thus LDECC will mark $W_2$ as a parent in Line~\ref{alg-ldecc:if-cond-ecc-mark}.

The other possibility is that $X \leftarrow W_1$ was oriented due to a UC like $W_3 \rightarrow X \leftarrow W_1$ but there is an edge $W_3 \text{---} W_2$
which causes the collider $W_2 \rightarrow X \leftarrow W_3$ to
be shielded and 
due to this, the $W_2 \text{---} X$ remained unoriented.
However, by definition of the Meek rule, 
the edge $W_2 \rightarrow W_1$ is oriented. 
Thus, 
(i) either there is a UC of the form $W_2 \rightarrow W_1 \leftarrow C$; or
(ii) there is a UC from which 
the $W_2 \rightarrow W_1$ orientation was propagated.
The relevant components of the graph for these two cases are illustrated in the figures below.

\includegraphics[scale=0.45]{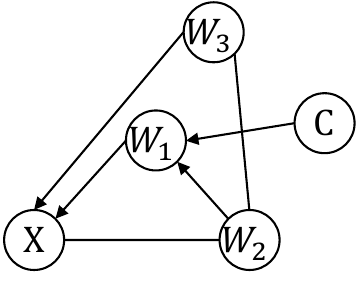}
\hspace{2em}
\includegraphics[scale=0.45]{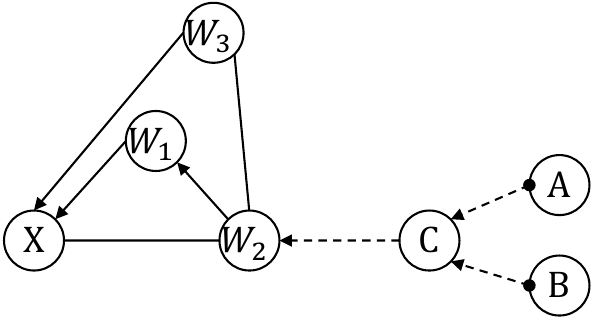}

For Case~(i), $W_2 \in \mns_X(C)$ and for Case~(ii),
$W_2 \in \mns_X(A)$ and $W_2 \in \mns_X(B)$.
In both cases, LDECC will mark $W_2$ as a parent in Lines~\ref{alg-ldecc:if-cond-ecc-mark}.

\emph{Meek Rule $3$:}

Consider a parent $W$ that gets oriented due to the application of Meek's rule $3$ 
(relevant component of the graph is shown in the figure below).

\includegraphics[scale=0.45]{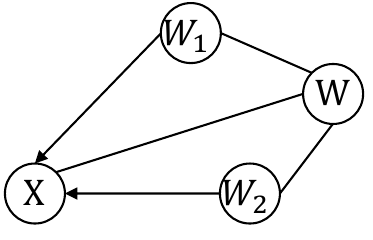}

By definition of the Meek rule, $W_1 \text{---} W \text{---} W_2$ is a non-collider (because if it were a collider, the edges would have been oriented since this triple is unshielded)
and therefore for any $\S \subseteq \V \setminus \{W_1, W_2\}$
such that $W_1 \indep W_2 | \S$, we have $W \in \S$.
Thus Line~\ref{alg-ldecc:meek-rule-3} will mark $W$ as a parent.

\emph{Meek Rule $4$:}

Consider a parent $W$ that gets oriented due to 
the application of Meek's rule $4$. The relevant component of the graph is shown in the figure below.

\includegraphics[scale=0.45]{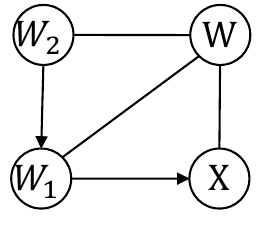}

The first possibility is that the orientations
$W_2 \rightarrow W_1 \rightarrow X$ were propagated 
from a UC $A \rightarrow C \leftarrow B$ with a path $C \rightarrow \hdots \rightarrow W_2$
(the relevant components of the graph are shown in the figure below).

\includegraphics[scale=0.45]{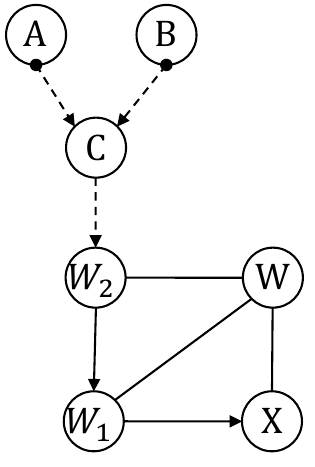}

In this case, due to the non-collider $W_2 \text{---} W \text{---} X$
(because if it were a collider, the edges would have been oriented since this triple is unshielded), 
we have $W \in \mns_X(A)$ and thus $W$ will be marked as a parent in Line~\ref{alg-ldecc:if-cond-ecc-mark}. 

The second possibility (similar to the Meek rule $2$ case) is that 
$W_2 \rightarrow W_1$ 
was oriented due to a UC like $Z \rightarrow W_1 \leftarrow W_2$ 
but there is an edge $Z \text{---} W$ which shields the
$Z \rightarrow W_1 \text{---} W$ causing the 
$W_1 \text{---} W$ edge to remain unoriented 
(the relevant components of the graph are shown in the figure below).

\includegraphics[scale=0.45]{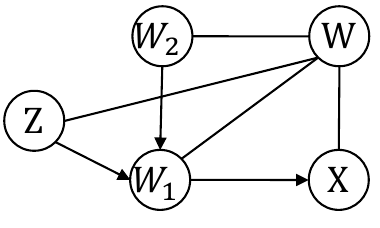}

In this case, we would have $W \in \mns_X(Z)$ and thus
$W$ gets marked as a parent in Line~\ref{alg-ldecc:if-cond-ecc-mark}.

\paragraph{Children are oriented correctly.}
We now similarly show that children of $X$ get oriented
correctly.

The simplest case is when there is a UC of the form
$X \rightarrow M \leftarrow V$.
Since $\MB(X)$ and $\DNe(X)$ are correct,
the function \emph{GetUCChildren} (Fig.~\ref{fig:algo-ldecc-functions}) will mark $M$ as a child.

Now, we consider each Meek rule one at a time and show that
LDECC will orient children for each rule.

\emph{Meek Rule $1$:}

Consider a child $M$ that gets oriented due to the application of Meek's rule $1$.
This can only happen if there is some parent $W$ that gets oriented
and $W \text{---} X \text{---} M$ forms an unshielded non-collider.
In this case, Line~\ref{alg-ldecc:mark-child-using-non-coll} will mark $M$ as a child.

\emph{Meek Rule $2$:}

Consider a child $M_2$ that gets oriented due to the application of Meek's rule $2$:
there is an oriented path $X \rightarrow M_1 \rightarrow M_2$ but the 
$X \text{---} M_2$ edge is still unoriented 
(the relevant component of the graph is shown in the figure below).

\includegraphics[scale=0.45]{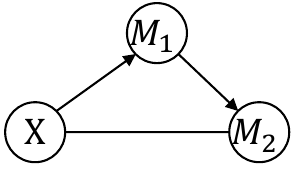}

One possibility is that there is UC of the form $V \rightarrow M_1 \rightarrow X$
which orients the $X \rightarrow M_1$ edge
and $V \text{---} M_1 \text{---} M_2$ is a non-collider which orients
the $M_1 \rightarrow M_2$ edge
(the relevant components of the graph are shown in the figure below).

\includegraphics[scale=0.45]{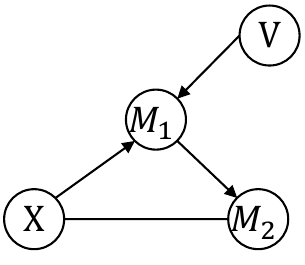}

In this case, since $V$ is a spouse (i.e., parent of child) of $X$,
the function \emph{GetUCChildren} (Fig.~\ref{fig:algo-ldecc-functions}) 
will mark $M_2$ as a child.

Note that if the $X \rightarrow M_1$ was oriented due to
a UC upstream of $X$ via the
application of Meek rule $1$, this would also
cause the $X \rightarrow M_2$ edge to be oriented 
(and thus Meek rule 2 would not apply).

The other possibility is that there might be a UC of the form
$M_1 \rightarrow M_2 \leftarrow Z$ that can orient $M_1 \rightarrow M_2$.
However, for the $X \rightarrow M_1$ edge to remain unoriented,
there must be an edge $Z \text{---} X$ to shield the $X \text{---} M_2 \text{---} Z$
collider. If this happens, Meek rule $3$ would apply (which we handle separately as shown next).

\emph{Meek Rule $3$:}

Consider a child $M$ that gets oriented due to the application of Meek's rule $3$
(the relevant component of the graph is shown in the figure below).

\includegraphics[scale=0.45]{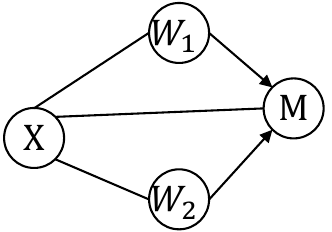}

By definition of the Meek rule, $W_1 \text{---} X \text{---} W_2$
is a non-collider 
(because if it were a collider, the edges would have been oriented since this triple is unshielded)
and since $W_1 \rightarrow M \leftarrow W_2$
forms a collider, we have $W_1 \notindep W_2 | \S \cup \{ M \}$
for any $\S$ s.t. $W_1 \indep W_2 | \S$.
Thus Line~\ref{alg-ldecc:meek-rule-3-and-4-child} will mark $M$ as a child.

\emph{Meek Rule $4$:}

Consider a child $M$ that gets oriented due to the application of Meek's rule $4$
(the relevant component of the graph is shown in the figure below).

\includegraphics[scale=0.45]{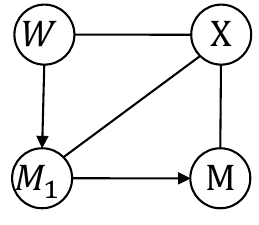}

One possibility such that the $W \rightarrow M_1$ gets oriented leaving the edges
$W \text{---} X$, $X \text{---} M$, and $X \text{---} M_1$ unoriented
is if there is a UC of the form $V \rightarrow M_1 \leftarrow W$
where there is an edge $V \text{---} X$ to shield the $X \text{---} M_1$ edge
(the relevant component of the graph is shown in the figure below).

\includegraphics[scale=0.45]{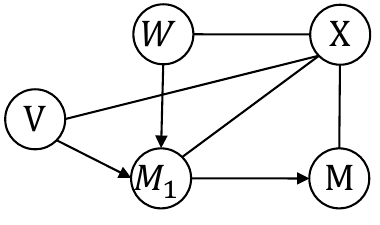}

Here the triple $W \text{---} X \text{---} V$ must be a non-collider to
keep the $X \text{---} M$ edge unoriented 
(otherwise applying Meek rule $1$ from the UC $W \rightarrow X \leftarrow V$ would orient $X \rightarrow M$).
So for any $\S$ such that $V \indep W | \S$, we must have $X \in \S$
and $V \notindep W | \S \cup \{M\}$.
Therefore, Line~\ref{alg-ldecc:meek-rule-3-and-4-child} will mark $M$ as a child.

The other possibility is that the $W \rightarrow M_1$ gets oriented due to
Meek rule $3$
(the relevant component of the graph is shown in the figure below).

\includegraphics[scale=0.45]{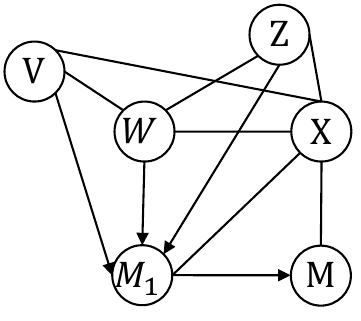}

Here the $Z \text{---} X$ and $V \text{---} X$ must be present to keep
the $X \text{---} M_1$ edge unoriented
(because otherwise an unshielded collider would be created).
Furthermore, the triple $Z \text{---} X \text{---} V$ must be a non-collider
in order to keep the $X \text{---} M$ edge unoriented
(otherwise applying Meek rule $1$ from the UC $Z \rightarrow X \leftarrow V$ would orient $X \rightarrow M$).
So for any $\S$ such that $V \indep Z | \S$, we must have $X \in \S$
and $V \notindep Z | \S \cup \{M\}$.
Therefore, Line~\ref{alg-ldecc:meek-rule-3-and-4-child} will mark $M$ as a child.

\paragraph{No spurious orientations.}
Now we prove that nodes are never oriented the wrong way by LDECC.

We show that \emph{GetUCChildren} (Fig.~\ref{fig:algo-ldecc-functions}) 
will never orient a parent as a child. 
We prove this by contradiction. Let's say there is a parent $W$ that is oriented as a child by \emph{GetUCChildren}. 
This will happen if there is a node $D \in \MB(X) \setminus \DNe(X)$ s.t.
$D \notindep W | \text{DSep}(D, X)$ with $W \notin \text{DSep}(D, X)$. 
Thus, there would a path $D \text{---} W$.
Since $W$ is a parent, there would be a path $D \text{---} W \rightarrow X$
leading to $W \in \text{DSep}(D, X)$ which is a contradiction.

Line~\ref{alg-ldecc:if-cond-nbr-uc-mark} can never mark a child as a parent 
since otherwise the CFA would be violated.

Line~\ref{alg-ldecc:meek-rule-3-and-4-child} will not mark a parent as a child. 
Consider a parent $M$ that incorrectly gets marked as a child by Line~\ref{alg-ldecc:meek-rule-3-and-4-child}
(the relevant component of the graph is shown in the figure below).

\includegraphics[scale=0.45]{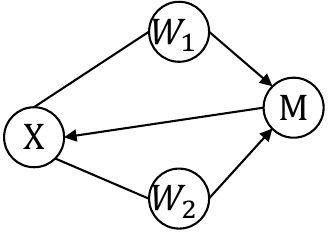}

For Line~\ref{alg-ldecc:meek-rule-3-and-4-child} to be reached, 
the if-condition in Line~\ref{alg-ldecc:if-cond-non-coll} must be \emph{True}.
This will happen if
$W1 \text{---} X \text{---} W2$ is a non-collider.
Thus at least one of $W_1$ or $W_2$ is a child. 
W.l.o.g., let's assume that $W_1$ is a child.
If that happens, a cycle gets created: $X \rightarrow W_1 \rightarrow M \rightarrow X$.
Therefore $M$ can never be oriented by Line~\ref{alg-ldecc:meek-rule-3-and-4-child}.

Line~\ref{alg-ldecc:if-cond-ecc-mark} cannot mark a child as a parent because of
the correctness of the ECC (Prop.~\ref{prop:ecc}).

Line~\ref{alg-ldecc:meek-rule-3} will not mark a child as a parent. 
Both $A$ and $B$
from Line~\ref{alg-ldecc:if-cond-nbr-uc} are parents of $X$.
By Lemma~\ref{lemma:apdx-child-cannot-sep-two-non-desc},
the set $\S$ in Line~\ref{alg-ldecc:meek-rule-3} cannot contain a child.

Line~\ref{alg-ldecc:mark-child-using-non-coll} will not add a child as a parent
because otherwise the CFA would be violated.
\end{proof}

\subsection{Omitted proofs for computational requirements (Sec.~\ref{sec:comparison-of-tests})} 
\label{sec:apdx-computational}
\newtheorem*{prop:num-tests-pc-vs-ldecc}{Proposition~\ref{prop:num-tests-pc-vs-ldecc}}
\begin{prop:num-tests-pc-vs-ldecc}[PC vs LDECC]
We have $T_{\text{LDECC}} \leq T_{\text{PC}} + \OM(|\V|^2) + \OM\left(|\V| \cdot 2^{|\DNe(X)|}\right).$
\end{prop:num-tests-pc-vs-ldecc}
\begin{proof}
LDECC performs $\OM\left(|\V| \cdot 2^{|\DNe(X)|}\right)$ CI tests to find $\DNe(X)$.
After that LDECC runs CI tests like PC.
The \emph{GetMNS} function requires $\OM(2^{|\DNe(X)|})$ CI tests.
The $\OM(|\V|^2)$ term accounts for the
extra CI tests of the form $A \notindep B | \S \cup \{X\}$
we might run for ECCs.
\end{proof}

\newtheorem*{prop:LDECC-upper-bound}{Proposition~\ref{prop:LDECC-upper-bound}}
\begin{prop:LDECC-upper-bound}[LDECC upper bound]
For a locally orientable DAG $\GM^{*}$, 
let $D = \max_{V \in \MBp(X)} \\ |\DNe(V)|$ and $S = \max_{ P \in \Pa(X)} \min_{\alpha \in \text{POC}(P)} \text{sep}(\alpha)$.
Then $T_{\text{LDECC}} \leq \OM(|\V|^{\max\{ S, D \}})$.
\end{prop:LDECC-upper-bound}
\begin{proof}
Since the graph is locally orientable,
all neighbors of $X$ will get oriented.
The complexity of discovering the neighbors of $X$ is
upper bounded by $\OM(|\V|^{|\MBp(X)|})$;
that of discovering any non-colliders of the
form $A \text{---} X \text{---} B$ is $\OM(|\V|^{D})$;
and in order to unshield the colliders that orient
the parents, LDECC runs $\OM(|\V|^S)$ tests.
Thus the total number of CI tests is $\OM(|\V|^{\max\{ S, D \}})$.
\end{proof}

\newtheorem*{prop:sd-upper-bound}{Proposition~\ref{prop:sd-upper-bound}}
\begin{prop:sd-upper-bound}[SD upper bound]
For a locally orientable DAG $\GM^{*}$,
let $\pi : \V \rightarrow \mathbb{N}$ be
the order in which nodes are processed by SD
(Line~\ref{alg:sd-queue-pop} of SD).
For $P \in \Pa(X)$, let
$\CUC(P) = \argmin_{\alpha \in \text{POC}(P)}\pi(\alpha)$
denote the closest UC
to $P$.
Let $C = \max_{ P \in \Pa(X)} \text{sep}(\CUC(P))$,
$D = \max_{V \in \MBp(X)} |\DNe(V)|$, and
$E = \max_{ \{ V : \pi(V) < \pi(\CUC(P)) \} } |\DNe(V)|$.
Then $T_{\text{SD}} \leq \OM(|\V|^{\max\{C, D, E \}})$.
\end{prop:sd-upper-bound}
\begin{proof}
Like LDECC, the complexity of discovering the
neighbors of $X$ and
non-colliders of the form $A \text{---} X \text{---} B$
is at most $\OM(|\V|^D)$ tests.
Then, in order to sequentially reach the closest UCs,
SD runs $\OM(|\V|^{E})$ tests. 
Once the collider is reached, SD runs
$\OM(|\V|^C)$ CI tests to unshield them.
\end{proof}

\subsection{Omitted proofs for faithfulness requirements (Sec.~\ref{sec:faithfulness})} \label{sec:apdx-faithfulness}
\newtheorem*{prop:pc-mec}{Proposition~\ref{prop:pc-mec}}
\begin{prop:pc-mec}
PC will identify the MEC of $\GM^*$ if LF holds for all nodes.
\end{prop:pc-mec}
\begin{proof}
It is known that the PC algorithm correctly identifies the
MEC of $\GM^*$ if AF and OF hold for all nodes (see e.g., \citet{zhang2008detection}).
AF for all nodes ensures that the skeleton is recovered correctly.
OF for all unshielded triples ensures that UCs are detected
correctly and that the orientations propagated via Meek's rules
are correct. 
\end{proof}

\newtheorem*{prop:sd-ff}{Proposition~\ref{prop:sd-ff}}
\begin{prop:sd-ff}[Faithfulness for PC and SD]
PC and SD will identify $\Theta^{*}$ if
(i) LF holds $\forall V \in \MBp(X)$;
(ii) $\forall (A \rightarrow C \leftarrow B) \in J^*$,
(a) LF holds for $A$, $B$, and $C$, and
(b) LF holds for each node on all paths $C \rightarrow \hdots \rightarrow V \in \GM^*$ 
s.t. $V \in \DNe(X)$;
(iii) For every edge $A\text{---}B \notin \GM^*$, $\exists \S \subseteq (\DNe_{\UM}(A) \cup \DNe_{\UM}(B))$ s.t. $A \indep B | \S$; and
(iv) OF holds for all unshielded triples in $\GM^*$.
\end{prop:sd-ff}
\begin{proof}
Similar to the proof of Thm.~\ref{thm:ldecc-correctness}, we will 
prove this by showing that all neighbors of $X$ get oriented
correctly and the unorientable neighbors remain unoriented.

The key ideas of the proof are as follows:
(1) Condition~(i) guarantees that the structure inside $\MBp(X)$
is discovered correctly which further ensures that Meek rules $2$-$4$
work correctly (since, as shown in the proof of Thm.~\ref{thm:ldecc-correctness},
they are only applied inside $\MBp(X)$);
(2) Condition~(ii)(a) guarantees that each UC in $\GM^*$ is detected
and unshielded;
(3) Condition~(ii)(b) guarantees that orientations from each UC in $\GM^*$
are propagated correctly to $X$;
(4) Condition~(iii) guarantees that the undirected skeleton discovered by
PC and SD is a subgraph of the skeleton of $\GM^*$: This is because
PC and SD remove an edge $A \text{---} B$ 
by running CI tests by conditioning on neighbors of $A$ and $B$
in the current undirected skeleton; and
(5) Condition~(iv) guarantees that in the skeleton recovered by PC and SD,
there are no incorrectly detected UCs.

\paragraph{Parents are oriented correctly.}
In PC and SD, edges get oriented using UCs and then additional orientations are
propagated via the application of Meek's rules (Figure~\ref{fig:meek-rules}). 

The simplest case is where two parents form a UC at $X$. 
Consider parents $W_1$ and $W_2$ that get oriented because 
they form a UC $W_1 \rightarrow X \leftarrow W_2$. 
Condition~(i) ensures they are marked as parents.

We will now consider parents that get oriented due
to each of the four Meek rules and show that SD and PC
orient parents for each of the four cases.

\emph{Meek Rule $1$:}

Consider a parent $W$ that gets oriented due to the application of Meek's rule $1$. 
This can only happen due to some UC $A \rightarrow C \leftarrow B$ from which these orientations have been propagated (relevant components of the graph are illustrated in the figure below).

\includegraphics[scale=0.45]{figures/ldecc-correctness-proof/parent-meek-rule-1.pdf}

Thus there is a directed path $C \rightarrow \hdots \rightarrow W \rightarrow X$.
By Condition~(ii)(a), the UC $A \rightarrow C \leftarrow B$ will get
detected correctly, and by Condition~(ii)(b), the orientations will 
be propagated correctly along this path.

\emph{Meek Rule $2$:}

Consider a parent $W_2$ that gets oriented due to the application of Meek's rule $2$.
In this case, we have an oriented path $W_2 \rightarrow W_1 \rightarrow X$
but the edge $W_2 \text{---} X$ is unoriented (and Meek Rule 2 must be applied to orient it).
Condition~(i) ensures that this relevant component of the graph
is discovered correctly.

\emph{Meek Rule $3$:}

Consider a parent $W$ that gets oriented due to the application of Meek's rule $3$ 
(relevant component of the graph is shown in the figure below).

\includegraphics[scale=0.45]{figures/ldecc-correctness-proof/parent-meek-rule-3.pdf}

Condition~(i) ensures that this relevant component of the graph
is discovered correctly.

\emph{Meek Rule $4$:}

Consider a parent $W$ that gets oriented due to 
the application of Meek's rule $4$. The relevant component of the graph is shown in the figure below.

\includegraphics[scale=0.45]{figures/ldecc-correctness-proof/parent-meek-rule-4-main-component.pdf}

Condition~(i) ensures that this relevant component of the graph
is discovered correctly.

\paragraph{Children are oriented correctly.}
We now similarly show that children of $X$ get oriented
correctly.

The simplest case is when there is a UC of the form
$X \rightarrow M \leftarrow V$.
Condition~(i) ensures that this UC is discovered correctly.

Now, we consider each Meek rule one at a time and show that
SD and PC will orient children for each rule.

\emph{Meek Rule $1$:}

Consider a child $M$ that gets oriented due to the application of Meek's rule $1$.
This can only happen if there is some parent $W$ that gets oriented
and $W \text{---} X \text{---} M$ forms an unshielded non-collider.
Condition~(i) ensures that this unshielded non-collider is discovered correctly.

\emph{Meek Rule $2$:}

Consider a child $M_2$ that gets oriented due to the application of Meek's rule $2$:
there is an oriented path $X \rightarrow M_1 \rightarrow M_2$ but the 
$X \text{---} M_2$ edge is still unoriented 
(the relevant component of the graph is shown in the figure below).

\includegraphics[scale=0.45]{figures/ldecc-correctness-proof/child-meek-rule-2-main-component.pdf}

Condition~(i) ensures that this relevant component of the graph
is discovered correctly.

\emph{Meek Rule $3$:}

Consider a child $M$ that gets oriented due to the application of Meek's rule $3$
(the relevant component of the graph is shown in the figure below).

\includegraphics[scale=0.45]{figures/ldecc-correctness-proof/child-meek-rule-3.pdf}

Condition~(i) ensures that this relevant component of the graph
is discovered correctly.

\emph{Meek Rule $4$:}

Consider a child $M$ that gets oriented due to the application of Meek's rule $4$
(the relevant component of the graph is shown in the figure below).

\includegraphics[scale=0.45]{figures/ldecc-correctness-proof/child-meek-rule-4-main-component.pdf}

Condition~(i) ensures that this relevant component of the graph
is discovered correctly.

\paragraph{No spurious orientations.}

As argued in the preamble of the proof, Conditions~(iii, iv)
ensure that there are no incorrectly detected UCs.
Condition~(ii) ensures that no incorrect orientations are propagated
from the detected UCs.
\end{proof}

\newtheorem*{prop:ldecc-ff}{Proposition~\ref{prop:ldecc-ff}}
\begin{prop:ldecc-ff}[Faithfulness for LDECC]
LDECC will identify $\Theta^{*}$  
if (i) LF holds $\forall V \in \MBp(X)$;
(ii) $H \subseteq H^*$;
(iii) $\forall (A, B, \S) \in H$, MFF holds for $\{A, B\} \setminus \DNe(X)$; and
(iv) $\forall (A, B, \S) \in H^*$ s.t. there is a UC $(A \rightarrow C \leftarrow B) \in \GM^*$,
we have (a) AF holds for $A$ and $B$; and (b) $(A, B, \S) \in H$.
\end{prop:ldecc-ff}
\begin{proof}
The proof is very similar to that of Theorem~\ref{thm:ldecc-correctness}.
The high-level idea is as follows.
For any nodes in $\DNe(X)$ that are oriented \emph{without}
using ECCs, Condition~(i) will ensure they get oriented correctly
as they only depend on the structure inside $\MB(X)$.
For nodes that get oriented via ECCs, Condition~(iv)(a) ensures that each UC eventually gets
unshielded.
For any UC $A \rightarrow C \leftarrow B$ in Condition~(iv),
since AF holds for both $A$ and $B$,
$\DNe(A)$ and $\DNe(B)$ are detected correctly.
Since $\exists \S \subseteq (\DNe(A) \cup \DNe(B))$ s.t. 
$A \indep B | \S$,
we will eventually remove the $A \text{---} B$
thereby unshielding this collider.
Condition~(iv)(b) ensures that we run an ECC for this UC,
i.e., the if-block in Line~\ref{alg-ldecc:if-cond-ecc} is entered.
Next, by Condition~(iii), since MFF holds for $A$ and $B$,
the \emph{GetMNS} function will correctly return the parents of $X$ that
this UC orients.

We assume that the function \emph{FindMarkovBlanket} identifies the
Markov blanket correctly under Condition~(i). The IAMB algorithm,
which we use in our experiments, has this property.
Additionally, the function \emph{RunLocalPC} will also identify
the correct $\DNe(X)$ under Condition~(i).

\paragraph{Parents are oriented correctly.}

The simplest case is where two parents form a UC at $X$. 
Consider parents $W_1$ and $W_2$ that get oriented because 
they form a UC $W_1 \rightarrow X \leftarrow W_2$. 
By Condition~(i), Lines~\ref{alg-ldecc:if-cond-nbr-uc},\ref{alg-ldecc:if-cond-nbr-uc-mark} 
will mark $W_1$ and $W_2$ as parents.

We will now consider parents that get oriented due
to each of the four Meek rules and show that LDECC
orients parents for each of the four cases.

\emph{Meek Rule $1$:}

Consider a parent $W$ that gets oriented due to the application of Meek's rule $1$. 
This can only happen due to some UC $A \rightarrow C \leftarrow B$ from which these orientations have been propagated (relevant components of the graph are illustrated in the figure below).

\includegraphics[scale=0.45]{figures/ldecc-correctness-proof/parent-meek-rule-1.pdf}

As explained in the preamble, Conditions~(iii, iv) ensure that
ECCs mark parent correctly and thus Line~\ref{alg-ldecc:if-cond-ecc-mark} will mark $W$ as a parent.

\emph{Meek Rule $2$:}

Consider a parent $W_2$ that gets oriented due to the application of Meek's rule $2$.
In this case, we have an oriented path $W_2 \rightarrow W_1 \rightarrow X$
but the edge $W_2 \text{---} X$ is unoriented (and Meek Rule 2 must be applied to orient it).

The first possibility is that 
the $X \leftarrow W_1$ was oriented due to 
some UC $A \rightarrow C \leftarrow B$ with a path $C \rightarrow \hdots \rightarrow W_1$
(relevant components of the graph are illustrated in the figure below):

\includegraphics[scale=0.45]{figures/ldecc-correctness-proof/parent-meek-rule-2a.pdf}

By Conditions~(iii, iv), LDECC will mark $W_2$ as a parent in Line~\ref{alg-ldecc:if-cond-ecc-mark}.

The other possibility is that $X \leftarrow W_1$ was oriented due to a UC like $W_3 \rightarrow X \leftarrow W_1$ but there is an edge $W_3 \text{---} W_2$
which causes the collider $W_2 \rightarrow X \leftarrow W_3$ to
be shielded and 
due to this, the $W_2 \text{---} X$ remained unoriented.
However, by definition of the Meek rule, 
the edge $W_2 \rightarrow W_1$ is oriented. 
Thus, 
(i) either there is a UC of the form $W_2 \rightarrow W_1 \leftarrow C$; or
(ii) there is a UC from which 
the $W_2 \rightarrow W_1$ orientation was propagated.
The relevant components of the graph for these two cases are illustrated in the figures below.

\includegraphics[scale=0.45]{figures/ldecc-correctness-proof/parent-meek-rule-2b.pdf}
\hspace{2em}
\includegraphics[scale=0.45]{figures/ldecc-correctness-proof/parent-meek-rule-2c.pdf}

In both cases, by Conditions~(iii, iv), LDECC will mark $W_2$ as a parent in Lines~\ref{alg-ldecc:if-cond-ecc-mark}.

\emph{Meek Rule $3$:}

Consider a parent $W$ that gets oriented due to the application of Meek's rule $3$ 
(relevant component of the graph is shown in the figure below).

\includegraphics[scale=0.45]{figures/ldecc-correctness-proof/parent-meek-rule-3.pdf}

By definition of the Meek rule, $W_1 \text{---} W \text{---} W_2$ is a non-collider (because if it were a collider, the edges would have been oriented since this triple is unshielded)
and therefore for any $\S \subseteq \V \setminus \{W_1, W_2\}$
such that $W_1 \indep W_2 | \S$,
by Condition~(i), we have $W \in \S$.
Thus Line~\ref{alg-ldecc:meek-rule-3} will mark $W$ as a parent.

\emph{Meek Rule $4$:}

Consider a parent $W$ that gets oriented due to 
the application of Meek's rule $4$. The relevant component of the graph is shown in the figure below.

\includegraphics[scale=0.45]{figures/ldecc-correctness-proof/parent-meek-rule-4-main-component.pdf}

The first possibility is that the orientations
$W_2 \rightarrow W_1 \rightarrow X$ were propagated 
from a UC $A \rightarrow C \leftarrow B$ with a path $C \rightarrow \hdots \rightarrow W_2$
(the relevant components of the graph are shown in the figure below).

\includegraphics[scale=0.45]{figures/ldecc-correctness-proof/parent-meek-rule-4a.pdf}

In this case, due to the non-collider $W_2 \text{---} W \text{---} X$
(because if it were a collider, the edges would have been oriented since this triple is unshielded), 
we have $W \in \mns_X(A)$ and thus by Conditions~(iii, iv), $W$ will be marked as a parent in Line~\ref{alg-ldecc:if-cond-ecc-mark}. 

The second possibility (similar to the Meek rule $2$ case) is that 
$W_2 \rightarrow W_1$ 
was oriented due to a UC like $Z \rightarrow W_1 \leftarrow W_2$ 
but there is an edge $Z \text{---} W$ which shields the
$Z \rightarrow W_1 \text{---} W$ causing the 
$W_1 \text{---} W$ edge to remain unoriented 
(the relevant components of the graph are shown in the figure below).

\includegraphics[scale=0.45]{figures/ldecc-correctness-proof/parent-meek-rule-4b.pdf}

In this case, we would have $W \in \mns_X(Z)$ and thus
by Conditions~(iii, iv), $W$ gets marked as a parent in Line~\ref{alg-ldecc:if-cond-ecc-mark}.

\paragraph{Children are oriented correctly.}
We now similarly show that children of $X$ get oriented
correctly.

The simplest case is when there is a UC of the form
$X \rightarrow M \leftarrow V$.
By Condition~(i), since $\MB(X)$ and $\DNe(X)$ are correct,
the function \emph{GetUCChildren} (Fig.~\ref{fig:algo-ldecc-functions}) 
will mark $M$ as a child.

Now, we consider each Meek rule one at a time and show that
LDECC will orient children for each rule.

\emph{Meek Rule $1$:}

Consider a child $M$ that gets oriented due to the application of Meek's rule $1$.
This can only happen if there is some parent $W$ that gets oriented
and $W \text{---} X \text{---} M$ forms an unshielded non-collider.
By Condition~(i), Line~\ref{alg-ldecc:mark-child-using-non-coll} will mark $M$ as a child.

\emph{Meek Rule $2$:}

Consider a child $M_2$ that gets oriented due to the application of Meek's rule $2$:
there is an oriented path $X \rightarrow M_1 \rightarrow M_2$ but the 
$X \text{---} M_2$ edge is still unoriented 
(the relevant component of the graph is shown in the figure below).

\includegraphics[scale=0.45]{figures/ldecc-correctness-proof/child-meek-rule-2-main-component.pdf}

One possibility is that there is UC of the form $V \rightarrow M_1 \rightarrow X$
which orients the $X \rightarrow M_1$ edge
and $V \text{---} M_1 \text{---} M_2$ is a non-collider which orients
the $M_1 \rightarrow M_2$ edge
(the relevant components of the graph are shown in the figure below).

\includegraphics[scale=0.45]{figures/ldecc-correctness-proof/child-meek-rule-2a.pdf}

In this case, since $V$ is a spouse (i.e., parent of child) of $X$,
by Condition~(i),
the function \emph{GetUCChildren} (Fig.~\ref{fig:algo-ldecc-functions}) 
will mark $M_2$ as a child.

Note that if the $X \rightarrow M_1$ was oriented due to
a UC upstream of $X$ via the
application of Meek rule $1$, this would also
cause the $X \rightarrow M_2$ edge to be oriented 
(and thus Meek rule 2 would not apply).

The other possibility is that there might be a UC of the form
$M_1 \rightarrow M_2 \leftarrow Z$ that can orient $M_1 \rightarrow M_2$.
However, for the $X \rightarrow M_1$ edge to remain unoriented,
there must be an edge $Z \text{---} X$ to shield the $X \text{---} M_2 \text{---} Z$
collider. If this happens, Meek rule $3$ would apply (which we handle separately as shown next).

\emph{Meek Rule $3$:}

Consider a child $M$ that gets oriented due to the application of Meek's rule $3$
(the relevant component of the graph is shown in the figure below).

\includegraphics[scale=0.45]{figures/ldecc-correctness-proof/child-meek-rule-3.pdf}

By definition of the Meek rule, $W_1 \text{---} X \text{---} W_2$
is a non-collider 
(because if it were a collider, the edges would have been oriented since this triple is unshielded)
and since $W_1 \rightarrow M \leftarrow W_2$
forms a collider, we have $W_1 \notindep W_2 | \S \cup \{ M \}$
for any $\S$ s.t. $W_1 \indep W_2 | \S$.
Thus, by Condition~(i), Line~\ref{alg-ldecc:meek-rule-3-and-4-child} will mark $M$ as a child.

\emph{Meek Rule $4$:}

Consider a child $M$ that gets oriented due to the application of Meek's rule $4$
(the relevant component of the graph is shown in the figure below).

\includegraphics[scale=0.45]{figures/ldecc-correctness-proof/child-meek-rule-4-main-component.pdf}

One possibility such that the $W \rightarrow M_1$ gets oriented leaving the edges
$W \text{---} X$, $X \text{---} M$, and $X \text{---} M_1$ unoriented
is if there is a UC of the form $V \rightarrow M_1 \leftarrow W$
where there is an edge $V \text{---} X$ to shield the $X \text{---} M_1$ edge
(the relevant component of the graph is shown in the figure below).

\includegraphics[scale=0.45]{figures/ldecc-correctness-proof/child-meek-rule-4.pdf}

Here the triple $W \text{---} X \text{---} V$ must be a non-collider to
keep the $X \text{---} M$ edge unoriented 
(otherwise applying Meek rule $1$ from the UC $W \rightarrow X \leftarrow V$ would orient $X \rightarrow M$).
So for any $\S$ such that $V \indep W | \S$,
by Condition~(i), we must have $X \in \S$
and $V \notindep W | \S \cup \{M\}$.
Therefore, Line~\ref{alg-ldecc:meek-rule-3-and-4-child} will mark $M$ as a child.

The other possibility is that the $W \rightarrow M_1$ gets oriented due to
Meek rule $3$
(the relevant component of the graph is shown in the figure below).

\includegraphics[scale=0.45]{figures/ldecc-correctness-proof/child-meek-rule-4b.pdf}

Here the $Z \text{---} X$ and $V \text{---} X$ must be present to keep
the $X \text{---} M_1$ edge unoriented
(because otherwise an unshielded collider would be created).
Furthermore, the triple $Z \text{---} X \text{---} V$ must be a non-collider
in order to keep the $X \text{---} M$ edge unoriented
(otherwise applying Meek rule $1$ from the UC $Z \rightarrow X \leftarrow V$ would orient $X \rightarrow M$).
So for any $\S$ such that $V \indep Z | \S$,
by Condition~(i), we must have $X \in \S$
and $V \notindep Z | \S \cup \{M\}$.
Therefore, Line~\ref{alg-ldecc:meek-rule-3-and-4-child} will mark $M$ as a child.

\paragraph{No spurious orientations.}
Now we prove that nodes are never oriented the wrong way by LDECC.

By Condition~(i), \emph{GetUCChildren} (Fig.~\ref{fig:algo-ldecc-functions}) 
will never orient a parent as a child.

Line~\ref{alg-ldecc:if-cond-nbr-uc-mark} can never mark a child as a parent 
since otherwise Condition~(i) would be violated.

Line~\ref{alg-ldecc:meek-rule-3-and-4-child} will not mark a parent as a child. 
Consider a parent $M$ that incorrectly gets marked as a child by Line~\ref{alg-ldecc:meek-rule-3-and-4-child}
(the relevant component of the graph is shown in the figure below).

\includegraphics[scale=0.45]{figures/ldecc-correctness-proof/parent-spurious-1.pdf}

For Line~\ref{alg-ldecc:meek-rule-3-and-4-child} to be reached, 
the if-condition in Line~\ref{alg-ldecc:if-cond-non-coll} must be \emph{True}.
This will happen if
$W1 \text{---} X \text{---} W2$ is a non-collider.
Thus at least one of $W_1$ or $W_2$ is a child. 
W.l.o.g., let's assume that $W_1$ is a child.
If that happens, a cycle gets created: $X \rightarrow W_1 \rightarrow M \rightarrow X$.
Therefore, Condition~(i) ensures that $M$ can never be oriented by Line~\ref{alg-ldecc:meek-rule-3-and-4-child}.

By Condition~(ii), every ECC that is run is valid, and thus
Line~\ref{alg-ldecc:if-cond-ecc-mark} cannot mark a child as a parent because of
the correctness of the ECC (Prop.~\ref{prop:ecc}).

Similarly, by Condition~(i), Line~\ref{alg-ldecc:meek-rule-3} 
will not mark a child as a parent. 
Both $A$ and $B$
from Line~\ref{alg-ldecc:if-cond-nbr-uc} are parents of $X$.
By Lemma~\ref{lemma:apdx-child-cannot-sep-two-non-desc},
a child cannot d-separate two non-descendants nodes
and thus the set $\S$ in Line~\ref{alg-ldecc:meek-rule-3} cannot contain a child.

Line~\ref{alg-ldecc:mark-child-using-non-coll} will not add a child as a parent
because otherwise Condition~(i) would be violated.
\end{proof}

We now provide sufficient faithfulness conditions for LDECC when \emph{ECCParents}($A, B, \S$)
is run with \emph{check=True}.

For a node $V \notin \DNep(X)$, let $\Q(V) = \{ \S \subseteq \DNe(X) : V \indep X | \S \}$
and $\Q_{\min}(V) = \{ \S \in \Q(V) : |\S| = \min_{\S' \in \Q(V)} |\S'| \}$.
With $H^*$ and $H$ as defined in Sec.~\ref{sec:faithfulness},
let $H^{\text{(check)}} = \{ (A, B, \S) \in H : \{A, B\} \cap \DNe(X) = \emptyset; \, \text{and} \, \Q_{\min}(A) \cap \Q_{\min}(B) \neq \emptyset \}$;
let $H^{\text{(single)}} = \{ (A, B, \S) \in H : |\{A, B\} \cap \DNe(X)| = 1 \}$; and
let $H^{\text{(total)}} = H^{\text{(check)}} \cup H^{\text{(single)}}$.

\begin{lemma}\label{lemma:apdx-mns-uc-same}
For a true graph $\GM^*$, let $(A \rightarrow C \leftarrow B)$
be a UC such that $(A, B, \S) \in H^*$ for some subset $\S$.
Then, we have that $\nss_X(A) = \nss_X(B)$.
\end{lemma}
\begin{proof}
Since $(A, B, \S) \in H^*$, by Prop.~\ref{prop:ecc},
this UC can be used to orient parents via an ECC.
We begin by considering node $A$. By definition of MNS, 
we have $A \indep X | \nss_X(A)$.
Since the nodes in $\nss_X(A)$ are on the path from 
$X \leftarrow \hdots \leftarrow C \leftarrow A$, conditioning on $\nss_X(A)$
opens up the $A \rightarrow C \leftarrow B$ path (because the nodes
in $\nss_X(A)$ are descendants of $C$).
Therefore, we have $\nss_X(A) \subseteq \nss_X(B)$.
We can make a similar argument for node $B$ to show that
$\nss_X(B) \subseteq \nss_X(A)$. Combining the two statements,
we get $\nss_X(A) = \nss_X(B)$.
\end{proof}

\begin{proposition}[Weaker faithfulness for LDECC]\label{prop:apdx-ldecc-weaker-ff}
If \emph{ECCParents}($A, B, \S$)
is run with \emph{check=True}, then LDECC will identify $\Theta^{*}$  
if: (i) LF holds for every $V \in \DNe^{+}(X)$;
(ii) $H^{\text{(total)}} \subseteq H^*$;
(iii) $\forall (A, B, \S) \in H^{\text{(total)}}$, MFF holds for $\{A, B\} \setminus \DNe(X)$; and
(iv) $\forall (A, B, \S) \in H^*$ s.t. there is a UC $(A \rightarrow C \leftarrow B) \in \GM^*$,
we have (a) AF holds for $A$ and $B$; and (b) $(A, B, \S) \in H^{\text{(total)}}$.
Futhermore, Conditions~(i)-(iv) of this proposition are implied
by the conditions of Prop.~\ref{prop:ldecc-ff} 
(i.e., this is a weaker sufficient faithfulness condition for LDECC).
\end{proposition}
\begin{proof}
The can be proved in the same way as Prop.~\ref{prop:ldecc-ff}
with the crucial difference that the set of ECCs that
LDECC now runs is restricted to the set $H^{\text{(total)}}$.
Condition~(ii) now ensures that every ECC that is run by LDECC
is a valid ECC.
We just have to show that even by restricting the ECCs to
$H^{\text{(total)}}$, we still run an ECC
for the UCs.
For this, we leverage Lemma~\ref{lemma:apdx-mns-uc-same} which states
that the MNS of nodes $A$ and $B$ for the UC $A \rightarrow C \leftarrow B$ will be the same.
Since, by Condition~(iii), MFF holds for such nodes $A$ and $B$,
the check \emph{GetMNS}($A$) $=$ \emph{GetMNS}($B$) will succeed
and thus, the ECCs for these UCs will still be run.

Now, we show that the conditions of this proposition are implied
by those of Prop.~\ref{prop:ldecc-ff}.
Conditions~(i),(iv)(a) of both propositions are the same.
Observe that $H^{\text{(total)}} \subseteq H$.
Therefore, Conditions~(ii),(iii) of Prop.~\ref{prop:ldecc-ff} imply Conditions~(ii),(iii) of this proposition.
Finally, we show that Condition~(iv)(b) of this proposition is also
implied by the conditions of Prop.~\ref{prop:ldecc-ff}.
Consider any $(A, B, \S) \in H^*$ such that the UC $A \rightarrow C \leftarrow B \in \GM^*$ and $\{ A, B \} \cap \DNe(X) = \emptyset$. 
If $|\{ A, B \} \cap \DNe(X)| = 1$, then $(A, B, \S) \in H^{\text{(single)}} \subseteq H^{\text{(total)}}$.
If $\{ A, B \} \cap \DNe(X) = \emptyset$, then by Condition~(iii) of Prop.~\ref{prop:ldecc-ff}, MFF holds
for $A$ and $B$ are therefore $Q_{\min}(A) \cap Q_{\min}(B) = \nss_X(A) = \nss_X(B)$. Thus, $(A, B, \S) \in H^{\text{(check)}} \subseteq H^{\text{(total)}}$.
\end{proof}

\newtheorem*{prop:test-mff}{Proposition~\ref{prop:test-mff}}
\begin{prop:test-mff}[Testing faithfulness for LDECC]
Consider running the algorithm in Fig.~\ref{fig:algo-test-mff}
before invoking \emph{GetMNS}($A$) for some node $A$ in LDECC.
If the algorithm returns \emph{Fail}, MFF is violated for node $A$.
If the algorithm returns \emph{Unknown}, 
we could not ascertain if MFF holds for node $A$.
\end{prop:test-mff}
\begin{proof}
The set $\Q$ contains all subsets $\S \subseteq \DNe(X)$
s.t. $A \indep X | \S$ (Lines~$1$--$4$).
The set $\Q_{\min}$ contains those sets from $\Q$ that are minimal,
i.e., for every $\S' \in \Q_{\min}$, 
there is no subset of $\S'$ in $\Q$.
If $|\Q_{\min}| > 1$, then there are multiple possible $\mns_X(A)$
violating the uniqueness of MNS (Prop.~\ref{prop:mns-unique}).
Thus we return \emph{Fail} (Example~\ref{example:testing-mff-1}
demonstrates this failure case).

Next, if Line~\ref{alg:test-mff-q-min-is-size-one} is reached, 
we know that $|\Q_{\min}| = 1$
and $\S$ is the single element from $\Q_{\min}$.
Line~\ref{alg:test-mff-failure-2} returns \emph{Fail} 
if there is a set $\S'$ such that
$\S \subset \S'$ and $\S' \in \Q$,
and an intermediate set $\S''$ such that $\S \subset \S'' \subset \S'$
and $\S'' \notin \Q$ 
(Example~\ref{example:testing-mff-2}
demonstrates this failure case).
Here MFF fails because if $\S$ was the correct $\mns_X(A)$,
then we must have $\S'' \in \Q$. 
This is because, since $\S' \in \Q$, there cannot be an
active path from $A$ to $X$ through nodes in $\S' \setminus \S$
(otherwise, we would not have $A \indep X | \S$).
Therefore, adding nodes in $\S'' \setminus \S$ to the conditioning
set should not violate the independence. But since we have
$\S'' \notin \Q$, there must be a faithfulness violation and
$\S$ is not guaranteed to be equal to $\mns_X(A)$.

If Line~\ref{alg:test-mff-end} is reached, 
we have not been able to detect an MFF
violation. However, the algorithm is not \emph{complete}, i.e.,
a failure to detect a violation does not mean a violation does
not exist. So we return \emph{Unknown} which signifies
that we were unable to ascertain if MFF was violated for $A$.

The additional tests are performed in Line~\ref{alg:test-mff-additional-tests}. 
Since these tests
are performed for every subset $\S \subseteq \DNe(X)$,
the number of extra CI tests is $\OM(2^{|\DNe(X)|})$.

\end{proof}

\newtheorem*{prop:test-sd-ff}{Proposition~\ref{prop:test-sd-ff}}
\begin{prop:test-sd-ff}[Testing faithfulness for SD]
Consider running the algorithm in Fig.~\ref{fig:test-sd-ff} with
each UC detected by SD.
If the algorithm returns \emph{Fail}, then faithfulness is violated for SD.
If the algorithm returns \emph{Unknown}, we could not
ascertain if the faithfulness assumptions for SD hold.
\end{prop:test-sd-ff}
\begin{proof}
In Line~\ref{alg:test-sd-ff-M-poc}, 
$\M$ contains every neighbor of $X$ 
that gets oriented as a parent due to the input UC
$A \rightarrow C \leftarrow B$ by SD.
If the faithfulness assumptions for SD did hold, then all
nodes in $\M$ would actually be parents of $X$.
Thus there should be at least one subset $\S \subseteq \DNe(X)$
such that $A \indep X | \S$ and $\M \subseteq \S$, and likewise for $B$.
If such a subset is not found, this means that one of the nodes
that was marked as a parent was actually a child.
In this case, Line~\ref{alg:test-sd-ff-fail} would return \emph{Fail}.
Similarly to the LDECC case, if we are unable to detect a faithfulness violation, 
we return \emph{Unknown} to indicate that we could not determine if the
faithfulness assumptions for SD hold.
Since CI tests
are performed for every subset $\S \subseteq \DNe(X)$,
the number of extra CI tests is $\OM(2^{|\DNe(X)|})$.
\end{proof}

\section{Experiments} \label{sec:apdx-experiments}
\begin{figure}
\centering
\subfigure[Accuracy]{\includegraphics[scale=0.35]{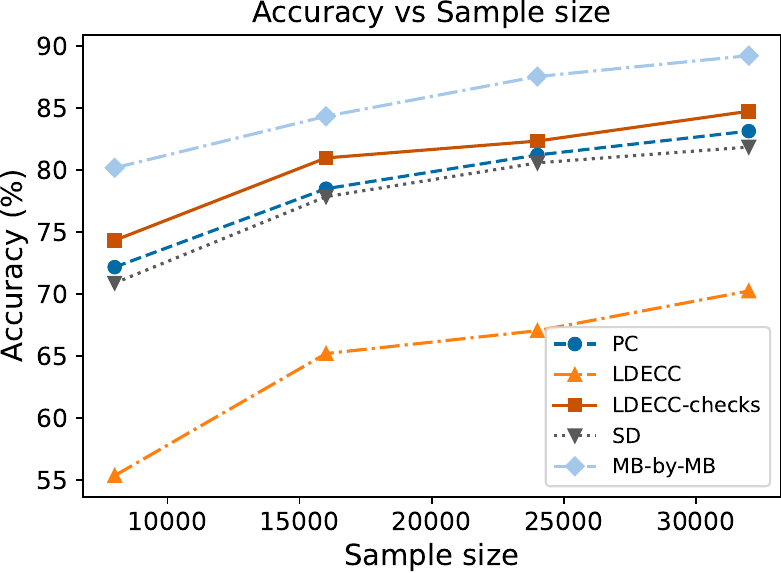}\label{fig:apdx-er-accuracy}}
\hspace{3em}
\subfigure[Recall]{\includegraphics[scale=0.35]{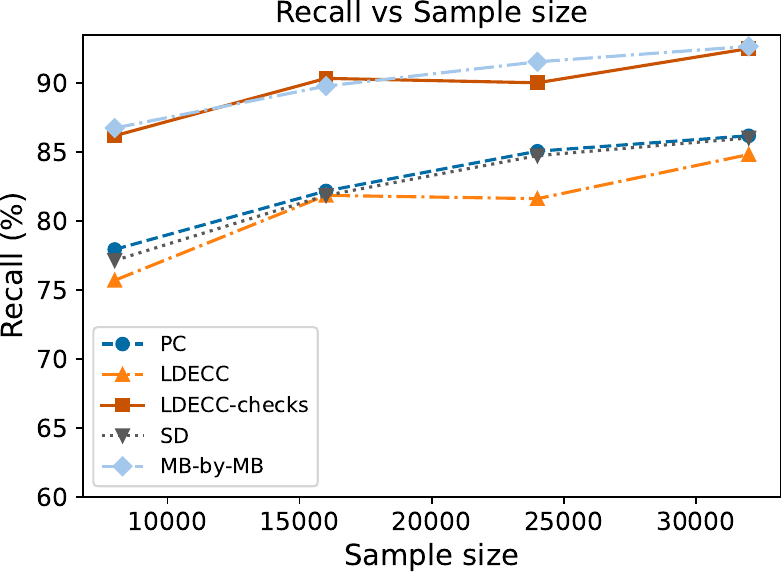}\label{fig:apdx-er-recall}}
\vspace{1em}
\\
\subfigure[MSE vs sample size]{\includegraphics[scale=0.33]{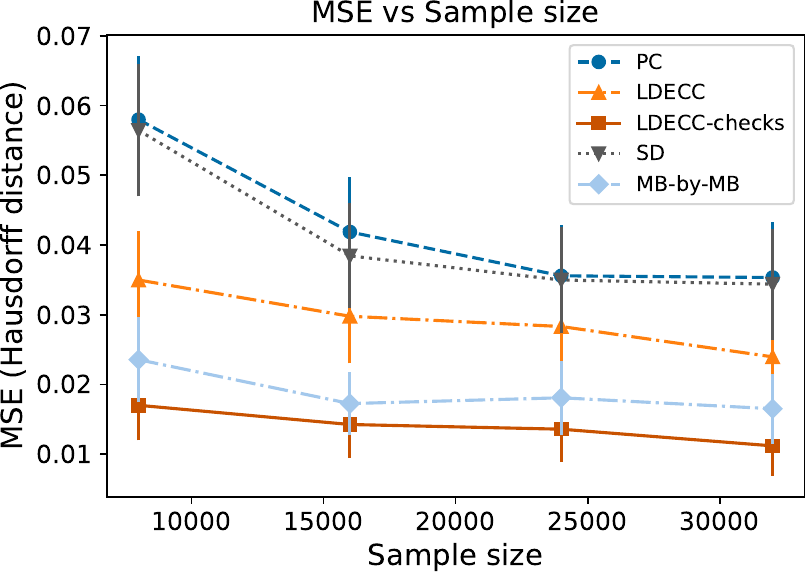}\label{fig:apdx-er-mse}}
\hspace{3em}
\subfigure[Average number of CI tests]{\includegraphics[scale=0.35]{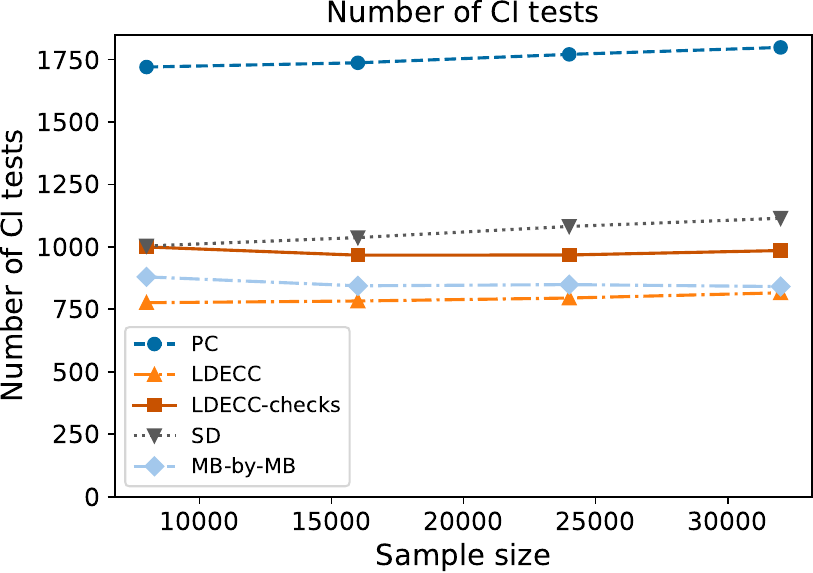}\label{fig:apdx-er-ci-tests}}
\caption{Results on synthetic linear Erdos-Renyi graphs.}
\end{figure}

\subsection{More details for the synthetic linear graph experiments.}\label{sec:apdx-experiments-synthetic-linear-dgp}

We generate synthetic linear graphs
with Gaussian errors,
$N_c=20$ covariates---non-descendants of $X$ and $Y$ 
with paths to both $X$ and $Y$---and $N_m=3$ mediators---nodes
on some causal paths from $X$ to $Y$. 
We generate edges between the different types of
nodes with varying probabilities:
(i) We connect the covariates to the treatment with probability $p_{cx}$;
(ii) We connect one covariate to another with probability $p_{cc}$;
(iii) We connect the covariates to the outcome with probability $p_{cy}$;
(iv) We connect the treatment to the mediators with probability $p_{mx}$;
(v) We connect one mediator to another with probability $p_{mm}$;
(vi) We connect the mediators to the outcome with probability $p_{my}$;
(vi) We connect a mediators to a covariate with probability $p_{cm}$.
For our experiments, we have used $p_{cx} = p_{cc} = p_{cy} = p_{mx} = p_{mm} =
p_{my} = 0.1$ and $p_{cm} = 0.05$.

For each node $V$, we generate data
using the following structural equation:
\begin{align*}
    v := b^\top_V \text{pa}(v) + \epsilon_V, \,\, \epsilon_V \sim \mathcal{N}(0, \sigma^2_V),
\end{align*}
where $v$ and $\text{pa}(v)$ are the realized values of node $V$ and its parents,
respectively; $b^\top_V \in \mathbb{R}^{|\Pa(V)|}$ is the vector denoting the
edge coefficients; and $\epsilon_V$ is an independently sampled noise term.
Each element of $b_V$ is sampled uniformly from the 
interval $[-1, -0.25] \cup [0.25, 1]$ and
$\sigma^2_V$ is sampled independently from a
uniform distribution $U(0.1, 0.2)$.

\subsection{Results on synthetic linear binomial (Erdos-Renyi) graphs.}\label{sec:apdx-experiments-synthetic-erdos-renyi}

We also test our methods on binomial graphs of size $|\V| = 30$.
In binomial graphs, an edge between two nodes is generated with
some probability $p_e$. For our experiments, we use $p_e = 0.8$.
We generate the data using linear Gaussian models where the
model parameters are sampled in the same way as Sec.~\ref{sec:apdx-experiments-synthetic-linear-dgp}.
We compare PC, SD, MB-by-MB, and LDECC based on accuracy (Fig.~\ref{fig:apdx-er-accuracy}), recall (Fig.~\ref{fig:apdx-er-recall}), 
MSE (Fig.~\ref{fig:apdx-er-mse}), and
number of CI tests (Fig.~\ref{fig:apdx-er-ci-tests})
across four different sample sizes.

In terms of accuracy, MB-by-MB performs the best
with LDECC-checks performing slightly better than SD and PC.
In terms of recall, LDECC-checks and MB-by-MB perform comparably.
We also observe that LDECC-checks has higher accuracy
and recall than LDECC.
LDECC-checks and MB-by-MB have the lower MSE than PC and SD.
Both variants of LDECC have lower MSE than PC and SD 
and all four local causal discovery algorithms perform
a comparable number of CI tests (and fewer tests than PC).

\subsection{Additional results on semi-synthetic graphs.}\label{sec:apdx-experiments-semi-synthetic}

\begin{figure}
\centering
\subfigure[Tests with CI oracle]{\includegraphics[scale=0.35]{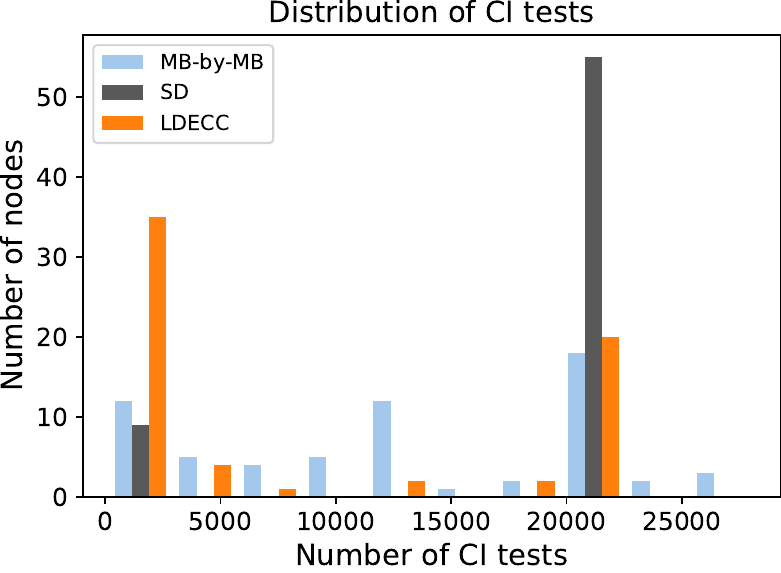}\label{fig:apdx-magic-irri-with-ci-oracle}}
\hfill
\subfigure[Accuracy]{\includegraphics[scale=0.35]{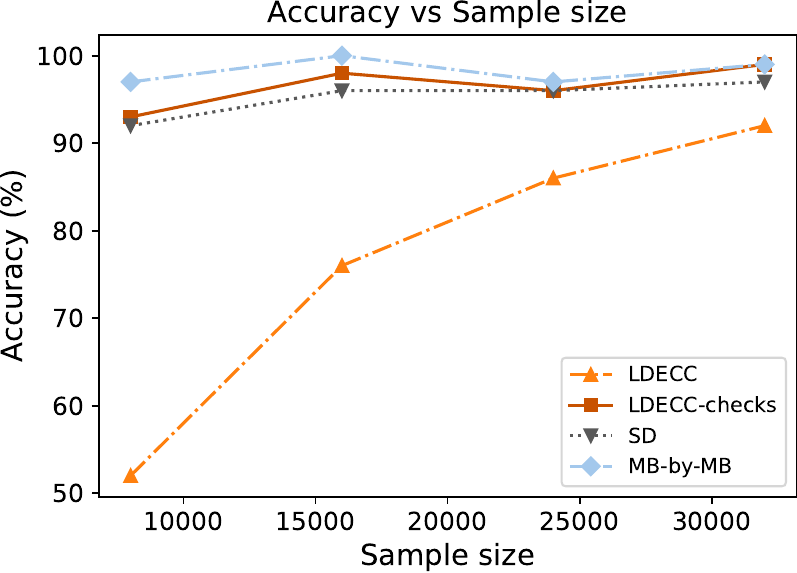}\label{fig:apdx-magic-irri-accuracy}}
\hfill
\subfigure[Recall]{\includegraphics[scale=0.35]{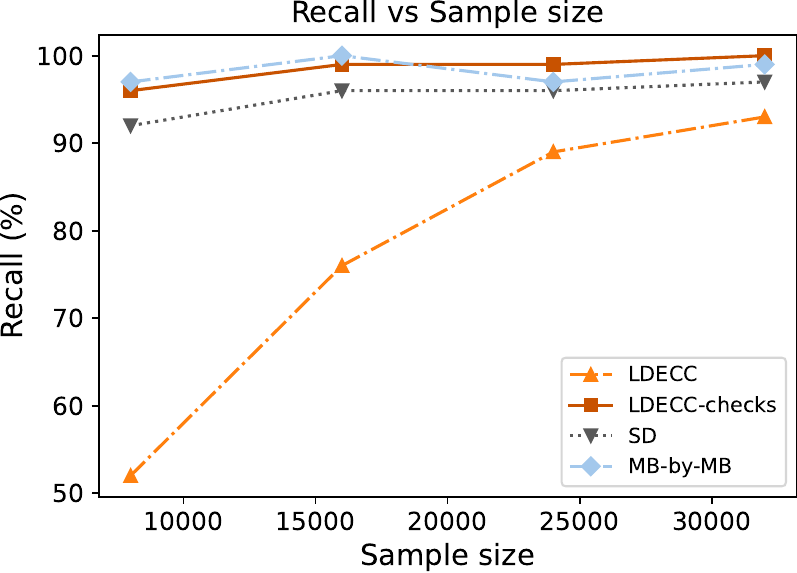}\label{fig:apdx-magic-irri-recall}}
\\
\subfigure[Median SE vs sample size]{\includegraphics[scale=0.33]{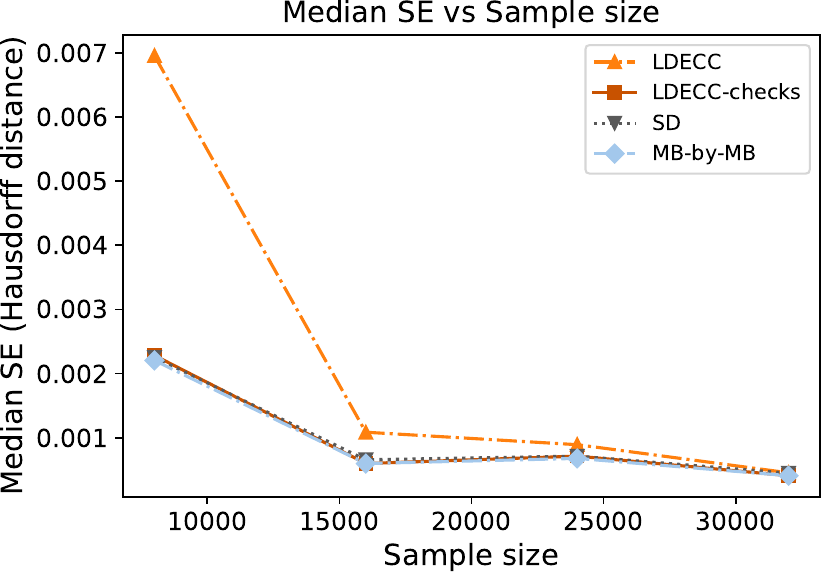}\label{fig:apdx-magic-irri-mse}}
\hspace{3em}
\subfigure[Average number of CI tests]{\includegraphics[scale=0.35]{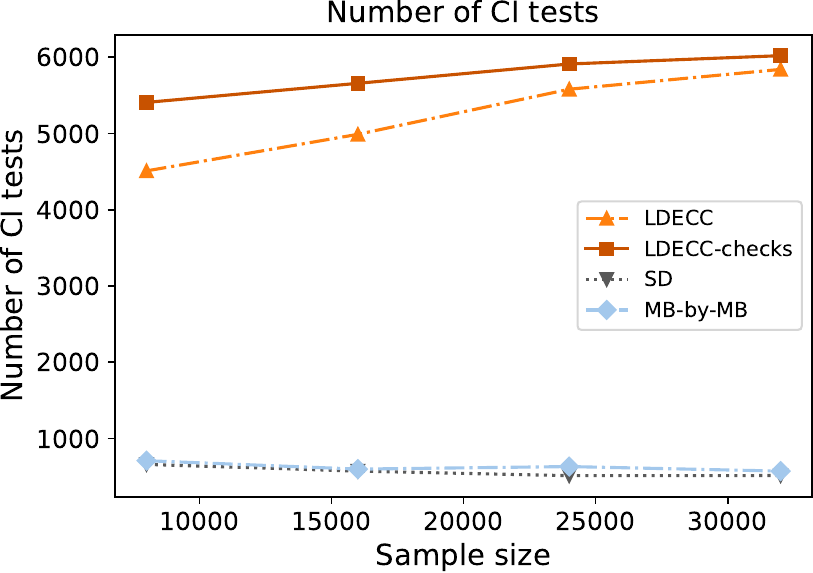}\label{fig:apdx-magic-irri-num-tests}}
\caption{Results on the semi-synthetic \emph{MAGIC-IRRI} graph.}
\label{fig:apdx-magic-irri-graphs}
\end{figure}

We also present results on the linear
Gaussian \emph{MAGIC-IRRI} graph from
\emph{bnlearn} (Fig.~\ref{fig:apdx-magic-irri-graphs}).
We plot the distribution of CI tests with a CI oracle
by repeatedly setting each node as the treatment 
(capping the maximum number of tests per node to $20000$).
We see that LDECC and MB-by-MB perform differently
across different nodes and outperform SD on most nodes
(Fig.~\ref{fig:apdx-magic-irri-with-ci-oracle}).
Next, we designated the nodes \emph{G6003} and \emph{BROWN}
as the treatment and outcome, respectively.
At four sample sizes, we sample data from
the graph $100$ times
(capping the maximum number of tests run by each algorithm to $7000$).
In terms of both accuracy (Fig.~\ref{fig:apdx-magic-irri-accuracy}) 
and recall (Fig.~\ref{fig:apdx-magic-irri-recall}),
LDECC-checks, SD, and MB-by-MB perform comparably
while LDECC does poorly at small sample sizes.
In terms of Median SE, LDECC-checks, SD, and MB-by-MB
perform comparably (Fig.~\ref{fig:apdx-magic-irri-mse}) but
LDECC performs substantially more CI tests (on average)
than SD and MB-by-MB (Fig.~\ref{fig:apdx-magic-irri-num-tests}).

Finally, we compare PC, SD, MB-by-MB, and LDECC based on the number of CI
tests (with access to a CI oracle) on three discrete graphs
from \emph{bnlearn}: \emph{Alarm} (Fig.~\ref{fig:apdx-bnlearn-alarm-oracle}), \emph{Insurance} (Fig.~\ref{fig:apdx-bnlearn-insurance-oracle}), and \emph{Mildew} (Fig.~\ref{fig:apdx-bnlearn-mildew-oracle}).
We plot the distribution of tests for SD, MB-by-MB, and LDECC
by setting each node in the graph as the treatment.
We see that for most nodes on all three graphs, MB-by-MB and
LDECC outperform SD.
MB-by-MB has the best performance for all three graphs
and performs well across all nodes.

\begin{figure}
\centering
\subfigure[\emph{Alarm} graph.]{\includegraphics[scale=0.35]{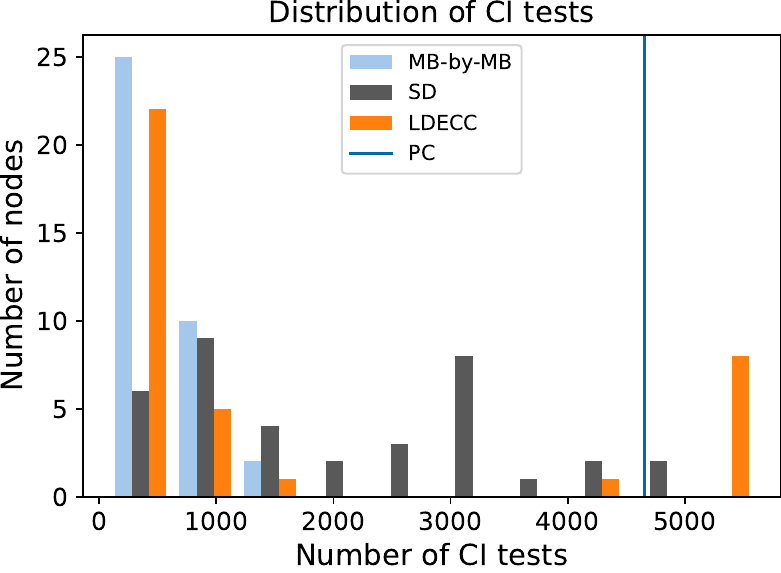}\label{fig:apdx-bnlearn-alarm-oracle}}
\hfill
\subfigure[\emph{Insurance} graph.]{\includegraphics[scale=0.35]{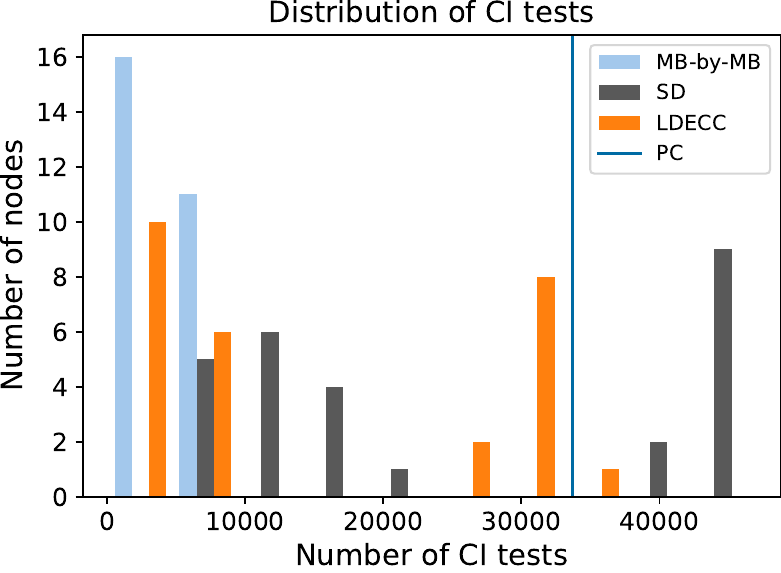}\label{fig:apdx-bnlearn-insurance-oracle}}
\hfill
\subfigure[\emph{Mildew} graph.]{\includegraphics[scale=0.35]{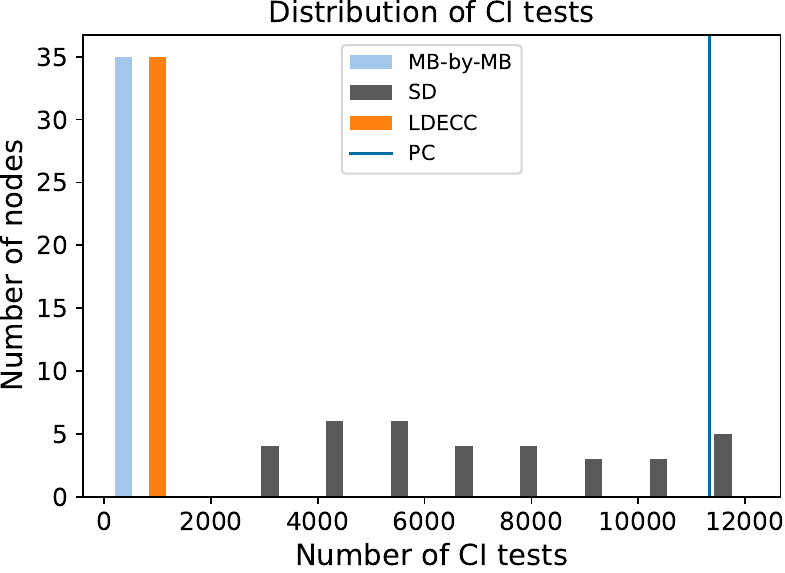}\label{fig:apdx-bnlearn-mildew-oracle}}
\caption{Comparison of PC, SD, MB-by-MB, and LDECC
based on the number of CI tests (with a CI oracle)
on three discrete graphs from \emph{bnlearn}.}
\label{fig:apdx-bnlearn-discrete-graphs}
\end{figure}

\section{Covariate adjustment using local information} \label{sec:apdx-covariate-adjustment}
\subsection{Checking the backdoor criterion.}\label{sec:apdx-adjustment-backdoor}

The \emph{backdoor criterion} \citep[Defn.~3.3.1]{pearl2009causality}
is a sufficient (but not necessary) condition
for a given subset of nodes
$\Z$ to be a valid adjustment set
with respect to a treatment $X$ and an outcome $Y$.
\citet{maathuis2015generalized} extended this criterion to
be applicable to
various MECs including CPDAGs.
We begin by briefly introducing the existing results 
in \citet{maathuis2015generalized} and then we prove that
it is possible to check the backdoor criterion using only
local information around $X$ and $\OM(|\V| \cdot 2^{|\DNe(X)|} )$ additional CI tests.

For introducing existing results, 
let the true DAG be $\GM^{*}$ and
let the corresponding CPDAG that represents the MEC of $\GM^{*}$
be $\mathcal{C}^{*}$ (see Defn.~\ref{defn:apdx-cpdag}).

\begin{definition}[Possibly causal path]
A path $A \, \text{---} \hdots \text{---} \, B$ is said to be 
possibly causal if there is at least one DAG in $\mathcal{C}^{*}$
with a directed path from $A$ to $B$: $A \rightarrow \hdots \rightarrow B$.
\end{definition}

\begin{definition}[Visible and invisible edges {\citep[Defn.~3.1]{maathuis2015generalized}}]
All directed edges in a CDPAG are said to be visible.
All undirected edges in a CPDAG are said to be invisible.
\end{definition}

\begin{definition}[Backdoor path {\citep[Defn.~3.2]{maathuis2015generalized}}]
We say that path between $X$ and $Y$ is a backdoor path if this path does not
have a visible edge out of $X$.
\end{definition}

\begin{definition}[Definite non-collider {\citep[Defn.~3.3]{maathuis2015generalized}}]
A nonendpoint vertex $V_j$ on a path $<\hdots, V_i, V_j, V_k, \hdots>$
in a CPDAG is a definite non-collider if the triple
$<V_i, V_j, V_k>$ is unshielded and the edges $V_i \text{---} V_j$ and
$V_j \text{---} V_k$ are undirected.
\end{definition}

\begin{definition}[Definite status path {\citep[Defn.~3.4]{maathuis2015generalized}}]
A nonendpoint vertex $B$ on a path $p$ in a CPDAG is said to be of a definite status 
if it is either a collider or a definite non-collider on $p$. 
The path $p$ is said to be of a definite status if all nonendpoint vertices on the path are of a definite status.
\end{definition}

\begin{definition}[Possible Descendant]
A node $A$ is a possible descendant of a node $V$ iff
there is a possibly causal path from $V$ to $A$.
\end{definition}

\begin{figure}[t]
\centering
\begin{minipage}[b]{0.65\textwidth}
    \vspace{0pt}
    \setlength{\interspacetitleruled}{0pt}%
    \setlength{\algotitleheightrule}{0pt}%
    \begin{algorithm}[H]
    \SetAlgoLined
    \KwInput{Treatment $X$, Outcome $Y$, $\Z$.}
    $\M \gets \Ch(X) \cup \Uo(X)$\;
    $\I \gets \{ V : \text{GetMNS}(V) = \text{invalid} \}$\; \label{algo-gbdc-mns-find-1}
    $\PossDesc(X) \gets \M \cup \I \cup \{ V \in \V \setminus \I : \M \cap \text{GetMNS}(V) \neq \emptyset \}$\; \label{algo-gbdc-mns-find-2}
    \lIf{$\Z \cap \PossDesc(X) \neq \emptyset$}{
        \textbf{return} False
    }
    
    \For{$Q \in (\Pa(X) \cup \Uo(X)) \setminus \Z$}{
        \lIf{$Q \notindep Y | \{X\} \cup \Z$}{
            \textbf{return} False
        } \label{algo-gbdc-additional-tests}
    }
    \textbf{return} True\;
    \end{algorithm}
    \captionof{figure}{Checking the Generalized Backdoor Criterion.}
    \label{fig:apdx-algo-check-gbdc}
\end{minipage}
\end{figure}

\begin{definition}[Generalized backdoor criterion (GBC) {\citep[Defn.~3.7]{maathuis2015generalized}}]
A set of variables $\Z$ satisfies the 
backdoor criterion relative to $(X, Y)$
in a CPDAG if
the following two conditions:
(1) $\Z$ does not contain any possible descendants of $X$; and
(2) $\Z$ blocks every definite status backdoor path from
$X$ to $Y$.
\end{definition}
Intuitively, the GBC checks whether $\Z$ satisfies the backdoor criterion for
every DAG in an MEC.
Having introduced the GBC,
we extend this result and prove that it is possible to check 
this criterion using only local information.
Our procedures are related to existing methods
in prior work used for covariate selection \citep{vanderweele2011new, entner2013data}
which also use very similar testing strategies.
However, these works make slightly stronger assumptions
on the known partial orderings (e.g., that all covariates are pre-treatment)
whereas our goal is to demonstrate that a similar testing strategy
along with the output of a local causal discovery algorithm is also
sufficient to determine if a given subset is a valid adjustment set.
These prior works also accommodate latent pre-treatment
variables whereas we make the assumption of causal sufficiency
throughout our work.

\begin{definition}[Unoriented nodes]\label{defn:apdx-unoriented-nodes}
We define unoriented nodes, denoted by $\Uo(X)$,
as the set of nodes $V \in \DNe(X)$ such that the
edge $X \text{---} V$ is invisible in $\mathcal{C}^{*}$.
\end{definition}

Importantly, both local discovery procedures, SD and LDECC,
find $\Uo(X)$.
These nodes are stored in the variable \emph{unoriented} in the
algorithms (See Figs.~\ref{fig:algo-sd},\ref{fig:algo-cdud}).

\begin{lemma}\label{lemma:apdx-poss-desc}
Let $\M = \Ch(X) \cup \Uo(X)$ and $\I = \{ V : \mns_X(V) = \text{invalid} \}$.
The possible descendants of a node $X$ are
$\PossDesc(X) = \M \cup \I \cup \{ V \in \V \setminus \I : 
\M \cap \mns_X(V) \neq \emptyset \}$.
\end{lemma}
\begin{proof}
The nodes in $\M$ are possible descendants of $X$.
By Prop.~\ref{prop:valid-nss-non-desc}, nodes in $\I$ are also possible
descendants.
For any possible descendant $V \notin \DNep(X) \cup \I$ of $X$, there must be a
path $X \rightarrow M \rightarrow \hdots \rightarrow V$,
where $M \in \M$,
in at least one DAG.
Therefore, it must be the case that for every such node $V$, we have
$\M \cap \mns_X(V) \neq \emptyset$.
\end{proof}

\begin{proposition}[Checking the GBC]\label{prop:apdx-backdoor-criterion}
Let $\M = \Ch(X) \cup \Uo(X)$, where $\Uo(X)$
is defined in Defn.~\ref{defn:apdx-unoriented-nodes}.
Let $\PossDesc(X) = \M \cup \{ V : \M \cap \nss_X(V) \neq \emptyset \}$.
Consider a subset of nodes $\Z$.
Let $\Q = (\Pa(X) \cup \Uo(X)) \setminus \Z$.
Then $\Z$ satisfies the backdoor criterion 
for every DAG in the MEC iff:
(i) $\Z \cap \PossDesc(X) = \emptyset$, and
(ii) $\forall Q \in \Q, Q \indep Y | \{X, \Z \}$.
The algorithm for checking the GBC is given in Fig.~\ref{fig:apdx-algo-check-gbdc} and it 
performs $\OM(|\V| \cdot 2^{|\DNe(X)|} )$ additional CI tests.
\end{proposition}
\begin{proof}
By Lemma~\ref{lemma:apdx-poss-desc}, $\PossDesc(X)$
contains the possible descendants of $X$.
Condition~(i) is therefore necessary since
the descendants of $X$ cannot satisfy the
backdoor criterion.
Thus, for the rest of the proof,
we assume that $\Z \cap \PossDesc(X) = \emptyset$.

We first prove the forward direction: 
If $\Z$ is a valid adjustment set then 
$\forall Q \in \Q, Q \indep Y | \{X, \Z \}$.
Since $\Z$ is a valid adjustment set, 
$\Z$ blocks all possibly backdoor paths in every DAG in the MEC.
Therefore, we have $\forall Q \in \Q, \, Q \indep Y | \{X, \Z \}$ 
because otherwise there will at least 
one DAG where the path 
$X \leftarrow Q \text{---} \hdots \rightarrow Y$ will be open
for some $Q \in \Q$.

Next, we prove the backward direction: 
if $\forall Q \in \Q, Q \indep Y | \{X, \Z \}$, 
then $\Z$ is a valid adjustment set.
Firstly, for all $P \in \Pa(X) \cap \Z$ (i.e., $P \notin \Q$),
all backdoor paths of the form
$X \leftarrow P \text{---} \hdots \text{---} Y$
are blocked because $P \in \Z$.
Since for all $Q \in \Q$, we have $Q \indep Y | \{X, \Z \}$, 
all possible backdoor paths 
$X \leftarrow Q \text{---} \hdots \rightarrow Y$ 
are blocked by $\Z$ and 
therefore it is a valid adjustment set 
in every DAG of the MEC.

Finally, we perform $\OM(|\V| \cdot 2^{|\DNe(X)|} )$ additional CI tests
in Lines~\ref{algo-gbdc-mns-find-1},\ref{algo-gbdc-mns-find-2}
to find $\PossDesc(X)$ (because for every node, we can find
its MNS in $\OM(2^{|\DNe(X)|})$ tests).
Next, in Line~\ref{algo-gbdc-additional-tests},
since we only run tests for $N \in (\Pa(X) \cup \Uo(X)) \setminus \Z$,
we perform $\OM(|\DNe(X)|)$ extra CI tests.
\end{proof}

\subsection{Finding the optimal adjustment set.}\label{sec:apdx-adjustment-oas}

A given DAG can have multiple valid adjustment sets.
\citet{henckel2019graphical}[Sec.~3.4] introduce a graphical criterion
for linear models
for determining the optimal adjustment set,
i.e., the set with the lowest asymptotic variance.
This criterion was later shown to hold non-parametrically \citep{rotnitzky2019efficient}.
We begin by introducing the existing results and then
prove that we can find the optimal adjustment set 
using only local information and $\OM(|\V|)$ additional CI tests
(see (Fig.~\ref{fig:apdx-algo-optimal-adjustment-set})).

Like the previous section,
let the true DAG be $\GM^{*}$ and
let the corresponding CPDAG that represents the MEC of $\GM^{*}$
be $\mathcal{C}^{*}$ (see Defn.~\ref{defn:apdx-cpdag}).

\begin{definition}[Possible causal nodes {\citep[Sec.~3.4, Pg.~29]{henckel2019graphical}}]
The causal nodes relative to $(X, Y)$, denoted by $\text{posscn}(X, Y)$,
are all nodes on possibly causal paths from $X$ to $Y$, excluding $X$.
\end{definition}

\begin{definition}[Forbidden nodes {\citep[Sec.~3.4, Pg.~29]{henckel2019graphical}}]
The forbidden nodes relative to $(X, Y)$, denoted by $\text{forb}(X, Y)$,
are defined as
\begin{align*}
    \text{forb}(X, Y) = \PossDesc(\text{posscn}(X, Y)) \cup \{ X \}.
\end{align*}
\end{definition}

\begin{definition}[Optimal adjustment set {\citep[Defn.~3.12]{henckel2019graphical}}]\label{defn:apdx-OAS}
The optimal adjustment set relative to $X, Y$ is defined as
\begin{align*}
    \mathbf{O}(X, Y, \mathcal{C}^{*}) = \Pa(\text{posscn}(X, Y)) \setminus \text{forb}(X, Y).
\end{align*}
\end{definition}

\begin{figure}[t]
\centering
\begin{minipage}[b]{0.55\textwidth}
    \setlength{\interspacetitleruled}{0pt}%
    \setlength{\algotitleheightrule}{0pt}%
    \begin{algorithm}[H]
    \SetAlgoLined
    \KwInput{Treatment $X$, Outcome $Y$, $\Pa(X), \Ch(X)$, unoriented $\Uo(X)$.}
    $\M \gets \Ch(X) \cup \Uo(X)$\;
    $\I \gets \{ V : \text{GetMNS}(V) = \text{invalid} \}$\;
    $\PossDesc(X) \gets \M \cup \I \cup \{ V \in \V \setminus \I : \M \cap \text{GetMNS}(V) \neq \emptyset \}$\;
    $\Z \gets (\V \setminus \text{PossDesc}(X))$\; \label{algo-oas-Z-largest-gbc}
    $\mathbf{O} \gets$ HenckelPrune($X, Y, \Z$) \citep[Alg.~1]{henckel2019graphical} (Fig.~\ref{fig:apdx-henckel-prune}) \; \label{algo-oas-invoke-henckel-prune}
    \For{$M \in (\Pa(X) \cup \Uo(X)) \setminus \mathbf{O}$}{ \label{algo-oas-check-if-gbc-holds}
        \lIf{$M \notindep Y | \{X, \mathbf{O} \}$}{\textbf{return} \emph{noValidAdj}} \label{algo-oas-no-valid-adj}
    }
    \KwOutput{$\mathbf{O}$}
    \end{algorithm}
    \captionof{figure}{Finding the optimal adjustment set.}
    \label{fig:apdx-algo-optimal-adjustment-set}
\end{minipage}
\hfill
\begin{minipage}[b]{0.40\textwidth}
    \vspace{0pt}
    \setlength{\interspacetitleruled}{0pt}%
    \setlength{\algotitleheightrule}{0pt}%
    \begin{algorithm}[H]
    \SetAlgoLined
    \SetKwFunction{HPAlogirthm}{HenckelPrune}
    \SetKwProg{HPA}{def}{:}{}
    \HPA{\HPAlogirthm{Treatment $X$, Outcome $Y$, Subset $\Z$}}{
        $\Z' \gets \Z$\;
        \For{$Z \in \Z$}{
            \lIf{$Y \indep Z | \{X\} \cup (\Z' \setminus \{Z\})$}{$\Z' \leftarrow \Z' \setminus \{Z\}$}
        }
        \KwRet $\Z'$\;
    }
    \end{algorithm}
    \captionof{figure}{The \emph{HenckelPrune} function.}
    \label{fig:apdx-henckel-prune}
\end{minipage}
\end{figure}

We now show that it is possible to find the optimal adjustment set
using only local information.

\begin{lemma}\label{lemma:apdx-optimal-adj-not-contain-desc}
Given a CPDAG $\mathcal{C}$, the optimal adjustment set
does not contain $\PossDesc(X)$.
\end{lemma}
\begin{proof}
As stated in Defn.~\ref{defn:apdx-OAS}, 
the optimal adjustment set is 
$\mathbf{O}(X, Y, \mathcal{C}^{*}) = \Pa(\text{posscn}(X, Y)) \setminus \text{forb}(X, Y)$.
Therefore, we only need to show that possible descendants of $X$ 
that are \emph{not} on a causal path from $X$ to $Y$ 
cannot be in $\mathbf{O}$ (otherwise they will be in $\text{forb}(X, Y)$).
Consider a node $V \in \Desc(X)$ not on a causal path from $X$ to $Y$. 
This node cannot be a parent of any node in $\text{cn}(X, Y)$. 
This is because if that happens, 
then $V$ must also belong to $\text{cn}(X, Y)$ which leads to a contradiction.
\end{proof}

\begin{cor}\label{cor:optimal-adj-is-backdoor}
The optimal adjustment set satisfies the GBC.
\end{cor}
\begin{proof}
By Lemma~\ref{lemma:apdx-optimal-adj-not-contain-desc}, the optimal adjustment set
does not contain any possible descendants of $X$.
Furthermore, by definition, it is a valid adjustment set and therefore
blocks all backdoor paths from $X$ to $Y$.
Therefore, it satisfies the GBC.
\end{proof}

\begin{proposition}[Optimal adjustment set]\label{prop:apdx-finding-the-OAS}
Consider the algorithm in Fig.~\ref{fig:apdx-algo-optimal-adjustment-set}. 
It performs $\OM(|\V| \cdot 2^{|\DNe(X)|})$ CI tests and
(i) returns \texttt{noValidAdj} if there is no valid adjustment that applies to all DAGs in the MEC;
(ii) else, returns the optimal adjustment set that is valid 
for all DAGs in the MEC (denoted by $\mathbf{O})$.
\end{proposition}
\begin{proof}
In Line~\ref{algo-oas-Z-largest-gbc}, $\Z = \V \setminus \PossDesc(X)$
represents the largest possible set that 
could satisfy the GBC, if any such set exists. 
Therefore, this set $\Z$
would be a superset of the optimal adjustment set, if it exists 
(by Cor.~\ref{cor:optimal-adj-is-backdoor}).
In Line~\ref{algo-oas-invoke-henckel-prune}, we invoke the pruning procedure in
\citet[Algorithm~1]{henckel2019graphical} which outputs the optimal adjustment 
when starting from a superset (see Fig.~\ref{fig:apdx-henckel-prune}).
In Line~\ref{algo-oas-check-if-gbc-holds}, we verify that the pruned set 
$\mathbf{O}$ is a valid adjustment set (see Prop~\ref{prop:apdx-backdoor-criterion}). 
\citet[Theorem~3.13(i)]{henckel2019graphical} also prove that an optimal adjustment set exists iff there is some valid adjustment set. 
Thus, if Line~\ref{algo-oas-no-valid-adj} is reached, it means that 
there is no valid adjustment set that applies to every DAG in the MEC.

We perform $\OM(|\V| \cdot 2^{|\DNe(X)|})$ CI tests to
find $\PossDesc(X)$.
The \emph{HenckelPrune} function and Line~\ref{algo-oas-check-if-gbc-holds}
perform $\OM(|\V|)$ additional CI tests.
\end{proof}

\end{document}